\documentclass[11pt]{article}
\usepackage{hyperref,xspace,fullpage}
\usepackage{amsmath}
\usepackage{amsthm}
\usepackage{amsfonts}
\usepackage{paralist}
\usepackage{bm}
\usepackage{array}
\usepackage{multirow}
\newif\iffull
\fulltrue

\newtheorem{theorem}{Theorem}[section]

\newtheorem{definition}[theorem]{Definition}

\newtheorem{lemma}[theorem]{Lemma}
\newtheorem{corollary}[theorem]{Corollary}
\newtheorem{observation}[theorem]{Observation}
\newtheorem{proposition}[theorem]{Theorem}
\newtheorem{remark}[theorem]{Remark}

\usepackage{bbm}
\usepackage{color}
\newcommand*{\citenames}[2]{#1~\cite{#2}}
\definecolor{DarkGreen}{rgb}{0.1,0.5,0.1}
\definecolor{DarkRed}{rgb}{0.5,0.1,0.1}
\definecolor{DarkBlue}{rgb}{0.1,0.1,0.5}
\def\ShowAuthNotes{1}
\ifnum\ShowAuthNotes=1
\newcommand{\authnote}[2]{{ \footnotesize \bf{\color{DarkRed}[#1:
{\color{DarkBlue}#2}]}}}
\else
\newcommand{\authnote}[2]{}
\fi

\usepackage{vfmacros}

\providecommand{\K}{{\mathcal K}}
\providecommand{\Ksym}{{{\mathcal K}_0}}
\providecommand{\W}{{\mathcal W}}

\newcommand{\bw}{{\mathbf w}}
\newcommand{\bx}{{\mathbf x}}
\newcommand{\by}{{\mathbf y}}
\newcommand{\bv}{{\mathbf v}}
\newcommand{\bz}{{\mathbf z}}
\newcommand{\bu}{{\mathbf u}}
\newcommand{\br}{{\mathbf r}}
\newcommand{\bU}{{\mathbf U}}
\newcommand{\ind}[1]{{\mathbf 1}_{\{#1\}}}
\newcommand{\dc}{\kappa_2}
\newcommand{\SDN}{{\mathrm{SDN}}}
\newcommand{\STAT}{{\mbox{STAT}}}
\newcommand{\VSTAT}{{\mbox{VSTAT}}}
\newcommand{\conv}{{\mbox{conv}}}

\newcommand{\Opt}{{\mbox{Opt}}}
\newcommand{\Stat}{{\mbox{Stat}}}
\newcommand{\cO}{{\mathcal O}}
\newcommand{\vol}{{\mbox{Vol}}}

\newcolumntype{M}[1]{>{\centering\arraybackslash}m{#1}}
\newcolumntype{N}{@{}m{0pt}@{}}

\title{Statistical Query Algorithms for Mean Vector Estimation and Stochastic Convex Optimization}
\usepackage{times}

\author{Vitaly Feldman\\
 \tt{vitaly@post.harvard.edu}\\
  IBM Research - Almaden
 \and
 Crist\'obal Guzm\'an\footnote{Part of this work was
 done during an internship at IBM Research - Almaden,
 at a postdoctoral position of N\'ucleo Milenio Informaci\'on y Coordinaci\'on
 en Redes (ICM/FIC P10-024F) at Universidad de Chile, and
 at a postdoctoral position of Centrum Wiskunde \& Informatica.}
 \\
 \tt{crguzmanp@uc.cl}
 \\ Facultad de Matem\'aticas \& Escuela de Ingenier\'ia\\
 Pontificia Universidad Cat\'olica de Chile
  \and
 Santosh Vempala\\
 \tt{vempala@gatech.edu}\\
School of Computer Science\\ Georgia Institute of Technology
}
\date{}
\begin{document}

\maketitle

\begin{abstract}
Stochastic convex optimization, where the objective is the expectation of a random convex function, is an important and widely used method with numerous applications in machine learning, statistics, operations research and other areas. We study the complexity of stochastic convex optimization given only {\em statistical query} (SQ) access to the objective function. We show that well-known and popular first-order iterative methods can be implemented using only statistical queries. For many cases of interest we derive nearly matching upper and lower bounds on the estimation (sample) complexity including linear optimization in the most general setting. We then present several consequences for machine learning, differential privacy and proving concrete lower bounds on the power of convex optimization based methods.

The key ingredient of our work is SQ algorithms and lower bounds for estimating the mean vector of a distribution over vectors supported on a convex body in $\R^d$. This natural problem has not been previously studied and we show that our solutions can be used to get substantially improved SQ versions of Perceptron and other online algorithms for learning halfspaces.
\end{abstract}

\thispagestyle{empty}
\newpage
\begin{small}
\tableofcontents
\end{small}
\thispagestyle{empty}
\newpage

\setcounter{page}{1}

\section{Introduction}

Statistical query (SQ) algorithms, defined by \citenames{Kearns}{Kearns:98} in the context of PAC learning and by
\citenames{Feldman \etal}{FeldmanGRVX:12} for general problems on inputs sampled i.i.d.~from distributions, are algorithms that can be implemented using estimates of the expectation of any given function on a sample drawn randomly from the input distribution $D$ instead of direct access to random samples. Such access is abstracted using a {\em statistical query oracle} that given a query function $\phi:\W \rar [-1,1]$ returns an estimate of $\E_{\bw}[\phi(\bw)]$ within some tolerance $\tau$ (possibly dependent on $\phi$). We will refer to the number of samples sufficient to estimate the expectation of each query of a SQ algorithm with some fixed constant confidence as its {\em estimation complexity}  (often $1/\tau^2$) and the number of queries as its {\em query complexity}.

Statistical query access to data was introduced as means to derive noise-tolerant algorithms in the PAC model of learning \cite{Kearns:98}. Subsequently, it was realized that reducing data access to estimation of simple expectations has a wide variety of additional useful properties. It played a key role in the development of the notion of differential privacy \cite{DinurN03,BlumDMN:05,DworkMNS:06} and has been subject of intense subsequent research in differential privacy\footnote{In this context an ``empirical" version of SQs is used which is referred to as {\em counting} or {\em linear} queries. It is now known that empirical values are close to expectations when differential privacy is preserved \cite{DworkFHPRR14:arxiv}.} (see \cite{DworkRoth:14} for a literature review).  It has important applications in a large number of other theoretical and practical contexts such as distributed data access \cite{ChuKLYBNO:06,RoySKSW10,Sujeeth:11}, evolvability \cite{Valiant:09,Feldman:08ev,Feldman:09sqd} and memory/communication limited machine learning \cite{BalcanBFM12,SteinhardtVW:2016}. Most recently, in a line of work initiated by \citenames{Dwork \etal}{DworkFHPRR14:arxiv}, SQs have been used as a basis for understanding generalization in adaptive data analysis \cite{DworkFHPRR14:arxiv,HardtU14,DworkFHPRR15:arxiv,SteinkeU15,BassilyNSSSU15}.

Here we consider the complexity of solving stochastic convex minimization problems by SQ  algorithms. In stochastic convex optimization the goal is to minimize a convex function $F(x) = \E_{\bw}[f(x,\bw)]$ over a convex set $\K \subset \R^d$, where $\bw$ is a random variable distributed according to some distribution $D$ over domain $\W$ and each $f(x,w)$ is convex in $x$. The optimization is based on i.i.d.~samples $w^1,w^2,\ldots,w^n$ of $\bw$. Numerous central problems in machine learning and statistics are special cases of this general setting with a vast literature devoted to techniques for solving variants of this problem (\eg \cite{SrebroTewari:2010Tutorial,Shalev-ShwartzBen-David:2014}). It is usually assumed that $\K$ is ``known" to the algorithm (or in some cases given via a sufficiently strong oracle) and the key challenge is understanding how to cope with estimation errors arising from the stochastic nature of information about $F(x)$.

Surprisingly, prior to this work, the complexity of this fundamental class of problems has not been studied in the SQ model. This is in contrast to the rich and nuanced understanding of the sample and computational complexity of solving such problems given unrestricted access to samples as well as in a wide variety of other oracle models.

The second important property of statistical algorithms is that it is possible to prove information-theoretic lower bounds on the complexity of any statistical algorithm that solves a given problem.  The first one was shown by \citenames{Kearns}{Kearns:98} who proved that parity functions cannot be learned efficiently using SQs. Subsequent work has developed several techniques for proving such lower bounds (\eg \cite{BlumFJ+:94,Simon:07,FeldmanGRVX:12,FeldmanPV:13}), established relationships to other complexity measures (\eg \cite{Sherstov:08,KallweitSimon:11}) and provided lower bounds for many important problems in learning theory (\eg \cite{BlumFJ+:94,KlivansSherstov:07a,FeldmanLS:11colt}) and beyond \cite{FeldmanGRVX:12,FeldmanPV:13,BreslerGS14a,WangGL:15}.

From this perspective, statistical algorithms for stochastic convex optimization have another important role. For many problems in machine learning and computer science,  convex optimization gives state-of-the-art results and therefore lower bounds against such techniques are a subject of significant research interest. Indeed, in recent years this area has been particularly active with major progress made on several long-standing problems (\eg \cite{Fiorini:2012,Rothvoss14,MekaPW15,LeeRS15}). As was shown in \cite{FeldmanPV:13}, it is possible to convert SQ lower bounds into purely structural lower bounds on convex relaxations, in other words lower bounds that hold without assumptions on the algorithm that is used to solve the problem (in particular, not just SQ algorithms). From this point of view, each SQ implementation of a convex optimization algorithm is a new lower bound against the corresponding convex relaxation of the problem.

\subsection{Overview of Results}
We focus on iterative first-order methods namely techniques that rely on updating the current point $x^t$ using only the (sub-)gradient of $F$ at $x^t$. These are among the most widely-used approaches for solving convex programs in theory and practice. It can be immediately observed that for every $x$, $\nabla F(x) = \E_{\bw}[\nabla f(x,{\bw})]$ and hence it is sufficient to estimate expected gradients to some sufficiently high accuracy in order to implement such algorithms (we are only seeking an approximate optimum anyway). The accuracy corresponds to the number of samples (or estimation complexity) and is the key measure of complexity for SQ algorithms. However, to the best of our knowledge, the estimation complexity for specific SQ implementations of first-order methods has never been formally addressed.

We start with the case of linear optimization, namely $\nabla F(x)$ is the same over the whole body $\K$. It turns out that in this case global approximation of the gradient (that is one for which the linear approximation of $F$ given by the estimated gradient is $\eps$ close to the true linear approximation of $F$) is sufficient. This means that the question becomes that of estimating the mean vector of a distribution over vectors in $\R^d$ in some norm that depends on the geometry of $\K$. This is a basic question (indeed, central to many high-dimensional problems) but it has not been carefully addressed even for the simplest norms like $\ell_2$. We examine it in detail and provide an essentially complete picture for all $\ell_q$ norms with $q\in [1,\infty]$. We also briefly examine the case of general convex bodies (and corresponding norms) and provide some universal bounds. 

The analysis of the linear case above gives us the basis for tackling first-order optimization methods for Lipschitz convex functions. That is, we can now obtain an estimate of the expected gradient at each iteration. However we still need to determine whether the global approximation is needed or a local one would suffice and also need to ensure that estimation errors from different iterations do not accumulate. Luckily, for this we can build on the study of the performance of first-order methods with inexact first-order oracles. Methods of this type
have a long history (\eg \cite{Polyak:1987,Shor:2011}), however some of our methods of choice have
only been studied recently.  We give SQ algorithms for implementing the global and local oracles and then systematically study several traditional setups of convex optimization: non-smooth, smooth and strongly convex. While that is not the most exciting task in itself, it serves to show the generality of our approach. Remarkably, in all of these common setups we achieve the same estimation complexity as what is known to be achievable with samples.

All of the previous results require that the optimized functions are Lipschitz, that is the gradients are bounded in the appropriate norm (and the complexity depends polynomially on the bound). Addressing non-Lipschitz optimization seems particularly challenging in the stochastic case and SQ model, in particular. Indeed, direct SQ implementation of some techniques would require queries of exponentially high accuracy. We give two approaches for dealing with this problem that require only that the convex functions in the support of distribution have bounded range. The first one avoids gradients altogether by only using estimates of function values. It is based on random walk techniques of \citenames{Kalai and Vempala}{KalaiV06} and \citenames{Lovasz and Vempala}{LovaszV06}. The second one is based on a new analysis of the classic center-of-gravity method. There we show that there exists a local norm, specifically that given by the inertial ellipsoid, that allows to obtain a global approximation relatively cheaply.
Interestingly, these very different methods have the same estimation complexity which is also within factor of $d$ of our lower bound.

Finally, we highlight some theoretical applications of our results. First, we describe a high-level methodology of obtaining lower bound for convex relaxations from our results and give an example for constraint satisfaction problems. We then show that our mean estimation algorithms can greatly improve estimation complexity of the SQ version of the classic Perceptron algorithm and several related algorithms.
Finally, we give corollaries for two problems in differential privacy: (a) new algorithms for solving convex programs with the stringent local differential privacy; (b) strengthening and generalization of algorithms for answering sequences of convex minimization queries differentially privately given by \citenames{Ullman}{Ullman15}.

\subsection{Linear optimization and mean estimation}
We start with the linear optimization case which is a natural special case and also the basis of our implementations of first-order methods. In this setting $\W \subseteq \R^d$ and $f(x,w) = \la x, w \ra$. Hence $F(x) = \la x, \bar{w} \ra$, where $\bar{w} = \E_{\bw}[\bw]$.  This reduces the problem to finding a sufficiently accurate estimate of $\bar{w}$. Specifically, for a given error parameter $\varepsilon$, it is sufficient to find a vector $\tilde{w}$, such that for every $x \in \K$, $|\la x, \bar{w} \ra - \la x, \tilde{w} \ra | \leq \varepsilon$. Given such an estimate $\tilde{w}$, we can solve the original problem with error of at most $2\varepsilon$ by solving $\min_{x\in \K} \la x, \tilde{w} \ra$.

An obvious way to estimate the high-dimensional mean using SQs is to simply estimate each of the coordinates of the mean vector using a separate SQ: that is $\E[\bw_i/B_i]$, where $[-B_i,B_i]$ is the range of $\bw_i$. Unfortunately, even in the most standard setting, where both $\K$ and $\W$ are $\ell_2$ unit balls, this method requires accuracy that scales with $1/\sqrt{d}$ (or estimation complexity that scales linearly with $d$). In contrast, bounds obtained using samples are dimension-independent making this SQ implementation unsuitable for high-dimensional applications. Estimation of high-dimensional means for various distributions is an even more basic question than stochastic optimization; yet we are not aware of any prior analysis of its statistical query complexity. In particular, SQ implementation of all algorithms for learning halfspaces (including the most basic Perceptron) require estimation of high-dimensional means but known analyses rely on inefficient coordinate-wise estimation  (\eg \cite{Bylander:94,BlumFKV:97,BalcanF13}).

The seemingly simple question we would like to answer is whether the SQ estimation complexity is different from the sample complexity of the problem. The first challenge here is that even the sample complexity of mean estimation depends in an involved way on the geometry of $\K$ and $\W$ 
(\cf \cite{Pisier:2011}). Also some of the general techniques for proving upper bounds on sample complexity (see App.~\ref{sec:Samples}) appeal directly to high-dimensional concentration and do not seem to extend to the intrinsically one-dimensional SQ model.
We therefore focus our attention on the much more benign and well-studied $\ell_{p}/\ell_q$ setting. That is $\K$ is a unit ball in $\ell_p$ norm and $\W$ is the unit ball in $\ell_q$ norm for $p\in [1,\infty]$ and $1/p+1/q=1$ (general radii can be reduced to this setting by scaling). This is equivalent to requiring that $\|\tilde{w} -\bar{w}\|_q \leq \varepsilon$ for a random variable $\bw$ supported on the unit $\ell_q$ ball and we refer to it as $\ell_q$ mean estimation.  Even in this standard setting the picture is not so clean in the regime when $q\in[1,2)$, where the sample complexity of $\ell_q$ mean estimation depends both on $q$ and the relationship between $d$ and $\varepsilon$.

In a nutshell, we give tight (up to a polylogarithmic in $d$ factor) bounds on the SQ complexity of $\ell_q$ mean estimation for all $q\in [1,\infty]$. These bounds match (up to a polylogarithmic in $d$ factor) the sample complexity of the problem. The upper bounds are based on several different algorithms.
\begin{itemize}
\item For $q=\infty$ straightforward coordinate-wise estimation gives the desired guarantees.
\item For $q=2$ we demonstrate that Kashin's representation of vectors introduced by \citenames{Lyubarskii and Vershynin}{Lyubarskii:2010} gives a set of $2d$ measurements which allow to recover the mean with estimation complexity of $O(1/\varepsilon^2)$.  We also give a randomized algorithm based on estimating the truncated coefficients of the mean in a randomly rotated basis. The algorithm has slightly worse $O(\log(1/\varepsilon)/\varepsilon^2)$ estimation complexity but its analysis is simpler and self-contained.
\item For $q \in (2,\infty)$ we use decomposition of the samples into $\log d$ ``rings" in which non-zero coefficients have low dynamic range. For each ring we combine $\ell_2$ and $\ell_\infty$ estimation to ensure low error in $\ell_q$ and nearly optimal estimation complexity.
\item For $q \in [1,2)$ substantially more delicate analysis is necessary. For large $\eps$ we first again use a decomposition into ``rings" of low dynamic range. For each ``ring" we use coordinate-wise estimation and then sparsify the estimate by removing small coefficients. The analysis requires using statistical queries in which accuracy takes into account the variance of the random variable (modeled by $\VSTAT$ oracle from \cite{FeldmanGRVX:12}). For small $\eps$ a better upper bound can obtained via a reduction to $\ell_2$ case.
\end{itemize}
The nearly tight lower bounds are proved using the technique recently introduced in \cite{FeldmanPV:13}. The lower bound also holds for the (potentially simpler) linear optimization problem. We remark that lower bounds on sample complexity do not imply lower bounds on estimation complexity since a SQ algorithm can use many queries.

We summarize the bounds in Table \ref{table:ellq_mean_est} and compare them with those achievable using samples (we provide the proof for those in Appendix \ref{sec:Samples} since we are not aware of a good reference for $q \in [1,2)$).
\begin{table}[h!] \label{table:ellq_mean_est}
\centering
\begin{tabular}{|M{1cm}|M{4.5cm}|M{4.5cm}|M{4cm}|N}
\hline
\(q\) & \multicolumn{2}{c|}{SQ estimation complexity}  & Sample  &\\ [8pt] \cline{2-3}

	& Upper Bound & Lower bound &  complexity  & \\ [8pt] \hline
$[1,2)$ &  $O\left(\min\left\{\frac{d^{\frac2q-1}}{\varepsilon^2}, \left(\frac{\log d}{\varepsilon}\right)^p \right\}\right)$ &
 $\tilde{\Omega}\left(\min\left\{\frac{d^{\frac2q-1}}{\varepsilon^2}, \frac{1}{\varepsilon^p\log d} \right\}\right)$
& $\Theta\left(\min\left\{\frac{d^{\frac2q-1}}{\varepsilon^2}, \frac{1}{\varepsilon^p} \right\}\right)$  & \\
[12pt] \hline

$2$            & $O(1/\varepsilon^2)$ &
 $\Omega(1/\varepsilon^2)$ &
$\Theta(1/\varepsilon^2)$ & \\ [12pt] \hline

$(2,\infty)$&  $O((\log d/\varepsilon)^2)$ &
 $\Omega(1/\varepsilon^2)$ &
$\Theta(1/\varepsilon^2)$ & \\ [12pt] \hline

$\infty$    &   $O(1/\varepsilon^2)$   &
 $\Omega(1/\varepsilon^2)$ &
$\Theta(\log d/\varepsilon^2)$ & \\ [12pt] \hline
\end{tabular}
\caption{Bounds on $\ell_q$ mean estimation and linear optimization over $\ell_p$ ball. Upper bounds use at most $3d\log d$ (non-adaptive) queries. Lower bounds apply to all algorithms using $\poly(d/\varepsilon)$ queries. Sample complexity is for algorithms with access to samples.}
\end{table}

We then briefly consider the case of general $\K$ with $\W = \conv(\K^*,-\K^*)$ (which corresponds to normalizing the range of linear functions in the support of the distribution). Here we show that for any polytope $\W$ the estimation complexity is still $O(1/\eps^2)$ but the number of queries grows linearly with the number of faces. More generally,
the estimation complexity of $O(d/\varepsilon^2)$ can be achieved for any $\K$. The algorithm relies on knowing John's ellipsoid \cite{John:1948} for $\W$ and therefore depends on $\K$. Designing a single algorithm that given a sufficiently strong oracle for $\K$ (such as a separation oracle) can achieve the same estimation complexity for all $\K$ is an interesting open problem (see Conclusions for a list of additional open problems). This upper bound is nearly tight since even for $\W$ being the $\ell_1$ ball we give a lower bound of $\tilde{\Omega}(d/\varepsilon^2)$.

\subsection{The Gradient Descent family}

The linear case gives us the basis for the study of the traditional setups of convex optimization for Lipschitz functions: non-smooth, smooth and strongly convex. In this setting we assume that for each $w$ in the support of the distribution $D$ and $x\in \K$, $\|\partial f(x,w)\|_q \leq L_0$ and the radius of $\K$ is bounded by $R$ in $\ell_p$ norm. The smooth and strongly convex settings correspond to second order assumptions on $F$ itself.  For the two first classes of problems, our algorithms use
global approximation of the gradient on $\K$ which as we know is necessary already in the linear case.
However, for the strongly convex case we can show that an oracle introduced by
\citenames{Devolder \etal}{Devolder:2014} only requires {\em local} approximation of the gradient, which
leads to improved estimation complexity bounds.

For the non-smooth case we analyze and apply the classic mirror-descent method \cite{nemirovsky1983problem}, for the smooth
case we rely on the analysis by \citenames{d'Aspremont}{dAspremont:2008} of an inexact variant of  Nesterov's accelerated method  \cite{nesterov1983method}, and for the strongly convex case we use the recent
results by \citenames{Devolder \etal}{Devolder2:2013} on the inexact dual gradient method.
We summarize our results for the $\ell_2$ norm in Table \ref{table:grad_meth_ell2}.
Our results for the mirror-descent and Nesterov's algorithm
apply in more general settings (e.g., $\ell_p$ norms): we refer the reader to Section \ref{sec:gradient} for the detailed statement of results. In \iffull Section \fi\ref{subsec:regression} we also demonstrate and discuss the implications of our results for the well-studied generalized linear regression problems.

\begin{table}[h!] \label{table:grad_meth_ell2}
\centering
\begin{tabular}{|M{2.5cm}|M{3cm}|M{4.5cm}|M{2.5cm}|N}
\hline
Objective & Inexact gradient method & Query complexity & Estimation complexity & \\ \hline

Non-smooth& Mirror-descent & $O\left(d\cdot\left(\frac{L_0R}{\varepsilon}\right)^2\right)$ &
$O\left(\left(\frac{L_0R}{\varepsilon}\right)^2\right)$ & \\ [12pt]\hline
Smooth	& Nesterov		&$O\left(d\cdot \sqrt{\frac{L_1R^2}{\varepsilon}}\right)$&
$O\left(\left(\frac{L_0R}{\varepsilon}\right)^2\right)$ & \\ [12pt]\hline
Strongly convex non-smooth & Dual gradient &
$O\left(d\cdot \frac{L_0^2}{\varepsilon\kappa} \log\left(\frac{L_0R}{\varepsilon}\right)\right)$ &
$O\left(\frac{L_0^2}{\varepsilon\kappa}\right)$ & \\ [12pt]\hline
Strongly convex smooth & Dual gradient &
$O\left(d \cdot \frac{L_1}{\kappa} \log\left(\frac{L_1R}{\varepsilon}\right)\right)$ &
$O\left(\frac{L_0^2}{\varepsilon\kappa}\right)$ & \\ [12pt]\hline
\end{tabular}
\caption{Upper bounds for inexact gradient methods in
the stochastic $\ell_2$-setup. Here $R$ is the Euclidean radius of the
domain, $L_0$ is the Lipschitz constant of all functions in the support of the distribution.
$L_1$ is the Lipschitz constant of the gradient and $\kappa$ is the strong
convexity parameter for the expected objective.}
\end{table}

It is important to note that, unlike in the linear case, the SQ algorithms for optimization of general convex functions are adaptive.
In other words, the SQs being asked at step $t$ of the iterative algorithm depend on the answers to queries in previous steps. This means that the number of samples that would be necessary to implement such SQ algorithms is no longer easy to determine. In particular, as demonstrated by \citenames{Dwork \etal}{DworkFHPRR14:arxiv}, the number of samples needed for estimation of adaptive SQs using empirical means might scale linearly with the query complexity. While better bounds can be easily achieved in our case (logarithmic --as opposed to linear-- in dimension), they are still worse than the sample complexity. We are not aware of a way to bridge this intriguing gap or prove that it is not possible to answer the SQ queries of these algorithms with the same sample complexity.

Nevertheless, estimation complexity is a key parameter even in the adaptive case. There are many other settings in which one might be interested in implementing answers to SQs and in some of those the complexity of the implementation depends on the estimation complexity and query complexity in other ways (for example, differential privacy). In a number of lower bounds for SQ algorithm (including those in Sec.~\ref{sec:lower-linear}) there is a threshold phenomenon in which as one goes below certain estimation complexity, the query complexity lower bound grows from polynomial to exponential very quickly (\eg \cite{FeldmanGRVX:12,FeldmanPV:13}). For such lower bounds only the estimation complexity matters as long as the query complexity of the algorithm is polynomial.

\subsection{Non-Lipschitz Optimization}
The estimation complexity bounds obtained for gradient descent-based methods depend polynomially on the Lipschitz constant $L_0$ and the radius $R$ (unless $F$ is strongly convex). In some cases such bounds are too large and we only have a bound on the range of $f(x,w)$ for all $w \in \W$ and $x \in \K$ (note that a bound of $L_0R$ on range is also implicit in the Lipschitz setting). This is a natural setting for stochastic optimization (and statistical algorithms, in particular) since even estimating the value of a given solution $x$ with high probability and any desired accuracy from samples requires some assumptions about the range of most functions. 

For simplicity we will assume $|f(x,w)| \leq B=1$, although our results can be extended to the setting where only the variance of $f(x,\bw)$ is bounded by $B^2$ using the technique from \cite{Feldman:16sqvar}. Now, for every $x\in \K$, a single SQ for function $f(x,w)$ with  tolerance $\tau$ gives a value $\tilde{F}(x)$ such that $|F(x) - \tilde{F}(x)| \leq \tau$. This, as first observed by \citenames{Valiant}{Valiant14}, gives a $\tau$-approximate value (or zero-order) oracle for $F(x)$.
It was proved by \citenames{Nemirovsky and Yudin}{nemirovsky1983problem} and also by \citenames{Gr\"{o}tschel \etal}{GroetschelLS88} (who refer to such oracle as {\em weak evaluation oracle}) that $\tau$-approximate value oracle suffices to $\varepsilon$-minimize $F(x)$ over $\K$ with running time and $1/\tau$ being polynomial in $d, 1/\varepsilon, \log (R_1/R_0)$, where $\B_2^d(R_0) \subseteq \K \subseteq \B_2^d(R_1)$.
The analysis in \cite{nemirovsky1983problem,GroetschelLS88} is relatively involved and does not provide explicit bounds on $\tau$.

Here we substantially sharpen the understanding of optimization with approximate value oracle. Specifically, we show that $(\varepsilon/d)$-approximate value oracle for $F(x)$ suffices to $\varepsilon$-optimize in polynomial time.
\begin{thm}
\label{thm:random-walk-zero-intro}
There is an algorithm that with probability at least $2/3$, given any convex program $\min_{x \in \K} F(x)$ in $\R^d$ where $\forall x\in \K,\ |F(x)| \leq 1$ and $\K$ is given by a membership oracle with the guarantee that $ \B_2^d(R_0) \subseteq \K \subseteq \B_2^d(R_1)$, outputs an $\eps$-optimal solution in time $\poly(d, \frac{1}{\eps}, \log{(R_1/R_0)})$ using $\poly(d, \frac{1}{\eps})$ queries to $\Omega(\eps/d)$-approximate value oracle.
\end{thm}
We outline a proof of this theorem which is based on an extension of the random walk approach of \citenames{Kalai and Vempala}{KalaiV06} and \citenames{Lovasz and Vempala}{LovaszV06}. This result was also independently obtained in a recent work of \citenames{Belloni \etal}{BelloniLNR15} who provide a detailed analysis of the running time and query complexity. 

It turns out that the dependence on $d$ in the tolerance parameter of this result cannot be removed altogether: \citenames{Nemirovsky and Yudin}{nemirovsky1983problem} prove that even linear optimization over $\ell_2$ ball of radius 1 with a $\tau$-approximate value oracle requires $\tau = \tilde \Omega(\eps/\sqrt{d})$ for any polynomial-time algorithm. This
result also highlights the difference between SQs and approximate value oracle since the problem can be solved using SQs of tolerance $\tau=O(\eps)$. Optimization with value oracle is also substantially more challenging algorithmically.

Luckily, SQs are not constrained to the value information and we give a substantially simpler and more efficient algorithm for this setting. Our algorithm is based on the classic center-of-gravity method with a crucial new observation: in every iteration the inertial ellipsoid, whose center is the center of gravity of the current body, can be used to define a (local) norm in which the gradients can be efficiently approximated globally. The exact center of gravity and inertial ellipsoid cannot be found efficiently and the efficiently implementable Ellipsoid method does not have the desired local norm. However\iffull, \fi we show that the approximate center-of-gravity method introduced by \citenames{Bertsimas and Vempala}{Bertsimas:2004} and approximate computation of the inertial ellipsoid \cite{LovaszV06b} suffice for our purposes.
\begin{theorem}[Informal]
\label{thm:cog-sq-efficient-intro}
Let $\K\subseteq \R^d$ be a convex body given by a membership oracle $\B_2^d(R_0) \subseteq \K \subseteq \B_2^d(R_1)$, and assume that for all $w\in\W, x\in \K$, $|f(x,w)|\leq 1$. Then there is a randomized algorithm that for every distribution $D$ over $\W$ outputs an $\eps$-optimal solution using $O(d^2\log(1/\varepsilon))$  statistical queries with tolerance $\Omega(\varepsilon/d)$ and runs in $\poly(d, 1/\eps, \log(R_1/R_0))$ time.
\end{theorem}
Closing the gap between the tolerance of $\eps/\sqrt{d}$ in the lower bound (already for the linear case) and the tolerance of $\eps/d$ in the upper bound is an interesting open problem. Remarkably, as Thm.~\ref{thm:random-walk-zero-intro} and the lower bound in \cite{nemirovsky1983problem} show, the same intriguing gap is also present for approximate value oracle.


\subsection{Applications}
We now highlight several applications of our results. Additional results can be easily derived in a variety of other contexts that rely on statistical queries (such as evolvability \cite{Valiant:09}, adaptive data analysis \cite{DworkFHPRR14:arxiv} and distributed data analysis \cite{ChuKLYBNO:06}).

\subsubsection{Lower Bounds}
The statistical query framework provides a natural way to convert algorithms into lower bounds. For many problems over distributions it is possible to prove information-theoretic lower bounds against statistical algorithms that are much stronger than known computational lower bounds for the problem. A classical example of such problem is learning of parity functions with noise (or, equivalently, finding an assignment that maximizes the fraction of satisfied XOR constraints). This implies that any algorithm that can be implemented using statistical queries with complexity below the lower bound cannot solve the problem. If the algorithm relies solely on some structural property of the problem, such as approximation of functions by polynomials or computation by a certain type of circuit, then we can immediately conclude a lower bound for that structural property. This indirect argument exploits the power of the algorithm and hence can lead to results which are harder to derive directly.

One inspiring example of this approach comes from using the statistical query algorithm for learning halfspaces \cite{BlumFKV:97}. The structural property it relies on is linear separability. Combined with the exponential lower bound for learning parities \cite{Kearns:98}, it immediately implies that there is no mapping from $\on^d$ to $\R^N$ which makes parity functions linearly separable for any $N\leq N_0=2^{\Omega(d)}$. Subsequently, and apparently unaware of this technique, \citenames{Forster}{Forster:02} proved a $2^{\Omega(d)}$ lower bound on the sign-rank (also known as the dimension complexity) of the Hadamard matrix which is exactly the same result (in \cite{Sherstov:08} the connection between these two results is stated explicitly). His proof relies on a sophisticated and non-algorithmic technique and is considered a major breakthrough in proving lower bounds on the sign-rank of explicit matrices.

Convex optimization algorithms rely on existence of convex relaxations for problem instances that (approximately) preserve the value of the solution. Therefore given a SQ lower bound for a problem, our algorithmic results can be directly translated into lower bounds for convex relaxations of the problem.
We now focus on a concrete example that is easily implied by our algorithm and a lower bound for planted constraint satisfaction problems from \cite{FeldmanPV:13}. Consider the task of distinguishing a random satisfiable $k$-SAT formula over $n$ variables of length $m$ from a randomly and uniformly drawn $k$-SAT formula of length $m$. This is the refutation problem studied extensively over the past few decades (\eg \cite{feige2002relations}). Now, consider the following common approach to the problem: define a convex domain $\K$ and map every $k$-clause $C$ ( (OR of $k$ distinct variables or their negations) to a convex function $f_C$ over $\K$ scaled to the range $[-1,1]$. Then, given a formula $\phi$ consisting of clauses $C_1,\ldots,C_m$, find $x$ that minimizes $F_\phi(x) = \frac{1}{m}\sum_i f_{C_i}(x)$ which roughly measures the fraction of unsatisfied clauses (if $f_C$'s are linear then one can also maximize $F(x)$ in which case one can also think of the problem as satisfying the largest fraction of clauses). The goal of such a relaxation is to ensure that for every satisfiable $\phi$ we have that $\min_{x\in\K} F_\phi(x) \leq \alpha$ for some fixed $\alpha$. At the same time for a randomly chosen $\phi$, we want to have with high probability  $\min_{x\in\K} F_\phi(x) \geq  \alpha+ \eps$. Ideally one would hope to get $\eps \approx 2^{-k}$ since for sufficiently large $m$, every Boolean assignment leaves at least $\approx 2^{-k}$ fraction of the constraints unsatisfied. But the relaxation can reduce the difference to a smaller value.

We now plug in our algorithm for $\ell_p/\ell_q$ setting to get the following broad class of corollaries.
\begin{cor}
\label{cor:lower-convex-program-norm}
For $p\in \{1,2\}$, let ${\cal K}\subseteq \B_p^d$ be a convex body and $\F_p = \left\{f(\cdot) \cond \forall x \in \K, \|\nabla f(x)\|_q \leq 1\right\}$. Assume that there exists a mapping that maps each $k$-clause $C$ to a convex function $f_C \in \F_p$. Further assume that for some $\eps > 0$:
If $\phi=  C_1,\ldots,C_m$ is satisfiable then $$\min_{x\in \K}\left\{\fr{m} \sum_i f_{C_i}(x)  \right\} \leq 0.$$
Yet for the uniform distribution $U_k$ over all the $k$-clauses:
$$\min_{x\in {\cal K}} \left\{\E_{C\sim U_k} \lb f_{C}(x) \rb \right\} > \eps.$$
Then $d = 2^{\tilde{\Omega}(n  \cdot \eps^{2/k})}$.
\end{cor}
Note that the second condition is equivalent to applying the relaxation to the formula that includes all the $k$-clauses. Also for every $m$, it is implied by the condition
$$\E_{C_1,\ldots,C_m\sim U_k}\left[\min_{x\in {\cal K}} \left\{\fr{m} \sum_i f_{C_i}(x) \right\}\right] > \eps .$$

As long as $k$ is a constant and $\eps =\Omega_k(1)$ we get a lower bound of $2^{\Omega(n)}$ on the dimension of any convex relaxation (where the radius and the Lipschitz constant are at most 1). We are not aware of any existing techniques that imply  comparable lower bounds. More importantly, our results imply that Corollary \ref{cor:lower-convex-program-norm} extends to a very broad class of general state-of-the-art approaches to stochastic convex optimization.

Current research focuses on the linear case and restricted $\K$'s which are obtained through various hierarchies of LP/SDP relaxations or extended formulations(\eg \cite{Schoenebeck:08}). The primary difference between the relaxations used in this line of work and our approach is that our approach only rules out relaxations for which the resulting stochastic convex program can be solved by a statistical algorithm. On the other hand, stochastic convex programs that arise from LP/SDP hierarchies and extended formulations cannot, in general, be solved given the available number of samples (each constraint is a sample). As a result, the use of such relaxations can lead to overfitting and this is the reason why these relaxations fail.
This difference makes our lower bounds incomparable and, in a way, complementary to existing work on lower bounds for specific hierarchies of convex relaxations. For a more detailed discussion of SQ lower bounds, we refer the reader to \cite{FeldmanPV:13}.

\subsubsection{Online Learning of Halfspaces using SQs}
Our high-dimensional mean estimation algorithms allow us to revisit SQ implementations of online algorithms for learning halfspaces, such as the classic Perceptron and Winnow algorithms. These algorithms are based on updating the weight vector iteratively using incorrectly classified examples. The convergence analysis of such algorithms relies on some notion of margin by which positive examples can be separated from the negative ones.

A natural way to implement such an algorithm using SQs is to use the mean vector of all positive (or negative) counterexamples to update the weight vector. By linearity of expectation, the true mean vector is still a positive (or correspondingly, negative) counterexample and it still satisfies the same margin condition. This approach was used by \citenames{Bylander}{Bylander:94} and \citenames{Blum \etal}{BlumFKV:97} to obtain algorithms tolerant to random classification noise for learning halfspaces and by \citenames{Blum \etal}{BlumDMN:05} to obtain a private version of Perceptron. The analyses in these results use the simple coordinate-wise estimation of the mean and incur an additional factor $d$ in their sample complexity. It is easy to see that to approximately preserve the margin $\gamma$ it suffices to estimate the mean of some distribution over an $\ell_q$ ball with $\ell_q$ error of $\gamma/2$.  We can therefore plug our mean estimation algorithms to eliminate the dependence on the dimension from these implementations (or in some cases have only logarithmic dependence). In particular, the estimation complexity of our algorithms is essentially the same as the sample complexity of PAC versions of these online algorithms.  Note that such improvement is particularly important since Perceptron is usually used with a kernel (or in other high-dimensional space) and Winnow's main property is the logarithmic dependence of its sample complexity on the dimension.

We note that a variant of the Perceptron algorithm referred to as Margin Perceptron outputs a halfspace that approximately maximizes the margin \cite{BalcanB06}. This allows it to be used in place of the SVM algorithm. Our SQ implementation of this algorithm gives an SVM-like algorithm with estimation complexity of $O(1/\gamma^2)$, where $\gamma$ is the (normalized) margin. This is the same as the sample complexity of SVM (\cf \cite{Shalev-ShwartzBen-David:2014}). Further details of this application are given in Sec.\ref{sec:halfspaces}.

\subsubsection{Differential Privacy}
In local or {\em randomized-response} differential privacy the users provide the analyst with differentially private versions of their data points. Any analysis performed on such data is differentially private so, in effect, the data analyst need not be trusted. Such algorithms have been studied and applied for privacy preservation since at least the work of \citenames{Warner}{Warner65} and have more recently been adopted in products by Google and Apple. While there exists a large and growing literature on mean estimation and convex optimization with (global) differential privacy (\eg \cite{ChaudhuriMS11,DworkRoth:14,BassilyST14}), these questions have been only recently and partially addressed for the more stringent local privacy.
Using simple estimation of statistical queries with local differential privacy by \citenames{Kasiviswanathan \etal}{KasiviswanathanLNRS11} we directly obtain a variety of corollaries for locally differentially private mean estimation and optimization. Some of them, including mean estimation for $\ell_2$ and $\ell_{\infty}$ norms and their implications for gradient and mirror descent algorithms are known via specialized arguments \cite{DuchiJW:13focs,DuchiJW14}. Our corollaries for mean estimation achieve the same bounds up to logarithmic in $d$ factors. We also obtain corollaries for more general mean estimation problems and results for optimization that, to the best of our knowledge, were not previously known.

An additional implication in the context of differentially private data analysis is to the problem of releasing answers to multiple queries over a single dataset. A long line of research has considered this question for {\em linear} or {\em counting} queries which for a dataset $S \subseteq \W^n$ and function $\phi:\W\rar [0,1]$ output an estimate of $\frac{1}{n}\sum_{w \in S} \phi(w)$ (see \cite{DworkRoth:14} for an overview). In particular, it is known that an exponential in $n$ number of such queries can be answered differentially privately even when the queries are chosen adaptively \cite{RothR10,HardtR10} (albeit the running time is linear in $|\W|$). Recently, \citenames{Ullman}{Ullman15} has considered the question of answering {\em convex minimization} queries which ask for an approximate minimum of a convex program taking a data point as an input averaged over the dataset. For several convex minimization problems he gives algorithms that can answer an exponential number of convex minimization queries. It is easy to see that the problem considered by \citenames{Ullman}{Ullman15} is a special case of our problem by taking the input distribution to be uniform over the points in $S$. A statistical query for this distribution is equivalent to a counting query and hence our algorithms effectively reduce answering of convex minimization queries to answering of counting queries. As a corollary we strengthen and substantially generalize the results in \cite{Ullman15}.

Details of these applications appear in Sections \ref{sec:app-dp-local} and \ref{sec:app-dp-queries}.

\subsection{Related work}
There is a long history of research on the complexity of convex optimization with access to some type of oracle (\eg \cite{nemirovsky1983problem,Braun:2014,Guzman:2015}) with a lot of renewed interest due to applications in machine learning (\eg \cite{Raginsky:2011,Agarwal:2012}). In particular, a number of works study robustness of optimization methods to errors by considering oracles that provide approximate information about $F$ and its (sub-)gradients \cite{dAspremont:2008,Devolder:2014}. Our approach to getting statistical query algorithms for stochastic convex optimization is based on both establishing bridges to that literature and also on improving state-of-the art for such oracles in the non-Lipschitz case.

A common way to model stochastic optimization is via a stochastic oracle for the objective function \cite{nemirovsky1983problem}. Such oracle is assumed to return a random variable whose expectation is equal to the exact value of the function and/or its gradient (most commonly the random variable is Gaussian or has bounded variance). Analyses of such algorithms (most notably Stochastic Gradient Descent (SGD)) are rather different from ours although in both cases linearity and robustness properties of first-order methods are exploited. In most settings we consider, estimation complexity of our SQ agorithms is comparable to sample complexity of solving the same problem using an appropriate version of SGD (which is, in turn, often known to be optimal). On the other hand lower bounds for stochastic oracles (\eg \cite{Agarwal:2012}) have a very different nature and it is impossible to obtain superpolynomial lower bounds on the number of oracle calls (such as those we prove in Section \ref{sec:lower-linear}).

SQ access is known to be equivalent (up to polynomial factors) to the setting in which the amount of information extracted from (or communicated about) each sample is limited \cite{Ben-DavidD98,FeldmanGRVX:12,FeldmanPV:13}.
In a recent (and independent) work \citenames{Steinhardt \etal}{SteinhardtVW:2016} have established a number of additional relationships between learning with SQs and learning with several types of restrictions on memory and communication. Among other results, they proved an unexpected upper bound on memory-bounded sparse least-squares regression by giving an SQ algorithm for the problem. Their analysis\footnote{The analysis and bounds they give are inaccurate but a similar conclusion follows from the bounds we give in Cor.\ref{cor:solve_cvx_ellp}.} is related to the one we give for inexact mirror-descent over the $\ell_1$-ball. Note that in optimization over $\ell_1$ ball, the straightforward coordinate-wise $\ell_\infty$ estimation of gradients suffices. Together with their framework our results can be easily used to derive low-memory algorithms for other learning problems.

\section{Preliminaries}
For integer $n\geq 1$ let $[n]\doteq \{1,\ldots, n\}$. Typically, $d$ will denote the ambient space dimension,
and $n$ will denote number of samples. Random variables are denoted by bold
letters, e.g., $\bw$, $\bU$. We denote the indicator function of an event $A$
(i.e., the function taking value zero outside of $A$, and one on $A$) by
${\mathbf 1}_{A}$.

For $i\in[d]$ we denote by $e_i$ the $i$-th basis
vector in $\R^d$. Given a norm $\|\cdot\|$ on $\R^d$ we denote the ball of radius $R>0$ by
$\B_{\|\cdot\|}^d(R)$, and the unit ball by $\B_{\|\cdot\|}^d$. We also recall the definition of the norm dual to $\|\cdot\|$, $\|w\|_{\ast}\doteq \sup_{\|x\|\leq1}\langle w,x\rangle$, where
$\langle \cdot,\cdot\rangle $ is the standard inner product of $\R^d$.

For a convex body (i.e., compact convex set with nonempty interior) $\K\subseteq \R^d$
we define its polar as $\K_{\ast}=\{w\in\R^d:\, \langle w,x\rangle\leq 1 \,\,\forall x\in\K  \}$,
and we have that $(\K_{\ast})_{\ast}=\K$.
Any origin-symmetric convex body $\K \subset \R^d$ (i.e., $\K=-\K$) defines a norm
$\|\cdot \|_\K$ as
follows: $\| x \|_\K = \inf_{\alpha > 0}\{\alpha \cond x/\alpha \in \K\}$, and $\K$ is the
unit ball of $\|\cdot\|_\K$. It is easy to see that
the norm dual to $\|\cdot\|_{\cal K}$ is $\|\cdot\|_{\cal K_{\ast}}$.

Our primary case of interest corresponds to $\ell_p$-setups. Given $1\leq p\leq \infty$,
we consider the normed space $\ell_p^d\doteq(\R^d,\|\cdot\|_p)$, where for a vector
$x\in\R^d$, $\|x\|_p\doteq \left(\sum_{i\in[d]}|x_i|^p \right)^{1/p}$. For
$R \geq 0$, we denote by $\B_p^d(R) = \B_{\|\cdot\|_p}^d(R)$ and similarly
for the unit ball, \(\B_p^d=\B_p^d(1)\). We denote the conjugate exponent of $p$
as $q$, meaning that $1/p+1/q=1$; with this, the norm dual to $\|\cdot\|_p$ is the norm
$\|\cdot\|_q$. In all definitions above, when clear from context, we will omit the
dependence on $d$.\\

We consider problems of the form
\begin{equation} \label{StochOpt}
F^{\ast} \doteq \min_{x\in \K}\left\{ F(x)\doteq\E_{\bw}[f(x,\bw)]\right\},
\end{equation}
where $\K$ is a convex body in $\R^d$, $\bw$ is a random variable defined over some domain $\W$, and for each $w\in\W$, $f(\cdot,w)$ is convex and subdifferentiable
on $\K$. For an approximation parameter $\eps>0$ the goal is to find $x\in{\cal K}$ such that $F(x) \leq F^* +\eps$, and we call any such $x$ an {\em$\varepsilon$-optimal solution}. We denote the probability distribution of $\bw$ by $D$ and refer to it as the input distribution. For convenience we will also assume that $\K$ contains the origin.

\paragraph{Statistical Queries:}
The algorithms we consider here have access to a statistical query oracle for the input distribution. For most of our results a basic oracle introduced by \citenames{Kearns}{Kearns:98} that gives an estimate of the mean with fixed tolerance will suffice.  We will also rely on a stronger oracle from \cite{FeldmanGRVX:12} that takes into the account the variance of the query function and faithfully captures estimation of the mean of a random variable from samples.
\begin{definition}
 Let $D$ be a distribution over a domain $\W$, $\tau >0$ and $n$ be an integer. A statistical query oracle $\STAT_D(\tau)$ is an oracle that given as input any function $\phi : \W \rightarrow [-1,1]$, returns some value $v$ such that
 $|v -  \E_{\bw\sim D}[\phi(\bw)]| \leq \tau$. A statistical query oracle $\VSTAT_D(n)$ is an oracle that given as input any function $\phi : \W \rightarrow [0,1]$ returns some value $v$ such that
 $|v -  p| \leq \max\left\{\frac{1}{n}, \sqrt{\frac{p(1-p)}{n}}\right\}$, where $p \doteq \E_{\bw\sim D}[\phi(\bw)]$.
We say that an algorithm is {\em statistical query} (or, for brevity, just SQ) if it does not have direct access
 to $n$ samples from the input distribution $D$, but instead makes calls to a statistical query oracle for the input distribution.
\end{definition}
Clearly $\VSTAT_D(n)$ is at least as strong as $\STAT_D(1/\sqrt{n})$ (but no stronger than $\STAT_D(1/n)$).
Query complexity of a statistical algorithm is the number of queries it uses. The {\em estimation complexity} of a statistical query algorithm using $\VSTAT_D(n)$ is the value $n$ and for an algorithm using $\STAT(\tau)$ it is $n=1/\tau^2$.
Note that the estimation complexity corresponds to the number of i.i.d.~samples sufficient to simulate the oracle for a single query with at least some positive constant probability of success. However it is not necessarily true that the whole algorithm can be simulated using $O(n)$ samples since answers to many queries need to be estimated. Answering $m$ fixed (or non-adaptive) statistical queries can be done using $O(\log m \cdot n)$ samples but when queries depend on previous answers the best known bounds require $O(\sqrt{m} \cdot n)$ samples (see \cite{DworkFHPRR14:arxiv} for a detailed discussion). This also implies that a lower bound on sample complexity of solving a problem does not directly imply lower bounds on estimation complexity of a SQ algorithm for the problem.

Whenever that does not make a difference for our upper bounds on estimation complexity, we state results for $\STAT$ to ensure consistency with prior work in the SQ model. All our lower bounds are stated for the stronger $\VSTAT$ oracle. One useful property of $\VSTAT$ is that it only pays linearly when estimating expectations of functions conditioned on a rare event:
\begin{lemma}
\label{lem:vstat-condition}
For any function $\phi : X \rightarrow [0,1]$,  input distribution $D$ and  condition $A:X \rightarrow \zo$ such that $p_A \doteq \pr_{x\sim D}[A(x)=1] \geq \alpha$, let $p \doteq \E_{x\sim D}[\phi(x) \cdot A(x)]$. Then query $\phi(x)\cdot A(x)$ to $\VSTAT(n/\alpha)$ returns a value $v$ such that $|v - p| \leq \frac{p_A}{\sqrt{n}}$.
\end{lemma}
\begin{proof}
The value  $v$ returned by $\VSTAT(n/\alpha)$ on query $\phi(x)\cdot A(x)$ satisfies:
$|v - p| \leq \min\left\{\frac{\alpha}{n}, \sqrt{\frac{p(1-p)\alpha}{n}}\right\}$. Note that $p = \E[\phi(x) A(x)] \leq \pr[A(x)=1] = p_A$. Hence $|v-p| \leq \frac{p_A}{\sqrt{n}}$.
\end{proof}
Note that one would need to use  $\STAT(\alpha/\sqrt{n})$ to obtain a value $v$ with the same accuracy of $\frac{p_A}{\sqrt{n}}$ (since $p_A$ can be as low as $\alpha$). This corresponds to estimation complexity of $n/\alpha^2$ vs.~$n/\alpha$ for $\VSTAT$.

\section{Stochastic Linear Optimization and Vector Mean Estimation}
\label{sec:linear}
We start by considering stochastic linear optimization, that is instances of the problem
$$\min_{x\in \K}\{ \E_{\bw}[f(x,\bw)]\} $$
in which $f(x,w) = \la x,w\ra$. From now on we will use the notation
$\bar w \doteq \E_{\bw}[\bw]$.

For normalization purposes we will assume that the random variable $\bw$ is supported on $\W = \{ w \cond \forall x \in \K,\ |\la x,w\ra| \leq 1\}$. Note that $\W = \conv(\K_*,-\K_*)$ and if $\K$ is origin-symmetric then $\W = \K_*$. More generally, if $\bw$ is supported on $\W$ and $B \doteq \sup_{x \in \K,\ w \in \W}\{ |\la x,w\ra|\}$ then optimization with error $\varepsilon$ can be reduced to optimization with error $\varepsilon/B$ over the normalized setting by scaling.

We first observe that for an origin-symmetric $\K$, stochastic linear optimization with error $\varepsilon$ can be solved by estimating the mean vector $\E[\bw]$ with error $\varepsilon/2$ measured in $\K_*$-norm and then optimizing a deterministic objective.
\begin{observation} \label{obs:lin_opt_mean_est}
Let $\W$ be an origin-symmetric convex body and $\K \subseteq \W_*$. Let $\min_{x\in \K}\{F(x) \doteq  \E[ \la x,\bw\ra]\}$ be an instance of stochastic linear optimization for $\bw$ supported on $ \W$. Let $\tilde{w}$ be a vector such that $\|\tilde{w} - \bar{w}\|_{\W} \leq \varepsilon/2$. Let $\tilde{x}\in \K$ be such that $\langle \tilde x, \tilde w\rangle \leq \min _{x\in \K} \la x,\tilde{w} \ra +\xi$. Then for all $x \in \K$, $F(\tilde{x}) \leq F(x) +\varepsilon + \xi$.
\end{observation}
\begin{proof}
Note that $F(x) = \la x, \bar{w} \ra$ and let $\bar{x} = \argmin _{x\in \K} \la x, \bar{w} \ra$. The condition $\|\tilde{w} - \bar{w}\|_{\W} \leq \varepsilon/2$ implies that for every $x \in \W_*$, $|\la x,\tilde{w}- \bar{w} \ra | \leq \varepsilon/2$.
Therefore, for every $x \in \K$,
$$F(\tilde{x}) = \la \tilde{x},\bar{w} \ra  \leq \la \tilde{x},\tilde{w} \ra +\varepsilon/2 \leq \la \bar{x},\tilde{w} \ra +\varepsilon/2 +\xi \leq \la \bar{x},\bar{w} \ra +\varepsilon +\xi \leq \la x,\bar{w} \ra +\varepsilon +\xi  = F(x) +\varepsilon+\xi .$$
\end{proof}

The mean estimation problem over $\W$ in norm $\| \cdot \|$ is the problem in which, given an error parameter $\varepsilon$ and access to a distribution $D$ supported over $\W$, the goal is to find a vector $\tilde{w}$ such that $\|\E_{\bw \sim D}[\bw] - \tilde{w} \| \leq \varepsilon$. We will be concerned primarily with the case when $\W$ is the unit ball of $\| \cdot \|$ in which case we refer to it as $\| \cdot \|$ mean estimation or mean estimation over $\W$.

We also make a simple observation that if a norm $\| \cdot \|_A$ can be embedded via a linear map into a norm $\| \cdot \|_B$ (possibly with some distortion) then we can reduce mean estimation in $\| \cdot \|_A$ to mean estimation in $\| \cdot \|_B$.
\begin{lemma}
\label{lem:norm-embed}
Let $\| \cdot \|_{A}$ be a norm over $\R^{d_1}$ and $\| \cdot \|_{B}$ be a norm over $\R^{d_2}$  that for some linear map $T:\R^{d_1} \rightarrow \R^{d_2}$  satisfy: $\forall w \in \R^{d_1}$, $a \cdot \|Tw\|_{B}  \leq \|w\|_{A} \leq b \cdot \|Tw\|_{B}$. Then mean estimation in $\| \cdot \|_{A}$ with error $\varepsilon$ reduces to mean estimation in $\| \cdot \|_{B}$ with error $\frac{a}{2b}\varepsilon$ (or error $\frac{a}{b}\varepsilon$ when $d_1 =d_2$).
\begin{proof}
Suppose there exists an statistical algorithm $\A$ that for any input distribution  supported on $\B_{\|\cdot\|_B}$ computes $\tilde z\in\R^{d_2}$ satisfying $\|\tilde z-\E_{\bz}[\bz]\|_B \leq \frac{a}{2b}\varepsilon$.

Let $D$ be the target distribution on $\R^{d_1}$, which is supported on $\B_{\|\cdot\|_A}$.
We use $\A$ on the image of $D$ by $T$, multiplied by $a$. That is, we replace each query $\phi:\R^{d_2} \rightarrow \R$ of $\A$ with query $\phi'(w) = \phi(a\cdot Tw)$.
Notice that by our assumption, $\|a\cdot Tw\|_B \leq \|w\|_{A} \leq 1$.
Let $\tilde y$ be the output of $\A$ divided by $a$. By linearity, we have that $\|\tilde y-T\bar w\|_B\leq \frac{1}{2b}\varepsilon$.
Let $\tilde w$ be any vector such that $\|\tilde y - T \tilde w\|_{B}\leq \frac{1}{2b}\varepsilon$.
Then,
$$\|\tilde w-\bar w\|_A\leq b\|T\tilde w-T\bar w\|_B\leq b\|\tilde y-T\tilde w\|_B+
b \|\tilde y-T\bar w\|_B \leq \varepsilon. $$
Note that if $d_1 = d_2$ then $T$ is invertible and we can use $\tilde w = T^{-1}\tilde y$.
\end{proof}
\end{lemma}
\begin{remark}
The reduction of Lemma \ref{lem:norm-embed} is computationally efficient when the following two tasks can
be performed efficiently: computing $Tw$ for any input $w$, and given $z\in\R^{d_2}$ such that there exists
$ w'\in\R^{d_1}$ with $\|z-Tw'\|_B\leq \delta$, computing $w$ such that
$\|z-Tw\|_B\leq \delta+\xi$, for some precision $\xi = O(\delta)$.

\end{remark}

An immediate implication of this is that if the Banach-Mazur distance between unit balls of two norms $\W_1$ and $\W_2$ is $r$ then mean estimation over $\W_1$ with error $\varepsilon$ can be reduced to mean estimation over $\W_2$ with error $\varepsilon/r$.

\subsection{$\ell_q$ Mean Estimation} \label{Subsec:L_q}
We now consider stochastic linear optimization over $\B_p^d$ and the corresponding $\ell_q$ mean estimation problem.
We first observe that for $q = \infty$ the problem can be solved by directly using coordinate-wise statistical queries with tolerance $\varepsilon$. This is true since each coordinate has range $[-1,1]$ and for an estimate $\tilde{w}$ obtained in this way we have $\|\tilde w -\bar w\|_\infty = \max_i\{ |\tilde{w}_i - \E[\bw_i]\} \leq \varepsilon$.
\begin{theorem}
\label{thm:L-infty}
$\ell_{\infty}$ mean estimation problem with error $\varepsilon$ can be efficiently solved using $d$ queries to $\STAT(\varepsilon)$.
\end{theorem}

A simple application of Theorem \ref{thm:L-infty} is to obtain an algorithm for $\ell_1$ mean estimation. Assume that $d$ is a power of two and let $H$ be the orthonormal Hadamard transform matrix (if $d$ is not a power of two we can first pad the input distribution to to $\R^{d'}$, where $d' = 2^{\lceil\log d \rceil} \leq 2d$). Then it is easy to verify that for every $w \in \R^d$, $\|Hw\|_\infty \leq \|w\|_1 \leq \sqrt{d} \|Hw\|_\infty$. By Lemma \ref{lem:norm-embed} this directly implies the following algorithm:
\begin{theorem}
\label{thm:L1}
$\ell_1$ mean estimation problem with error $\varepsilon$ can be efficiently solved using $2d$ queries to $\STAT(\varepsilon/\sqrt{2d})$.
\end{theorem}

We next deal with an important case of $\ell_2$ mean estimation. It is not hard to see that using statistical queries for direct coordinate-wise estimation will require estimation complexity of  $\Omega(d/\varepsilon^2)$.
We describe two algorithms for this problem with (nearly) optimal estimation complexity. The first one relies on so called Kashin's representations introduced by \citenames{Lyubarskii and Vershynin}{Lyubarskii:2010}.
The second is a simpler but slightly less efficient method based on truncated coordinate-wise estimation in a randomly rotated basis.
\subsubsection{$\ell_2$ Mean Estimation via Kashin's representation}
A Kashin's representation is a representation of a vector in an overcomplete linear system such that the magnitude of each coefficient is small (more precisely, within a constant of the optimum) \cite{Lyubarskii:2010}. Such representations, also referred to as ``democratic", have a variety of applications including vector quantization and peak-to-average power ratio reduction in communication systems (\cf \cite{StuderGYB14}). We show that existence of such representation leads directly to SQ algorithms for $\ell_2$ mean estimation.

We start with some requisite definitions.
\begin{definition}
 A sequence $(u_j)_{j=1}^N\subseteq \R^d$ is a {\em tight
frame}\footnote{In \cite{Lyubarskii:2010} complex vector spaces are considered but
the results also hold in the real case.}
 if for all $w\in\R^d$,
$$\|w\|_2^2=\sum_{j=1}^{N}|\langle w,u_i\rangle|^2.$$
The redundancy of a frame is defined as $\lambda\doteq N/d\geq 1$.
\end{definition}
An easy to prove property of a tight frame (see Obs.~2.1 in \cite{Lyubarskii:2010}) is that for every frame representation $w = \sum_{j=1}^{N} a_i u_i$ it holds that $\sum_{j=1}^{N} a_i^2 \leq \|w\|_2^2$.
\begin{definition}
Consider a sequence $(u_j)_{j=1}^N\subseteq \R^d$ and $w \in \R^d$. An expansion
$w = \sum_{i=1}^N a_i u_i$ such that $\|a\|_\infty \leq \frac{K}{\sqrt{N}}\|w\|_2$ is referred to as a Kashin's representation of $w$ with level $K$.
\end{definition}

\begin{theorem}[\cite{Lyubarskii:2010}]
\label{thm:kashin-frame}
For all $\lambda=N/d>1$ there exists a tight frame $(u_j)_{j=1}^N\subseteq \R^d$ in which every $w \in \R^d$ has a Kashin's representation of $w$ with level $K$ for some constant $K$ depending only on $\lambda$.
Moreover, such a frame can be computed
in (randomized) polynomial time. 
\end{theorem}
The existence of such frames follows from Kashin's theorem \cite{Kashin:1977}. \citenames{Lyubarskii and Vershynin}{Lyubarskii:2010} show that any frame that satisfies a certain uncertainty principle (which itself is implied by the well-studied Restricted Isometry Property) yields a Kashin's representation for all $w \in \R^d$. In particular, various random choices of $u_j$'s have this property with high probability. Given a vector $w$, a Kashin's representation of $w$ for level $K$ can be computed efficiently (whenever it exists) by solving a convex program. For frames that satisfy the above mentioned uncertainty principle a Kashin's representation can also be found using a simple algorithm that involves $\log(N)$ multiplications of a vector by each of $u_j$'s. Other algorithms for the task are discussed in \cite{StuderGYB14}.

\begin{theorem}
\label{thm-l2-kashin}
For every $d$ there is an efficient algorithm that solves $\ell_2$ mean estimation problem (over $\B_2^d$) with error $\varepsilon$ using $2d$ queries to $\STAT(\Omega(\varepsilon))$.
\end{theorem}
\begin{proof}
For $N=2d$ let  $(u_j)_{j=1}^N\subseteq \R^d$ be a frame in which every $w \in \R^d$ has a Kashin's representation of $w$ with level $K=O(1)$ (as implied by Theorem \ref{thm:kashin-frame}). For a vector $w\in\R^d$ let $a(w) \in \R^N$ denote the coefficient vector of some specific Kashin's representation of $w$ (\eg that computed by the algorithm in \cite{Lyubarskii:2010}). Let $\bw$ be a random variable supported on $\B_2^d$ and let $\bar{a}_j \doteq \E[a(\bw)_j]$. By linearity of expectation, $\bar{w} = \E[\bw] = \sum _{j=1}^N \bar{a}_j u_j$.

For each $j\in[N]$, let $\phi_j(w) \doteq \frac{\sqrt{N}}{K} \cdot a(w)_j$. Let $\tilde a_j$  denote the answer of $\STAT(\varepsilon/K)$ to query $\phi_j$ multiplied by $\frac{K}{\sqrt{N}}$. By the definition of Kashin's representation with level $K$, the range of $\phi_j$ is $[-1,1]$ and, by the definition of $\STAT(\varepsilon/K)$, we have that $|\bar{a}_j - \tilde{a}_j| \leq \frac{\varepsilon}{\sqrt{N}}$ for every $j \in [N]$. Let $\tilde{w} \doteq \sum_{j=1}^N \tilde{a}_j u_j$.

Then by the property of tight frames mentioned above, $$\| \bar{w} - \tilde{w}\|_2 = \left\| \sum_{j=1}^N (\bar{a}_j - \tilde{a}_j) u_j \right \|_2 \leq \sqrt{\sum_{j=1}^N ( \bar{a}_j - \tilde{a}_j)^2} \leq \varepsilon .$$
\end{proof}

\subsubsection{$\ell_2$ Mean Estimation using a Random Basis}

We now show a simple to analyze randomized algorithm that achieves dimension independent estimation complexity for $\ell_2$ mean estimation. The algorithm will use coordinate-wise estimation in a randomly and uniformly chosen basis. We show that for such a basis simply truncating coefficients that are too large will, with high probability, have only a small effect on the estimation error.

More formally, we define the truncation operation as follows. For a real value $z$ and $a \in \R^+$, let
 \[
m_a(z):=\left\{
\begin{array}{rl}
z & \mbox{if } |z| \leq a\\
a			     & \mbox{if } z> a\\
-a			     & \mbox{if } z<-a.
\end{array}
\right.
\]
For a vector $w\in \R^d$ we define $m_a(w)$ as the coordinate-wise application of $m_a$ to $w$.
For a $d\times d$ matrix $U$ we define $m_{U,a}(w) \doteq U^{-1}m_a(Uw)$ and define $r_{U,a}(w) \doteq w - m_{U,a}(w)$. The key step of the analysis is the following lemma:

\begin{lemma}
\label{lem:error-distr}
Let $\bU$ be an orthogonal matrix chosen uniformly at random and $a > 0$.
For every $w$, with  $\|w\|_2=1$, $\E[\|r_{\bU,a}(w)\|_2^2] \leq 4 e^{-da^2/2}$.
\end{lemma}
\begin{proof}
Notice that $\|r_{\bU,a}(w) \|_2 = \| \bU w - m_a(\bU w)\|_2$.
It is therefore sufficient to analyze $\| \bu - m_a(\bu)\|_2$ for $\bu$ a random uniform vector of length 1. Let $\br \doteq \bu - m_a(\bu)$. For each $i$,
\begin{eqnarray*}
\E[\br_i^2] &=& \int_0^{\infty} 2t\, \pr[|\br_i|> t] \, dt
	= \int_0^{\infty} 2t\,\{\pr[\br_i>t]+\pr[\br_i<-t] \}\,dt \\
	&=& \int_0^{\infty} 4t\, \pr[\br_i>t] \, dt
	= \int_0^{\infty} 4t\,\pr[\bu_i-a>t]\, dt\\
	&=& 4\left\{\int_0^{\infty}(t+a)\pr[\bu_i>t+a]\,dt-a\int_0^{\infty}\pr[\bu_i>t+a]\,dt \right\}\\
	&\leq& 4\dfrac{e^{-da^2/2}}{d},
\end{eqnarray*}
where we have used the symmetry of $\br_i$ and concentration on the unit sphere. From this we obtain $\E[\|\br\|_2^2]\leq 4e^{-da^2/2}$, as claimed.
\end{proof}

From this lemma is easy to obtain the following algorithm.
\begin{theorem} \label{thm:ell_2_estimation}
There is an efficient randomized algorithm that solves the $\ell_2$ mean estimation problem with error $\varepsilon$ and success probability $1-\delta$ using $O(d \log(1/\delta))$  queries to $\STAT(\Omega(\varepsilon/\log(1/\varepsilon)))$.
\end{theorem}
\begin{proof}
Let $\bw$ be a random variable supported on $\B_2^d$. For an orthonormal $d\times d$ matrix $U$, and for $i\in[d]$, let $\phi_{U,i}(w)=(m_a(U w))_i/a$ (for some $a$ to be fixed later). Let $v_i$ be the output of $\STAT(\varepsilon/[2\sqrt da])$ for query $\phi_{U,i}:\W\to[-1,1]$, multiplied by $a$. Now, let $\tilde w_{U,a} \doteq U^{-1} v$, and let $\bar w_{U,a}\doteq \E[m_{U,a}(\bw)]$. This way,
\begin{eqnarray*}
\|\bar w-\tilde w_{U,a}\|_2 &\leq&  \|\bar w-\bar w_{U,a}\|_2 +  \|\bar w_{U,a}-\tilde w_{U,a}\|_2 \\
		&\leq& \|\bar w-\bar w_{U,a}\|_2 + \|\E[m_a(U \bw)]-v\|_2 \\
		&\leq&  \|\bar w-\bar w_{U,a}\|_2 +\varepsilon/2.
\end{eqnarray*}
Let us now bound the norm of $\bv \doteq \bar{w} -\bar{w}_{\bU,a}$ where $\bU$ is a randomly and uniformly chosen orthonormal $d\times d$ matrix.  By Chebyshev's inequality:
$$ \pr[\|\bv\|_2 \geq \varepsilon/2] \leq 4\dfrac{\E[\|\bv\|_2^2]}{\varepsilon^2} \leq \dfrac{16\exp(-da^2/2)}{\varepsilon^2}. $$
Notice that to bound the probability above by $\delta$ we may choose
$a=\sqrt{2\ln(16/(\delta\varepsilon^2))/d}$. Therefore, the queries above require querying
$\STAT(\varepsilon/[2\sqrt{2\ln(16/\delta\varepsilon^2)}])$, and they guarantee to
solve the $\ell_2$ mean estimation problem with probability at least $1-\delta$.

Finally, we can remove the dependence on $\delta$ in $\STAT$ queries by confidence
boosting. Let $\varepsilon^{\prime}=\varepsilon/3$ and
$\delta^{\prime}=1/8$, and
run the algorithm above with error $\varepsilon^{\prime}$ and success probability
$1-\delta^{\prime}$ for $\bU_1,\ldots, \bU_k$ i.i.d. random orthogonal matrices.
If we define
$\tilde w^1,\ldots,\tilde w^k$ the outputs of the algorithm, we can compute
the (high-dimensional) median $\tilde w$, namely the point $\tilde w^j$ whose median $\ell_2$ distance to all the other points is the smallest.
It is easy to see that (\eg \cite{nemirovsky1983problem,HsuSabato:2013arxiv})
$$ \pr[\|\tilde w-\bar w\|_2 > \varepsilon] \leq e^{-Ck},$$
where $C>0$ is an absolute constant.

Hence, as claimed, it suffices to choose $k=O(\log(1/\delta))$, which means using
$O(d\log(1/\delta))$ queries to $\STAT(\Omega(\varepsilon/\log(1/\varepsilon))$,
 to obtain success probability $1-\delta$.
\end{proof}

\subsubsection{$\ell_q$ Mean Estimation for $q > 2$}
We now demonstrate that by using the results for $\ell_{\infty}$ and $\ell_2$ mean estimation we can get algorithms for $\ell_q$ mean estimation with nearly optimal estimation complexity.

The idea of our approach is to decompose each point into a sum of at most $\log d$ points each of which has a small ``dynamic range" of non-zero coordinates. This property ensures a very tight relationship between the $\ell_{\infty}$, $\ell_2$ and $\ell_q$ norms of these points allowing us to estimate their mean with nearly optimal estimation complexity. More formally we will rely on the following simple lemma.
\begin{lemma}
\label{lem:norm-invert}
For any $x \in \R^d$ and any two $0<p<r$:
\begin{enumerate}
\item  $\|x\|_r \leq \|x\|_\infty^{1-p/r} \cdot \|x\|_p^{p/r}$;
\item Let $a = \min_{i \in [d]}\{x_i \cond x_i \neq 0\}$. Then $\|x\|_p \leq a^{1-r/p} \cdot \|x\|_r^{r/p}$.
\end{enumerate}
\end{lemma}
\begin{proof}
\begin{enumerate}
\item $$\|x\|_r^r =  \sum_{i=1}^d |x_i|^r \leq \sum_{i=1}^d \|x\|_\infty^{r-p} \cdot |x_i|^p = \|x\|_\infty^{r-p} \cdot \|x\|_p^p$$
\item $$\|x\|_r^r =  \sum_{i=1}^d |x_i|^r \geq \sum_{i=1}^d a^{r-p} \cdot |x_i|^p = a^{r-p} \cdot \|x\|_p^p .$$
\end{enumerate}
\end{proof}

\begin{theorem}
For any $q \in (2,\infty)$ and $\varepsilon > 0$, $\ell_q$ mean estimation with error $\varepsilon$ can be solved using $3d \log d$ queries to
$\STAT(\varepsilon/\log(d))$.
\end{theorem}
\begin{proof}
Let $k\doteq \lfloor \log(d)/q \rfloor-2$. For $w\in\R^d$, and $j=0,\ldots,k$ we define
$$R_j(w)\doteq \sum_{i=1}^d e_i w_i \ind{2^{-(j+1)}<|w_i|\leq 2^{-j}}, $$
and $R_{\infty}(w)\doteq \sum_{i=1}^d e_i w_i \ind{|w_i|\leq 2^{-(k+1)}}.$
It is easy to see that if $w\in\B_q$ then $w= \sum_{j=0}^k R_j(w) + R_{\infty}(w)$. Furthermore,
observe that $\|R_j(w)\|_{\infty}\leq 2^{-j}$, and by Lemma \ref{lem:norm-invert},
$\|R_j(w)\|_2 \leq 2^{-(j+1)(1-q/2)}.$
Finally, let $\bar w^j=\E[R_j(\bw)]$, and $\bar w^{\infty}=\E[R_{\infty}(\bw)]$.

Let $\varepsilon^{\prime}\doteq 2^{2/q-3} \varepsilon/(k+1)$. For each level $j=0,\ldots, k$, we
perform the following queries:
\begin{itemize}
\item By using $2d$ queries to $\STAT(\Omega(\varepsilon^{\prime}))$ we obtain a vector
$\tilde w^{2,j}$
such that $\|\tilde w^{2,j} - \bar w^j\|_{2}\leq 2^{(\frac{q}{2}-1)(j+1)}\varepsilon^{\prime}$.
For this, simply observe that $R_j(\bw)/[2^{(\frac{q}{2}-1)(j+1)}]$ is supported on $\B_2^d$,
so our claim follows from Theorem \ref{thm-l2-kashin}.
\item By using $d$ queries to $\STAT(\varepsilon^{\prime})$ we obtain a vector
$\tilde w^{\infty,j}$ such
that $\|\tilde w^{\infty,j} - \bar w^j\|_{\infty}\leq 2^{-j}\varepsilon^{\prime}$. For this, notice
that $R_j(\bw)/[2^{-j}]$ is supported on $\B_{\infty}^d$ and appeal to Theorem
\ref{thm:L-infty}.
\end{itemize}
We consider the following feasibility problem, which is always solvable
(e.g., by $\bar w^j$)
$$ \|\tilde w^{\infty,j} - w\|_{\infty}\leq 2^{-j}\varepsilon^{\prime}, \quad
\|\tilde w^{2,j} - w\|_{2}\leq 2^{(\frac{q}{2}-1)(j+1)}\varepsilon^{\prime}.$$
Notice that this problem can be solved easily
(we can minimize $\ell_2$ distance to $\tilde w^{2,j}$ with the $\ell_{\infty}$ constraint
above, and this minimization problem can be solved
coordinate-wise), so let $\tilde w^j$ be a solution. By the
triangle inequality, $\tilde w^j$ satisfies
$ \|\tilde w^j - \bar w^j\|_{\infty}\leq 2^{-j}(2\varepsilon^{\prime})$, and
$\|\tilde w^j - \bar w^j\|_{2}\leq 2^{(\frac{q}{2}-1)(j+1)}(2\varepsilon^{\prime}).$

By Lemma \ref{lem:norm-invert},
$$\|\tilde w^j-\bar w^j\|_q \leq \|\tilde w^j-\bar w^j\|_2^{2/q} \cdot
\|\tilde w^j-\bar w^j\|_{\infty}^{1-2/q} \leq 2^{(1-2/q)(j+1)} \, 2^{-j(1-2/q)} (2\varepsilon^{\prime})
= \varepsilon/[2(k+1)].
$$

Next we estimate $\bar w^{\infty}$. Since
$2^{-(k+1)}=2^{-\lfloor\ln d/q\rfloor+1}\leq 4d^{-1/q}$, by using $d$ queries to
$\STAT(\varepsilon/8)$ we can estimate each coordinate of $\bar w^{\infty}$ with
accuracy $\varepsilon/[2d^{1/q}]$ and
obtain $\tilde w^{\infty}$ satisfying
$\|\tilde w^{\infty}-\bar w^{\infty}\|_{q}
\leq d^{1/q}\|\tilde w^{\infty}-\bar w^{\infty}\|_{\infty}\leq \varepsilon/2$.
Let now $\tilde w=[\sum_{j=0}^k \tilde w^j]+\tilde w^{\infty}$. We have,
\begin{equation*}
\|\tilde w-\bar w\|_q \,\leq\, \sum_{j=0}^k \|\tilde w^j-\bar w^j\|_q +
\|\tilde w^{\infty}-\bar w^{\infty}\|_q
\,\leq\, (k+1)\frac{\varepsilon}{2(k+1)} + \frac{\varepsilon}{2}
\,=\, \varepsilon.
\end{equation*}
\end{proof}

\subsubsection{$\ell_q$ Mean Estimation for $q \in (1,2)$}
Finally, we consider the case when $q \in (1,2)$. 
Here we get the nearly optimal estimation complexity via two bounds.

The first bound follows from the simple fact that for all $w \in \R^d$, $\|w\|_2 \leq  \|w\|_q \leq d^{1/q - 1/2} \|w\|_2$. Therefore we can reduce $\ell_q$ mean estimation with error $\varepsilon$ to $\ell_2$ mean estimation with error $\varepsilon/d^{1/q - 1/2}$ (this is a special case of Lemma \ref{lem:norm-embed} with the identity embedding). Using Theorem \ref{thm-l2-kashin} we then get the following theorem.

\begin{theorem}
\label{thm-lq-small-eps}
For $q \in (1,2)$ and every $d$ there is an efficient algorithm that solves $\ell_q$ mean estimation problem with error $\varepsilon$ using $2d$ queries to $\STAT(\Omega(d^{1/2 - 1/q}\varepsilon))$.
\end{theorem}

It turns out that for large $\varepsilon$ better sample complexity can be achieved using a different algorithm. Achieving (nearly) optimal estimation complexity in this case requires the use of $\VSTAT$ oracle. (The estimation complexity for $\STAT$ is quadratically worse. That still gives an improvement over Theorem \ref{thm-lq-small-eps} for some range of values of $\varepsilon$.) In in the case of $q >2$, our algorithm decompose each point into a sum of at most $\log d$ points each of which has a small ``dynamic range" of non-zero coordinates. For each component we can then use coordinate-wise estimation with an additional zeroing of coordinates that are too small. Such zeroing ensures that the estimate does not accumulate large error from the coordinates where the mean of the component itself is close to 0.

\begin{theorem}
\label{thm:lq-large-eps}
For any $q \in (1,2)$ and $\varepsilon > 0$, the $\ell_q$ mean estimation problem
can be solved with error $\varepsilon$ using $2d \log d$ queries to
$\VSTAT((16\log(d)/\varepsilon)^p)$.
\end{theorem}
\begin{proof}
Given $w\in\B_q$ we consider its positive and negative parts: $w = w^+ -w^-$, where $w^+ \doteq \sum_{i=1}^d e_i w_i \ind{w_i \geq 0}$ and $w^- \doteq -\sum_{i=1}^d e_i w_i \ind{w_i < 0}$. We again rely on the decomposition of $w$ into ``rings" of dynamic range 2, but now for its positive and negative parts. Namely, $w= \sum_{j=0}^{k} [R_j(w^+)-R_j(w^-)] + [R_\infty(w^+)-R_{\infty}(w^-)]$, where $k\doteq \lfloor \log(d)/q \rfloor-2$, $R_j(w) \doteq \sum_{i=1}^d e_i w_i \ind{2^{-(j+1)}<|w_i|\leq 2^{-j}}$ and $R_\infty(w) \doteq \sum_{i=1}^d e_i w_i \ind{|w_i|\leq 2^{-k-1}}$.

Let $\bw$ be a random variable supported on $\B_q^d$.
Let $\varepsilon' \doteq \varepsilon/(2k+3)$. For each level $j=0,\ldots, k$, we now describe how to estimate $\overline{w^{+,j}} = \E[R_j(\bw^+)]$ with accuracy $\varepsilon'$. The estimation is essentially just coordinate-wise use of $\VSTAT$ with zeroing of coordinates that are too small.  Let $v'_i$ be the value returned by $\VSTAT(n)$ for query $\phi_i(w)= 2^j \cdot (R_j(w^+))_i$, where $n = (\varepsilon'/8)^{-p}\leq (16\log(d)/\varepsilon)^p$. Note that $2^j \cdot (R_j(w^+))_i \in [0,1]$ for all $w$ and $j$. Further, let $v_i = v'_i \cdot \ind{|v'_i| \geq 2/n}$. We start by proving the following decomposition of the error of $v$.
\newcommand{\sml}{<}
\newcommand{\lrg}{>}
\begin{lemma}
\label{lem:error-two-bounds}
Let $u \doteq 2^j \cdot \overline{w^{+,j}}$, and $z \doteq u - v$. Then $\|z\|_q^q \leq \|u^{\sml} \|_q^q + n^{-q/2}\cdot \|u^{\lrg}\|_{q/2}^{q/2}$, where $u^{\sml}_i = u_i \cdot \ind{u_i < 4/n}$ and $u^{\lrg}_i = u_i \cdot \ind{u_i \geq 1/n}$ and for all $i$.
\end{lemma}
\begin{proof}
For every index $i \in [d]$ we consider two cases. The first case is when $v_i = 0$. By the definition of $v_i$, we know that $v'_i< 2/n$. This implies that $u_i = 2^j \E[(R_j(\bw^+))_i] < 4/n$. This is true since, otherwise (when $u_i \geq 4/n$), by the guarantees of $\VSTAT(n)$, we would have $|v'_i - u_i| \leq \sqrt{\frac{u_i}{n}}$ and $v'_i \geq u_i - \sqrt{\frac{u_i}{n}} \geq 2/n$. Therefore in this case, $u_i = u^<_i$ and $z_i = u_i-v_i = u^<_i$.

In the second case $v_i\neq 0$. In this case we have that $v'_i \geq 2/n$. This implies that $u_i \geq 1/n$. This is true since, otherwise (when $u_i < 1/n$), by the guarantees of $\VSTAT(n)$, we would have $|v'_i - u_i| \leq \sqrt{\frac{u_i}{n}}$ and $v'_i \leq u_i+ \frac{1}{n} < 2/n$. Therefore in this case, $u_i = u^>_i$ and $z_i = u_i-v'_i$. By the guarantees of $\VSTAT(n)$, $|z_i| =|u^>_i-v'_i| \leq \max\left\{\frac{1}{n},\sqrt{\frac{u^>_i}{n}}\right\}=\sqrt{\frac{u^>_i}{n}}$.

The claim now follows since by combining these two cases we get $|z_i|^q \leq (u^<_i)^q + \left(\frac{u^>_i}{n}\right)^{q/2}$.
\end{proof}

We next observe that by Lemma \ref{lem:norm-invert}, for every $w \in \B_q^d$,
$$\|R_j(w^+)\|_1 \leq (2^{-j-1})^{1-q} \|R_j(w^+)\|_q^q \leq (2^{-j-1})^{1-q}.$$
This implies that \equ{\|u\|_1 = 2^j \cdot \left\|\overline{w^{+,j}}\right\|_1 = 2^j \cdot \left\|\E[R_j(\bw^+)]\right\|_1 \leq 2^j \cdot (2^{-j-1})^{1-q} = 2^{(j+1)q-1}. \label{eq:bound-u}}

Now by Lemma \ref{lem:norm-invert} and eq.\eqref{eq:bound-u}, we have
\equ{\|u^<\|_q^q \leq  \left(\frac{4}{n}\right)^{q-1} \cdot \|u^<\|_1 = n^{1-q} \cdot 2^{(j+3)q-3}. \label{eq:bound-small-u}}
Also by Lemma \ref{lem:norm-invert} and eq.\eqref{eq:bound-u}, we have
\equ{\|u^{\lrg}\|_{q/2}^{q/2} \leq \left(\frac{1}{n}\right)^{q/2-1} \cdot \|u^>\|_1 \leq  n^{1-q/2} \cdot 2^{(j+1)q-1} \label{eq:bound-large-u}.}

Substituting eq.~\eqref{eq:bound-small-u} and eq.~\eqref{eq:bound-large-u} into Lemma \ref{lem:error-two-bounds} we get
$$\|z\|_q^q \leq \|u^{\sml} \|_q^q + n^{-q/2}\cdot \|u^{\lrg}\|_{q/2}^{q/2} \leq n^{1-q} \cdot \left(2^{(j+3)q-3} + 2^{(j+1)q-1}\right) \leq  n^{1-q} \cdot 2^{(j+3)q}.$$
Let $\tilde{w}^{+,j} \doteq 2^{-j} v$. We have $$\left\|\overline{w^{+,j}} - 2^{-j} v\right\|_q = 2^{-j} \cdot \|z\|_q \leq 2^3 \cdot n^{1/q-1}=\varepsilon'.$$

We obtain an estimate of $\overline{w^{-,j}}$ in an analogous way. Finally, to estimate, $\bar{w}^\infty \doteq \E[R_\infty(\bw)]$ we observe that $2^{-k-1} \leq 2^{1-\lfloor \log(d)/q \rfloor} \leq 4 d^{-1/q}$. Now using $\VSTAT(1/(4\varepsilon')^2)$ we can obtain an estimate of each coordinate of $\bar{w}^\infty$ with accuracy $\varepsilon' \cdot d^{-1/q}$. In particular, the estimate $\tilde{w}^\infty$ obtained in this way satisfies $\|\bar{w}^\infty - \tilde{w}^\infty\|_q \leq \varepsilon'$.

Now let $\tilde{w} = \sum_{j=0}^{k} (\tilde{w}^{+,j} - \tilde{w}^{-,j}) + \tilde{w}^\infty$. Each of the estimates has $\ell_q$ error of at most $\varepsilon' =\varepsilon/(2k+3)$ and therefore the total error is at most $\varepsilon$.
\end{proof}

\subsubsection{General Convex Bodies}
Next we consider mean estimation and stochastic linear
optimization for convex bodies beyond $\ell_p$-balls. A first observation is that
Theorem \ref{thm:L-infty} can be easily generalized to origin-symmetric polytopes. The
easiest way to see the result is to use the standard embedding of the origin-symmetric
polytope norm into $\ell_{\infty}$ and appeal to Lemma \ref{lem:norm-embed}.
\begin{corollary}
Let $\W$ be an origin-symmetric polytope with $2m$ facets. Then mean estimation over $\W$ with error $\varepsilon$ can be efficiently solved using $m$ queries to $\STAT(\varepsilon/2)$.
\end{corollary}

In the case of an arbitrary origin-symmetric convex body \(\W\subseteq \R^d\), we can reduce mean estimation over $\W$ to $\ell_2$ mean estimation using the John ellipsoid. Such an ellipsoid
${\cal E}$ satisfies the inclusions $\frac{1}{\sqrt d}{\cal E}\subseteq \W\subseteq {\cal E}$ and any ellipsoid is linearly isomorphic to a unit $\ell_2$ ball.
Therefore appealing to Lemma \ref{lem:norm-embed} and Theorem \ref{thm-l2-kashin} we have
the following.
\begin{proposition}
Let $\W\subseteq\R^d$ an origin-symmetric convex body. Then the mean estimation
problem over $\W$ can be solved using $2d$ queries to
$\STAT(\Omega(\varepsilon/\sqrt d))$.
\end{proposition}

By Observation \ref{obs:lin_opt_mean_est}, for an arbitrary convex body \(\K\), the
stochastic linear optimization problem over $\K$ reduces to mean estimation over
$\W\doteq\conv(\K_{\ast},-\K_{\ast})$. This leads to a nearly-optimal (in terms of worst-case dimension dependence) estimation complexity. A matching lower bound for this task will be proved in Corollary \ref{cor:lower-bound-small-eps}.

A drawback of this approach is that it depends on knowledge of the John ellipsoid for $\W$,
which is, in general, cannot be computed efficiently (\eg \cite{Nemirovski:2013lectures}).
However, if $\K$ is a polytope with a polynomial number of facets, then $\W$ is an origin-symmetric polytope
with a polynomial number of vertices, and the John ellipsoid can be computed in
polynomial time \cite{Khachiyan:1996}. From this, we conclude that
\begin{corollary}
There exists an efficient algorithm that given as input the vertices of an origin-symmetric polytope $\W\subseteq \R^d$
 solves the mean estimation problem over $\W$ using $2d$ queries to
$\STAT(\Omega(\varepsilon/\sqrt{d}))$. The algorithm runs in time polynomial in the number of vertices.
\end{corollary}

\subsection{Lower Bounds}
\label{sec:lower-linear}
We now prove lower bounds for stochastic linear optimization over the $\ell_p$ unit ball and consequently also for $\ell_q$ mean estimation.
We do this using the technique from \cite{FeldmanPV:13} that is based on bounding the statistical dimension with discrimination norm.
The {\em discrimination norm} of a set of distributions $\D'$ relative to a distribution $D$ is denoted by $\dc(\D',D)$ and defined as follows:
\begin{align*}
\dc(\D',D) \doteq \max_{h:X \rightarrow \R, \|h\|_D=1} \left\{ \E_{D' \sim \D'}\left[\left| \E_{D'}[h]-\E_{D}[h]\right| \right] \right\},
\end{align*}
where the norm of $h$ over $D$ is $\|h\|_D = \sqrt{\E_D[h^2(x)]}$ and $D' \sim \D'$ refers to choosing $D'$ randomly and uniformly from the set $\D'$.

 Let $\B(\D,D)$ denote the decision problem in which given samples from an unknown input distribution $D' \in \D \cup \{D\}$ the goal is to output $1$ if $D'\in \D$ and 0 if $D'=D$.
\begin{definition}[\cite{FeldmanGRVX:12}]\label{def:sdima}
  For $\kappa>0$, domain $X$ and a decision problem $\B(\D,D)$, let $t$ be the largest integer
  such that there exists a finite set of distributions $\D_D \subseteq \D$ with the following property:
  for any subset $\D' \subseteq \D_D$, where $|\D'| \ge |\D_D|/t$, $\dc(\D',D) \leq \kappa$.
The \textbf{statistical dimension} with discrimination norm $\kappa$ of $\B(\D,D)$
is $t$ and denoted by $\SDN(\B(\D,D),\kappa)$.
\end{definition}

The statistical dimension with discrimination norm $\kappa$ of a problem over distributions gives
a lower bound on the complexity of any statistical algorithm.
\begin{thm}[\cite{FeldmanGRVX:12}]
\label{thm:avgvstat-random}
  Let $X$ be a domain and $\B(\D,D)$ be a decision problem over a class of distributions $\D$ on $X$ and reference distribution $D$. For $\kappa > 0$, let $t = \SDN(\B(\D,D),\kappa)$. Any randomized statistical algorithm that solves $\B(\D,D)$ with probability $\geq 2/3$ requires $t/3$ calls to $\VSTAT(1/(3 \cdot \kappa^2))$.
\end{thm}

We now reduce a simple decision problem to stochastic linear optimization over the $\ell_p$ unit ball. Let $E = \{e_i \cond i\in [d]\} \cup \{-e_i \cond i\in [d]\}$. Let the reference distribution $D$ be the uniform distribution over $E$.
For a vector $v \in [-1,1]^d$, let $D_v$ denote the following distribution: pick $i\in [d]$ randomly and uniformly, then pick $b \in \on$ randomly subject to the expectation being equal to $v_i$ and output $b \cdot e_i$. By definition, $\E_{\bw \sim D_v}[\bw] = \frac{1}{d} v$. Further $D_v$ is supported on $E \subset \B_q^d$.

\remove{
For $\alpha \in [0,1]$ and every $v \in \on^{d}$,  $\|\E_{\bw \sim D_{\alpha v}}[\bw]\|_q = \alpha/d \cdot d^{1/q} = \alpha \cdot d^{1/q-1}$. At the same time for the reference distribution $D$, $\|\E_{\bw \sim D}[\bw]\|_q = 0$.
Therefore to estimate the mean with accuracy $\varepsilon = \alpha d^{1/q-1}/2$ it is necessary to distinguish every distribution in $\D_\alpha$ from $D$, in other words to solve the decision problem $\B(\D_\alpha,D)$.
}

For $q \in [1,2]$,  $\alpha \in [0,1]$ and every $v \in \on^{d}$, $d^{1/q-1} \cdot v \in \B_p^d$ and $\la d^{1/q-1} v, \E_{\bw \sim D_{\alpha v}}[\bw] \ra =  \alpha \cdot d^{1/q-1}$. At the same time for the reference distribution $D$ and every $x \in \B_p^d$, we have that $\la x, \E_{\bw \sim D}[\bw] \ra = 0$. Therefore to optimize with accuracy $\varepsilon = \alpha d^{1/q-1}/2$ it is necessary distinguish every distribution in $\D_\alpha$ from $D$, in other words to solve the decision problem $\B(\D_\alpha,D)$.

\begin{lemma}
\label{lem:lower-bound-decision}
For any $r > 0$, $2^{\Omega(r)}$ queries to $\VSTAT(d/(r \alpha^2))$ are necessary to solve the decision problem $\B(\D_\alpha,D)$ with success probability at least $2/3$.
\end{lemma}
\begin{proof}
We first observe that for any function $h: \B_1^d \rightarrow \R$,
\equ{\E_{D_{\alpha v}}[h] - \E_{D}[h] = \frac{ \alpha}{2d} \sum_{i \in [d]} v_i \cdot (h( e_i) - h(-e_i)) . \label{eq:discr2h}}
Let $\beta =  \sqrt{\sum_{i \in [d]} (h( e_i) - h(-e_i))^2}$.
By Hoeffding's inequality we have that for every $r > 0$,
$$\pr_{v \sim \on^d } \left[  \left| \sum_{i \in [d]} v_i \cdot (h( e_i) - h(-e_i)) \right| \geq r \cdot \beta  \right] \leq 2e^{-r^2/2} .$$
This implies  that for every set $\V \subseteq \on^d$ such that $|\V| \geq 2^d/t$ we have that
$$\pr_{v \sim \V } \left[  \left| \sum_{i \in [d]} v_i \cdot (h( e_i) - h(-e_i)) \right| \geq r \cdot \beta  \right] \leq t \cdot 2 e^{-r^2/2} .$$
From here a simple manipulation (see Lemma A.4 in \cite{Shalev-ShwartzBen-David:2014}) implies that
$$ \E_{v \sim \V } \left[  \left| \sum_{i \in [d]} v_i \cdot (h( e_i) - h(-e_i)) \right| \right] \leq \sqrt{2}(2+ \sqrt{\ln t}) \cdot \beta \leq \sqrt{2\log t} \cdot \beta .$$
Note that $$\beta \leq \sqrt{\sum_{i \in [d]}2 h( e_i)^2 +2 h(-e_i)^2 }  =  \sqrt{2d} \cdot \|h\|_D.$$
For a set of distributions $\D' \subseteq \D_\alpha$ of size at least $2^d/t$, let $\V \subseteq \on^d$ be the set of vectors in $\on^d$ associated with $\D'$. By eq.\eqref{eq:discr2h} we have that
\alequn{\E_{D' \sim \D'}\left[\left| \E_{D'}[h]-\E_{D}[h]\right| \right] &= \frac{ \alpha}{2d}  \E_{v \sim \V } \left[  \left| \sum_{i \in [d]} v_i \cdot (h( e_i) - h(-e_i)) \right| \right] \\& \leq  \frac{ \alpha}{2d} 2 \sqrt{d\log t} \cdot \|h\|_D = \alpha \sqrt{\log t/d} \cdot \|h\|_D  .}
By Definition \ref{def:sdima}, this implies that for every $t>0$, $\SDN(\B(\D_\alpha,D), \alpha \sqrt{\log t/d} ) \geq t$. By Theorem \ref{thm:avgvstat-random} that for any $r > 0$, $2^{\Omega(r)}$ queries to $\VSTAT(d/(r \alpha^2))$ are necessary to solve the decision problem $\B(\D_\alpha,D)$ with success probability at least $2/3$.
\end{proof}
To apply this lemma with our reduction we set $\alpha = 2\varepsilon d^{1-1/q}$. Note that $\alpha$ must be in the range $[0,1]$ so this is possible only if $\varepsilon < d^{1/q-1}/2$. Hence the lemma gives the following corollary:
\begin{corollary}
\label{cor:lower-bound-small-eps}
For any $\varepsilon \leq d^{1/q-1}/2$ and $r > 0$, $2^{\Omega(r)}$ queries to $\VSTAT(d^{2/q-1} /(r \varepsilon^2))$ are necessary to find an
$\varepsilon$-optimal solution to the stochastic linear optimization problem over $\B_p^d$ with success probability at least $2/3$. The same lower bound holds for $\ell_q$ mean estimation with error $\varepsilon$.
\end{corollary}

Observe that this lemma does not cover the regime when $q > 1$ and $\varepsilon \geq d^{1/q-1}/2 = d^{-1/p}/2$. We analyze this case via a simple observation that for every $d' \in [d]$, $\B_p^{d'}$ and $\B_q^{d'}$ can be embedded into $\B_p^{d}$ and $\B_q^{d}$ respectively in a trivial way: by adding $d-d'$ zero coordinates. Also the mean of the distribution supported on such an embedding of $\B_q^{d'}$ certainly lies inside the embedding. In particular, a $d$-dimensional solution $x$ can be converted back to a $d'$-dimensional solution $x'$ without increasing the value achieved by the solution. Hence lower bounds for optimization over $\B_p^{d'}$ imply lower bounds for optimization over $\B_p^{d}$.  Therefore for any $\varepsilon \geq d^{-1/p}/2$, let $d' = (2\varepsilon)^{-p}$ (ignoring for simplicity the minor issues with rounding). Now Corollary \ref{cor:lower-bound-small-eps} applied to $d'$ implies that  $2^{\Omega(r)}$ queries to $\VSTAT((d')^{2/q-1} /(r \varepsilon^2))$ are necessary for stochastic linear optimization. Substituting the value of $d' =(2\varepsilon)^{-p}$ we get $(d')^{2/q-1} /(r \varepsilon^2) = 2^{2-p}/(r\varepsilon^p)$ and hence we get the following corollary.
\begin{corollary}
\label{cor:lower-bound-large-eps}
For any $q > 1$, $\varepsilon \geq d^{1/q-1}/2$ and $r > 0$, $2^{\Omega(r)}$ queries to $\VSTAT(1/(r \varepsilon^p))$ are necessary to find an $\varepsilon$-optimal solution to the stochastic linear optimization problem over $\B_p^d$ with success probability at least $2/3$.
The same lower bound holds for $\ell_q$ mean estimation with error $\varepsilon$.
\end{corollary}

These lower bounds are not tight when $q > 2$. In this case a lower bound of $\Omega(1/\varepsilon^2)$ (irrespective of the number of queries) follows from a basic property of $\VSTAT$: no query to $\VSTAT(n)$ can distinguish between two input distributions $D_1$ and $D_2$ if the total variation distance between $D_1^n$ and $D_2^n$ is smaller than some (universal) positive constant \cite{FeldmanGRVX:12}. 

\section{The Gradient Descent Family}
\label{sec:gradient}
We now describe approaches for solving convex programs by SQ algorithms that are based on the broad
literature of inexact gradient methods. We will show that some of the standard oracles
proposed in these works can be implemented by SQs; more precisely, by estimation
of the mean gradient. This reduces the task of solving a stochastic convex program to a
polynomial number of calls to the algorithms for
mean estimation from Section \ref{sec:linear}.

For the rest of the section we use the following notation.
Let ${\cal K}$ be a convex body in a normed space $(\R^d,\|\cdot\|)$,
and let $\W$ be a parameter space (notice we make no assumptions on this set).
Unless we explicitly state it, $\K$ is not assumed to be origin-symmetric.
Let $R\doteq\max_{x,y\in\K}\|x-y\|/2$, which is the $\|\cdot\|$-radius of $\K$.
For a random variable $\bw$ supported on $\W$ we consider the stochastic convex optimization problem $ \min_{x\in\K}\left\{ F(x)\doteq\E_{\bw}[f(x,\bw)]\right\},$
where for all $w\in\W$, $f(\cdot,w)$ is convex and subdifferentiable on $\K$.
Given $x\in\K$, we denote $\nabla f(x,w)\in \partial f(x,w)$ an arbitrary selection of a
subgradient;\footnote{We omit some necessary technical conditions,
\eg measurability, for the gradient
selection in the stochastic setting. We refer the reader to \cite{Rockafellar}
for a detailed discussion.}
similarly for $F$, $\nabla F(x)\in \partial F(x)$ is arbitrary.

Let us make a brief reminder of some important classes of convex functions.
We say a subdifferentiable convex function $f:\K\to\R$ is in the class
\begin{itemize}
\item $\F(\K,B)$ of $B$-bounded-range functions if for all $x\in \K$, $|f(x)| \leq B$.
\item $\F_{\|\cdot\|}^0(\K,L_0)$ of $L_0$-Lipschitz continuous functions
w.r.t.~$\|\cdot\|$, if for all $x,y\in\K$, $|f(x)-f(y)|\leq L_0\|x-y\|$; this implies
\begin{equation} \label{nonsmooth_dif_ineq}
f(y) \leq f(x) +\la\nabla f(x),y-x \ra +L_0\|y-x\|.
\end{equation}
\item $\F_{\|\cdot\|}^1(\K,L_1)$ of functions with $L_1$-Lipschitz continuous
gradient w.r.t.~$\|\cdot\|$, if for all $x,y\in\K$, $\|\nabla f(x)-\nabla f(y)\|_{\ast}\leq L_1\|x-y\|$;
this implies
\begin{equation}\label{smooth_dif_ineq}
f(y) \leq f(x) +\la\nabla f(x),y-x \ra +\frac{L_1}{2}\|y-x\|^2.
\end{equation}
\item ${\cal S}_{\|\cdot\|}(\K,\kappa)$ of $\kappa$-strongly convex functions
w.r.t.~$\|\cdot\|$, if for all $x,y\in\K$
\begin{equation} \label{str_cvx_dif_ineq}
f(y) \geq f(x) +\la\nabla f(x),y-x \ra +\frac{\kappa}{2}\|y-x\|^2.
\end{equation}
\end{itemize}

\subsection{SQ Implementation of Approximate Gradient Oracles}

Here we present two classes of oracles previously studied in the literature, together
with SQ algorithms for implementing them.

\begin{definition}[Global approximate gradient \cite{dAspremont:2008}]
\label{def:approx-grad}
Let \(F:{\cal K}\to \mathbb{R}\) be a convex subdifferentiable function.
We say that \(\tilde g:{\cal K}\to \mathbb{R}^d\) is an
{\em \(\eta\)-approximate gradient} of $F$ over $\K$ if for all \(u,x,y \in {\cal K}\)
\begin{equation} \label{ApproxSubgrad}
 |\langle \tilde g(x)-\nabla F(x), y-u\rangle| \leq \eta.
\end{equation}
\end{definition}

\begin{observation} \label{obs:approx_grad_oracle}
Let $\Ksym\doteq\{x-y \cond x,y\in\K\}$ (which is origin-symmetric
by construction), let furthermore $\|\cdot\|_{\Ksym}$ be the norm
induced by $\Ksym$ and $\|\cdot\|_{\Ksym_{\ast}}$ its dual norm.
Notice that under this notation, \eqref{ApproxSubgrad} is equivalent to
\(\|\tilde g(x)-\nabla F(x)\|_{\Ksym_\ast} \leq \eta\). Therefore, if
$F(x)=\E_{\bw}[f(x,\bw)]$ satisfies for all $w\in\W$,
$f(\cdot,w)\in \F_{\|\cdot\|_\Ksym}^0(\K,L_0)$
then implementing a $\eta$-approximate gradient reduces to
mean estimation in $\|\cdot\|_{\Ksym_{\ast}}$ with error $\eta/L_0$.
\end{observation}

\begin{definition}[Inexact Oracle \cite{Devolder:2014,Devolder2:2013}]
Let $F:\K\to\R$ be a convex subdifferentiable function. We say that
$(\tilde F(\cdot),\tilde g(\cdot)):\K\to\mathbb{R}\times \mathbb{R}^d$ is
a {\em first-order \((\eta,M,\mu)\)-oracle} of $F$ over $\K$ if for all \(x, y \in {\cal K}\)
\begin{equation} \label{str_cvx_oracle}
\dfrac{\mu}{2}\|y-x\|^2 \leq F(y) - [\tilde F(x) - \langle \tilde g(x),y-x\rangle]
\leq \dfrac{M}{2}\|y-x\|^2+\eta.
\end{equation}
\end{definition}

An important feature of this oracle is that the error for approximating
the gradient is {\em independent of the radius}. This observation
was established by \citenames{Devolder \etal}{Devolder2:2013}, and the consequences for statistical
algorithms are made precise in the following lemma.

\begin{lemma} \label{lem:str_cvx_oracle}
Let $\eta>0$, $0<\kappa\leq L_1$ and assume that for all $w\in\W$, $f(\cdot,w)\in\F(\K,B) \cap \F_{\|\cdot\|}^0(\K,L_0)$ and $F(\cdot) = \E_{\bw}[f(\cdot,\bw)]\in {\cal S}_{\|\cdot\|}(\K,\kappa) \cap \F_{\|\cdot\|}^1(\K,L_1)$.
Then implementing a first-order $(\eta, M,\mu)$-oracle (where $\mu=\kappa/2$
and $M=2L_1$) for $F$ reduces to mean estimation in $\|\cdot\|_{\ast}$
with error $\sqrt{\eta\kappa}/[2L_0]$, plus
a single query to $\STAT(\Omega(\eta/B))$. Furthermore,
for a first-order method that does not require values of $F$, the latter
query can be omitted.

If we remove the assumption $F\in \F_{\|\cdot\|}^1(\K,L_1)$ we can
instead use the upper bound $M=2L_0^2/\eta$.
\end{lemma}
\iffull
\begin{proof}
\iffull\else
\section{Proof of Lemma \ref{lem:str_cvx_oracle}}
\label{sec:StronglyConvex}
\fi

We first observe that we can obtain an approximate
zero-order oracle for $F$ with error $\eta$ by a single query to
$\STAT(\Omega(\eta/B)).$
In particular, we can obtain a value \(\hat F(x)\) such that
\(|\hat F(x)-F(x)|\leq \eta/4\), and then use as approximation
\[ \tilde F(x) = \hat F(x)-\eta/2.\]
This way \(|F(x)-\tilde F(x)| \leq |F(x)-\hat F(x)|+|\hat F(x)-\tilde F(x)|\leq 3\eta/4\),
and also \(F(x)-\tilde F(x) = F(x)-\hat F(x)+\eta/2 \geq \eta/4 \).
Finally, observe that for any gradient method that does not require access
to the function value we can skip the estimation of $\tilde F(x)$, and simply replace
it by $F(x)-\eta/2$ in what comes next.

Next, we prove that an approximate gradient \(\tilde g(x)\) satisfying
\begin{equation} \label{str_cvx_grad_approx}
\|\nabla F(x) - \tilde g(x)\|_{\ast} \leq \sqrt{\eta\kappa}/2 \leq \sqrt{\eta L_1}/2 ,
\end{equation}
suffices for a \((\eta,\mu,M)\)-oracle, where,
\(\mu=\kappa/2\), \(M=2L_1\). For convenience, we refer to the first inequality in \eqref{str_cvx_oracle} as the
{\em lower bound}  and the second as the {\em upper bound}.\\

\noindent{\bf Lower bound.}
Since \(F\) is \(\kappa\)-strongly convex, and by the lower bound on
$F(x)-\tilde F(x)$
\begin{eqnarray*}
F(y) &\geq& F(x)+\langle \nabla F(x),y-x\rangle +\frac{\kappa}{2}\|x-y\|^2 \\
	 &\geq& \tilde F(x) + \eta/4 +\langle \tilde g(x),y-x\rangle
	 			+\langle \nabla F(x)-\tilde g(x),y-x\rangle + \frac{\kappa}{2}\|x-y\|^2.
\end{eqnarray*}
Thus to obtain the lower bound it suffices prove that for all \(y\in\R^d\),
\begin{equation} \label{to_prove}
 \frac{\eta}{4} + \langle \nabla F(x)-\tilde g(x),y-x\rangle + \frac{\mu}{2}\|x-y\|^2\geq 0.
\end{equation}
In order to prove this inequality, notice that among all \(y\)'s such that
\(\|y-x\|=t\), the minimum of the expression above is attained when
\(\langle \nabla F(x)-\tilde g(x),y-x\rangle = -t\|\nabla F(x)-\tilde g(x)\|_{\ast}\). This
leads to the one dimensional inequality
\[\frac{\eta}{4} - t\|\nabla F(x)-\tilde g(x)\|_{\ast} + \frac{\mu}{2}t^2 \geq 0,\]
whose minimum is attained at \(t=\frac{\|\nabla F(x)-\tilde g(x)\|_{\ast}}{\mu}\),
and thus has minimum value \(\eta/4-\|\nabla F(x)-\tilde g(x)\|_{\ast}^2/(2\mu)\).
Finally, this value is nonnegative by assumption, proving
the lower bound.

\noindent{\bf Upper bound.} Since
\(F\) has \(L_1\)-Lipschitz continuous gradient, and by the bound on $|F(x)-\tilde F(x)|$
\begin{eqnarray*}
 F(y) &\leq& F(x) +\langle \nabla F(x),y-x\rangle +\dfrac{L_1}{2}\|y-x\|^2 \\
 	  &\leq& \tilde F(x) +\dfrac{3\eta}{4} +\langle \tilde g(x),y-x\rangle +
	  \langle \nabla F(x)-\tilde g(x),y-x\rangle+\dfrac{L_1}{2}\|x-y\|^2.
\end{eqnarray*}
Now we show that for all \(y\in \R^d\)
\begin{eqnarray*}
 \dfrac{L_1}{2}\|y-x\|^2-\langle \nabla F(x)-\tilde g(x),y-x\rangle +\frac{\eta}{4}\geq 0.
\end{eqnarray*}
Indeed, minimizing the expression above in \(y\)
shows that it suffices to have
\(\|\nabla F(x)-\tilde g(x)\|_{\ast}^2\leq \eta L_1/2\),
which is true by assumption.

Finally, combining the two bounds above we get that for all $y\in\K$
\[ F(y)\leq [\tilde F(x)+ \langle \tilde g(x),y-x\rangle]+\frac{M}{2}\|y-x\|^2+\eta,\]
which is precisely the upper bound.

As a conclusion, we proved that in order to obtain $\tilde g$ for
a $(\eta,M,\mu)$-oracle it suffices to obtain an approximate gradient
satisfying \eqref{str_cvx_grad_approx}, which can be obtained by solving
a mean estimation problem in $\|\cdot\|_{\ast}$ with error $\sqrt{\eta\kappa}/[2L_0]$.
This together with our analysis of the zero-order oracle proves the result.

Finally, if we remove the assumption $F\in{\cal F}_{\|\cdot\|}^1(\K,L_1)$ then from
\eqref{nonsmooth_dif_ineq} we can prove that for all $x,y\in\K$
\[ F(y) - [F(x)+\langle \nabla F(x),y-x\rangle] \leq \frac{L_0^2}{\eta}\|x-y\|^2+\frac{\eta}{4},\]
where $M=2L_0^2/\eta$. This is sufficient for carrying out the proof above, and
the result follows.
\iffull \else \hfill $\qed$ \fi
\end{proof}
\fi

\subsection{Classes of Convex Minimization Problems}
We now use known inexact convex minimization algorithms together with our SQ implementation of approximate gradient oracles
to solve several classes of stochastic optimization problems. We will see that in terms
of estimation complexity there is no significant gain from the non-smooth to the smooth
case; however, we can significantly reduce the number of queries by acceleration
techniques.

On the other hand, strong convexity leads to improved estimation complexity
bounds: The key insight here is that only a local approximation of the gradient around the
current query point suffices for methods, as a first order $(\eta,M,\mu)$-oracle
is robust to crude approximation of the gradient at far away points from the query
(see Lemma \ref{lem:str_cvx_oracle}).
We note that both smoothness and strong convexity are required only for the objective function and not for each function in the support of the distribution. This opens up the possibility of applying this algorithm without
the need of adding a strongly convex term pointwise --\eg in regularized linear regression--
as long as the expectation is strongly convex.

\subsubsection{Non-smooth Case: The Mirror-Descent Method}

Before presenting the mirror-descent method we give some necessary background
on prox-functions. We assume the existence of a subdifferentiable $r$-uniformly
convex function (where $2\leq r<\infty$) $\Psi:\K\to\R_+$ w.r.t. the norm $\|\cdot\|$, i.e.,
that satisfies\footnote{We have normalized the function so that the constant of
$r$-uniform convexity is 1.} for all $x,y\in\K$
\begin{equation} \label{unif_conv_grad}
\Psi(y)\geq \Psi(x) + \la \nabla \Psi(x), y-x \ra +\frac1r\|y-x\|^r.
\end{equation}
We will assume w.l.o.g. that $\inf_{x\in \K}\Psi(x) = 0$.

The existence of $r$-strongly convex functions holds in rather general situations
\cite{Pisier:2011}, and, in particular, for finite-dimensional $\ell_p^d$ spaces we have
explicit constructions for $r=\min\{2,p\}$ (see Appendix \ref{sec:unif_cvx} for details).
Let $D_{\Psi}(\K)\doteq\sup_{x\in\K}\Psi(x)$ be the
{\em prox-diameter of} $\K$ w.r.t.~$\Psi$.

We define the prox-function (a.k.a.~Bregman distance)
at $x\in \mbox{int}(\K)$ as $V_x(y) = \Psi(y) -\Psi(x) -\la\nabla\Psi(x), y-x \ra$.
In this case we say the prox-function is based on $\Psi$ proximal setup.
Finally, notice that by \eqref{unif_conv_grad} we have $V_x(y) \geq \frac1r\|y-x\|^r$.\\

For the first-order methods in this section we will assume $\K$ is such that for any
vector $x\in\K$ and $g\in\R^d$ the {\em proximal problem}
$\min\{\la g,y-x\ra+V_{x}(y) :\,y\in\K\}$ can be solved efficiently. For the case
$\Psi(\cdot)=\|\cdot\|_2^2$ this corresponds to Euclidean projection, but
this type of problems can be efficiently solved in more general situations
\cite{nemirovsky1983problem}.

The first class of functions we study is $\F_{\|\cdot\|}^0({\cal K},L_0)$.
We propose to solve problems in this class by the mirror-descent method
\cite{nemirovsky1983problem}. This is a classic method for minimization of non-smooth functions, with various applications to stochastic and online learning. Although simple
and folklore, we are not aware of a reference on the analysis of the inexact version
with proximal setup based on a $r$-uniformly convex function. Therefore we include
its analysis \iffull here\fi.

Mirror-descent uses a prox function $V_x(\cdot)$ based on $\Psi$ proximal setup.
The method starts querying a gradient at point
$x^0=\arg\min_{x\in\K} \Psi(x)$, and given a response $\tilde g^t\doteq \tilde g(x^t)$ to the gradient query at point $x^t$
it will compute its next query point as \begin{equation} \label{Prox_step}
x^{t+1} = \arg\min_{y\in\K} \{ \alpha\la \tilde g^t,y-x^t \ra + V_{x^t}(y) \},
\end{equation}
which corresponds to a proximal problem.
The output of the method is the average of iterates $\bar x^T\doteq \frac1T \sum_{t=1}^T x^t$.

\begin{proposition} \label{Prop:Inexact_MD}
Let \(F\in \F_{\|\cdot\|}^0(\K,L_0)\) and $\Psi:\K\to\R$ be an $r$-uniformly
convex function. Then the inexact mirror-descent method with
$\Psi$ proximal setup, step size $\alpha=\frac{1}{L_0}[rD_{\Psi}(\K)/T]^{1-1/r}$, and an $\eta$-approximate gradient for $F$ over $\K$, guarantees after $T$ steps an accuracy
\[F(\bar x^T)-F^{\ast} \leq L_0\left( \frac{rD_{\Psi}(\K)}{T}\right)^{1/r}+\eta.\]
\end{proposition}
\iffull \iffull
\begin{proof}
\else \section{Proof of Proposition \ref{Prop:Inexact_MD}}
\label{sec:proof_inexact_MD}
\fi
We first state without proof the following identity for prox-functions (for example, see
(5.3.20) in \cite{Nemirovski:2013lectures}): for all \(x\), \(x^{\prime}\)
and \(u\) in \(\K\)
\[ V_{x}(u)-V_{x^{\prime}}(u)-V_{x}(x^{\prime}) =
\langle \nabla V_{x}(x^{\prime}),u-x^{\prime}\rangle. \]

On the other hand, the optimality conditions of problem \eqref{Prox_step} are
\begin{equation*}
\langle \alpha \tilde g^t+\nabla V_{x^t}(x^{t+1}),u-x^{t+1} \rangle \geq 0,
\quad \forall u\in \K.
\end{equation*}

Let \(u\in \K\) be an arbitrary vector, and let $s$ be such that $1/r+1/s=1$.
Since $\tilde g^t$ is a $\eta$-approximate gradient,
\begin{eqnarray*}
\alpha [F(x^t)-F(u)] &\leq& \alpha \langle \nabla F(x^t),x^t-u\rangle \\
					   &\leq& \alpha \langle \tilde g^t,x^t-u\rangle +\alpha\eta \\
			      &  =  & \alpha \langle \tilde g^t,x^t-x^{t+1}\rangle
				+ \alpha \langle \tilde g^t,x^{t+1}-u\rangle +\alpha\eta \\
			       &\leq& \alpha \langle \tilde g^t,x^t-x^{t+1}\rangle
				 -  \langle \nabla V_{x^t}(x^{t+1}),x^{t+1}-u\rangle +\alpha\eta \\
			      &  =  & \alpha \langle \tilde g^t,x^t-x^{t+1}\rangle
				   +V_{x^t}(u)-V_{x^{t+1}}(u)-V_{x^t}(x^{t+1}) +\alpha\eta\\
			&\leq& [\alpha \langle \tilde g^t,x^t-x^{t+1}\rangle
					                           - \frac{1}{r}\|x^t-x^{t+1}\|^r]
						    +V_{x^t}(u)-V_{x^{t+1}}(u) +\alpha \eta \\
					&\leq& \frac{1}{s}\|\alpha \tilde g^t\|_{\ast}^{s}
					         + V_{x^t}(u)-V_{x^{t+1}}(u) +\alpha\eta,
\end{eqnarray*}
where we have used all the observations above, and the last step holds by
Fenchel's inequality.

Let us choose \(u\) such that \(F(u)=F^{\ast}\), thus by definition of \(\bar x^T\)
and by convexity of \(f\)
$$ \alpha T[F(\bar x^T)-F^{\ast}] \,\,\leq\,\,  \sum_{t=1}^T  \alpha[F(x^t)-F^{\ast}]
		\,\,\leq \frac{(\alpha L_0)^{s}}{s}T+D_{\Psi}(\K) +\alpha T\eta.$$
and since
\(\alpha=\frac{1}{L_0}\left( \frac{rD_{\Psi}(\K)}{T}\right)^{1/s}\) we obtain
$F(\bar x^T)-F^{\ast} \leq L_0\left( \frac{rD_{\Psi}(\K)}{T}\right)^{1/r} +\eta. $
\iffull\end{proof}
\else\fi 
\fi

\begin{remark}
As in our mean estimation problems, we assume a uniform bound on the norm of the gradient over the whole support of the input distribution $D$. It is known that techniques based on stochastic gradient descent achieve similar guarantees when one uses a bound on the second moment of the norm of gradients instead of the uniform bound (\eg \cite{Nemirovski:1978,Nemirovski:2009}). For SQs the same setting and (almost) the same estimation complexity can be obtained using recent results from \cite{Feldman:16sqvar}.  The results show that $\VSTAT$ allows estimation of expectation of any unbounded function $\phi$ of $\bw$ within $\eps \sigma$ using $1/\eps^2$ queries of estimation complexity $\tilde{O}(1/\eps^2)$, where $\sigma$ is the standard deviation of $\phi(\bw)$.
\end{remark}

We can readily apply the result above to stochastic convex programs in
non-smooth $\ell_p$ settings.
\begin{definition}[$\ell_p$-setup]
Let $1\leq p\leq \infty$, $L_0,R>0$, and $\K\subseteq\B_p^d(R)$ be a convex body. We define as
the (non-smooth) $\ell_p$-setup the family of problems
$\min_{x\in\K}\{ F(x) \doteq \E_{\bw}[f(x,\bw)]\}$, where for all $w\in\W$, $
f(\cdot,w)\in\F_{\|\cdot\|_p}^0(\K,L_0)$.

In the smooth $\ell_p$-setup we additionally assume that $F\in\F_{\|\cdot\|_p}^1(\K,L_1)$.
\end{definition}

From constructions of $r$-uniformly
convex functions for $\ell_p$ spaces, with $r=\min\{2,p\}$ (see Appendix
\ref{sec:unif_cvx}), we know that there exists an efficiently computable Prox function $\Psi$
(\ie whose value and gradient can be computed exactly, and thus
problem \eqref{Prox_step} is solvable for simple enough $\K$). The
consequences in terms of estimation complexity are summarized in the
following corollary, and proved in Appendix \ref{proof_solve_cvx_ellp}.

\begin{corollary} \label{cor:solve_cvx_ellp}
The stochastic optimization problem in the non-smooth $\ell_p$-setup can be solved with
accuracy $\varepsilon$ by:
\begin{itemize}
\item If $p=1$, using
$O\left(d\log d\cdot \left(\dfrac{L_0R}{\varepsilon}\right)^2 \right)$ queries to
$\STAT\left(\dfrac{\varepsilon}{4L_0R}\right)$;
\item If $1<p< 2$, using
$O\left(d\log d\cdot \dfrac{1}{(p-1)}\left(\dfrac{L_0R}{\varepsilon}\right)^2\right)$
queries to $\STAT\left( \Omega\left(\dfrac{\varepsilon}{[\log d]L_0R}\right)\right)$;
\item If $p=2$, using $O\left(d \cdot \left(\dfrac{L_0R}{\varepsilon}\right)^2\right)$
queries to $\STAT\left( \Omega\left(\dfrac{\varepsilon}{L_0R}\right)\right)$;
\item If $2<p<\infty$, using
$O\left(d\log d\cdot 4^{p}\left(\dfrac{L_0R}{\varepsilon}\right)^p\right)$
queries to $\VSTAT\left(\left(\dfrac{64 L_0R \log d}{\varepsilon}\right)^p\right)$.
\end{itemize}
\end{corollary}

\subsubsection{Smooth Case: Nesterov Accelerated Method}

Now we focus on the class of functions whose expectation has
Lipschitz continuous gradient. For simplicity, we will restrict the
analysis to the case where the Prox function is obtained from a strongly
convex function, i.e., $r$-uniform convexity with $r=2$. We utilize a known
inexact variant of Nesterov's accelerated method \cite{nesterov1983method}.

\begin{proposition}[\cite{dAspremont:2008}] \label{prop:dAspremont}
Let $F\in \F_{\|\cdot\|}^1(\K,L_1)$, and let $\Psi:\K\to\R_+$ be a $1$-strongly
convex function w.r.t. $\|\cdot\|$.
Let \((x^t,y^t,z^t)\) be the iterates of the accelerated method with $\Psi$
proximal setup, and where the algorithm has access to an \(\eta\)-approximate
gradient oracle for $F$ over $\K$. Then,
\[ F(y^T)-F^{\ast} \leq \dfrac{L_1D_{\Psi}(\K)}{T^2}+3\eta.\]
\end{proposition}

The consequences for the  smooth $\ell_p$-setup, which are straightforward
from the theorem above and Observation \ref{obs:approx_grad_oracle},
are summarized below, and proved in Appendix \ref{proof_solve_smooth_cvx_ellp}.
\begin{corollary} \label{cor:solve_smooth_cvx_ellp}
Any stochastic convex optimization problem in the smooth $\ell_p$-setup
can be solved with accuracy $\varepsilon$ by:
\begin{itemize}
\item If $p=1$, using
$O\left(d \sqrt{\log d}\cdot \sqrt{\dfrac{L_1R^2}{\varepsilon}} \right)$ queries to
$\STAT\left(\dfrac{\varepsilon}{12L_0R} \right)$;
\item If $1<p< 2$, using
$O\left(d\log d\cdot \dfrac{1}{\sqrt{p-1}}\sqrt{\dfrac{L_1R^2}{\varepsilon}}\right)$
queries to
$\STAT\left( \Omega\left(\dfrac{\varepsilon}{[\log d]L_0R}\right)\right)$;
\item If $p=2$, using $O\left(d\cdot \sqrt{\dfrac{L_1R^2}{\varepsilon}}\right)$
queries to
$\STAT\left( \Omega\left(\dfrac{\varepsilon}{L_0R}\right)\right)$.
\end{itemize}
\end{corollary}
\subsubsection{Strongly Convex Case}

Finally, we consider the class ${\cal S}_{\|\cdot\|}(\K,\kappa)$ of strongly
convex functions.
We further restrict our attention to the Euclidean case, i.e., $\|\cdot\|=\|\cdot\|_2$.
There are two main advantages of having a strongly convex objective: On
the one hand, gradient methods in this case achieve linear convergence
rate, on the other hand we will see that estimation complexity is
independent of the radius. Let us first make precise the first statement:
It turns out that with a \((\eta,M,\mu)\)-oracle we can implement
the inexact dual gradient method \cite{Devolder2:2013} achieving
linear convergence rate. The result is as follows

\begin{theorem}[\cite{Devolder2:2013}] \label{thm:linear_conv}
Let $F:{\cal K}\to\R$ be a subdifferentiable convex function endowed
with a \((\eta,M,\mu)\)-oracle over $\K$.
Let $y^t$ be the sequence of averages of the inexact dual gradient method, then
\[ F(y^T)-F^{\ast} \leq \dfrac{MR^2}{2} \exp\left(-\frac{\mu}{M}(T+1) \right)+\eta.\]
\end{theorem}

The results in \cite{Devolder2:2013} indicate that the accelerated method can also
be applied in this situation, and it does not suffer from noise accumulation. However,
the accuracy requirement is more restrictive than for the primal and dual
gradient methods. In fact, the required accuracy for the approximate gradient is
$\eta=O(\varepsilon\sqrt{\mu/M})$; although this is still independent of the radius,
it makes estimation complexity much more sensitive to condition number, which is
undesirable.

An important observation of the dual gradient algorithm is that it does not
require function values (as opposed to its primal version). This together
with Lemma \ref{lem:str_cvx_oracle} leads to the following result.
\begin{corollary} \label{cor:solve_str_cvx}
The stochastic convex optimization problem $\min_{x\in \K} \{F(x)\doteq \E_{\bw}[f(x,w)]\}$,
where $F\in{\cal S}_{\|\cdot\|_2}(\K,\kappa)\cap\F_{\|\cdot\|_2}^1(\K,L_1)$, and for all $w\in\W$,
$f(\cdot,w)\in \F_{\|\cdot\|_2}^0(\K,L_0)$,
can be solved to accuracy $\varepsilon>0$ using
$O\left( d\cdot \dfrac{L_1}{\kappa}\log\left(\dfrac{L_1R}{\varepsilon} \right) \right)$
queries to $\STAT(\Omega(\sqrt{\varepsilon \kappa}/L_0))$.

Without the assumption $F\in {\cal F}_{\|\cdot\|_2}^1({\cal K},L_1)$ the
problem can be solved to accuracy $\varepsilon>0$ by
using $O\left(d\cdot \dfrac{L_0^2}{\varepsilon\kappa}
\log\left(\dfrac{L_0R}{\varepsilon}\right) \right)$ queries
to $\STAT(\Omega(\sqrt{\varepsilon\kappa}/L_0))$.
\end{corollary}

\iffull
\iffull \subsection{Applications to Generalized Linear Regression}
\else \section{Applications to Generalized Linear Regression} \fi
\label{subsec:regression}
We provide a comparison of the bounds obtained by statistical query inexact
first-order methods with some state-of-the-art error bounds for linear regression problems. To be precise, we
compare sample complexity of obtaining excess error $\varepsilon$ (with
constant success probability or in expectation) with the estimation complexity of the
SQ oracle for achieving $\varepsilon$ accuracy. It is worth
noticing though that these two quantities are not directly comparable,
as an SQ algorithm performs a (polynomial) number
of queries to the oracle. However, this comparison shows that our results roughly match what can
be achieved via samples.

We consider the {\em generalized linear regression} problem:
Given a normed space $(\R^d,\|\cdot\|)$, let $\W\subseteq \R^d$ be
the input space, and $\R$ be the output space.
Let $(\bw,\bz)\sim D$, where $D$ is an unknown target distribution
supported on $\W\times\R$. The objective is to obtain a linear predictor
$x\in\K$ that predicts the outputs as a function of the inputs coming from $D$. Typically, $\K$ is prescribed by
desirable structural properties of the predictor, {\em e.g.}~sparsity or low norm.
The parameters determining complexity are given by bounds on the predictor and
input space: $\K\subseteq \B_{\|\cdot\|}(R)$
and $\W\subseteq \B_{\|\cdot\|_{\ast}}(W)$. Under these assumptions we
may restrict the output space to $[-M,M]$, where $M=RW$.

The prediction error is measured using a {\em loss function}.
For a function $\ell:\R \times\R\to\R_+$, letting $f(x,(w,z))=\ell(\la w,x\ra,z)$,
we seek to solve the stochastic convex program
$\min_{x\in\K}\{F(x)=\E_{(\bw,\bz)\sim D}[f(x,(\bw,\bz))]\}$.
We assume that $\ell(\cdot,z)$ is convex for every $z$ in the support of $D$. A common example of this problem is
the (random design) least squares linear regression, where $\ell(z',z)=(z'-z)^2$.

\paragraph{Non-smooth case:}
We assume that for every $z$ in the support of $D$, $\ell(\cdot, z)\in{\cal F}_{|\cdot|}^0([-M,M],L_{\ell,0})$.
To make the discussion concrete, let us consider the $\ell_p$-setup, \ie
$\|\cdot\|=\|\cdot\|_p$. Hence the Lipschitz constant of our stochastic objective
$f(\cdot,(w,z))=\ell(\la w,\cdot\ra,z)$ can be upper bounded as $L_0\leq L_{\ell,0}\cdot W$.
For this setting \citenames{Kakade \etal}{Kakade:2008} show that the sample complexity of achieving excess error $\varepsilon>0$
with constant success probability is $n=O\left(\left(\frac{L_{\ell,0}WR}{\varepsilon}\right)^2 \ln d\right)$ when $p=1$;
and $n = O\left(\left(\frac{L_{\ell,0}WR}{\varepsilon}\right)^2 (q-1)\right)$ for $1< p\leq 2$.
Using Corollary \ref{cor:solve_cvx_ellp} we obtain that the estimation complexity of solving this problem using our SQ implementation of
 the mirror-descent method gives the same up to (at most) a logarithmic in $d$ factor.

\citenames{Kakade \etal}{Kakade:2008} do not provide sample complexity bounds for $p>2$, however
since their approach is based on Rademacher complexity (see Appendix
\ref{sec:Samples} for the precise bounds), the bounds in this
case should be similar to ours as well.
\remove{
\paragraph{Smooth case:}
Risk bounds for the Hilbert space case
(e.g., $\ell_2$) were studied by
\citenames{Srebro}{Srebro:2010}. Their bounds are not described precisely in our setup, but by
observing that
$$ \|\nabla f(x,(w,z))-\nabla f(x,(w,z))\|_{\ast}
\leq \|w\|_{\ast} |\ell^{\prime}(\la w,x\ra,z)-\ell^{\prime}(\la w,y\ra,z)|
\leq \|w\|_{\ast}^2 L_{\ell,1}\|x-y\|,$$
we have $L_1\leq L_{\ell,1}W^2$, which implies an expected risk bound of
$\varepsilon=
O\left(\frac{L_{\ell,1}(WR)^2}{n}+\sqrt{\frac{L_{\ell,1}(WR)^2 F^{\ast}}{n}}\right)$.
In the non-realizable case (i.e., where $F^{\ast}>0$) the second term dominates
asymptotically in $n$, which leads to a sample complexity
$O((\frac{WR}{\varepsilon})^2L_{\ell,1}F^{\ast})$.
Our estimation
complexity bound from Corollary \ref{cor:solve_smooth_cvx_ellp} is
$O((\frac{WR}{\varepsilon})^2L_{\ell,1}B)$, where $B$ is a global upper bound
on $f$, so it matches the sample complexity bound up to the gap between
$F^{\ast}$ and $B$. This gap is related to a different technique they use
for the risk bound.
}
\remove{ 
\paragraph{Incorporating sparsity:}
Our optimization algorithms also hold for $\ell_p$ spaces where $1\leq p\leq 2$.
A case of particular interest is $p=1$ which is used for sparse regression. For
the nonsmooth case, the
analysis of mirror-descent over the $\ell_1$-ball was already studied in
\cite{Steinhardt:2015}, and their results in this case are equivalent to ours.
In short, the estimation complexity in this case is $O\left(\frac{k^4}{\varepsilon^2}\right)$,
where $k$ is the target sparsity level (here $W=1$ and $R=k$).
}

\paragraph{Strongly convex case:}
Let us now consider a generalized linear regression with regularization. Here
$$f(x,(w,z)) = \ell(\la w,x\ra,z) +\lambda\cdot \Phi(x),$$
where $\Phi:\K\to \R$ is a 1-strongly convex function and $\lambda>0$.
This model has a variety of applications in machine learning, such as
ridge regression and soft-margin SVM.
For the non-smooth linear regression in $\ell_2$ setup (as described above), \citenames{Shalev-Shwartz \etal}{SSSSS:2009} provide a sample complexity bound of $O\left(\frac{(L_{\ell,0}W)^2}{\lambda\varepsilon}\right)$ (with constant success
probability). Note that the expected objective is $2\lambda$-strongly convex and therefore, applying Corollary \ref{cor:solve_str_cvx}, we get the same (up to constant factors) bounds on estimation complexity of solving this problem by SQ algorithms.


\fi 
\section{Optimization without Lipschitzness}
\label{sec:range}
The estimation complexity bounds obtained for gradient descent-based methods depend polynomially either on the
the Lipschitz constant $L_0$ and the radius $R$ of $\K$
(unless $F$ is strongly convex).
In some cases such bounds are too large and instead we know that the range of functions in the support of the distribution is bounded, that is, $\max_{(x,y \in \K,\ v,w\in \W)} (f(x,v) - f(y,w)) \le 2B$ for some $B$. Without loss of generality we may assume that for all $w\in \W, f(\cdot,w) \in \F(\K,B)$.

\subsection{Random walks}
\label{sec:random-walk}
We first show that a simple extension of the random walk approach of \citenames{Kalai and Vempala}{KalaiV06} and \citenames{Lovasz and Vempala}{LovaszV06} can be used to address this setting. One advantage of this approach is that to optimize $F$ it requires only access to approximate values of $F$ (such an  oracle is also referred to as approximate zero-order oracle). Namely, a $\tau$-approximate value oracle for a function $F$ is the oracle that for every $x$ in the domain of $F$, returns value $v$ such that $|v - F(x)| \leq \tau$.

We note that the random walk based approach was also (independently\footnote{The statement of our result and proof sketch were included by the authors for completeness in the appendix of \cite[v2]{FeldmanPV:13}.}) used in a recent work of \citenames{Belloni \etal}{BelloniLNR15}. Their work includes an optimized and detailed analysis of this approach and hence we only give a brief outline of the proof here.
\begin{thm}
\label{thm:random-walk-zero}
There is an algorithm that with probability at least $2/3$, given any convex program $\min_{x \in \K} F(x)$ in $\R^d$ where $\forall x\in \K,\ |F(x)| \leq 1$ and $\K$ is given by a membership oracle with the guarantee that $ \B_2^d(R_0) \subseteq \K \subseteq \B_2^d(R_1)$, outputs an $\eps$-optimal solution in time $\poly(d, \frac{1}{\eps}, \log{(R_1/R_0)})$ using $\poly(d, \frac{1}{\eps})$ queries to $(\eps/d)$-approximate value oracle.
\end{thm}
\begin{proof}
Let $x^{\ast} = \argmin_{x \in \K} F(x)$ and $F^{\ast} = F(x^\ast)$. The basic idea is to sample from a distribution that has most of its measure on points with $F(x) \le F^{\ast} + \eps$. To do this, we use the random walk approach as in \cite{KalaiV06, LovaszV06} with a minor extension. The algorithm performs a random walk whose stationary distribution is proportional to $g_\alpha(x)=e^{-\alpha F(x)}$, with $g(x)=e^{-F(x)}$. Each step of the walk is a function evaluation. Noting that $e^{-\alpha F(x)}$ is a logconcave function, the number of steps is
$\poly(d, \log \alpha, \beta)$ to get a point from a distribution within total variation distance $\beta$ of the target distribution.
\remove{
Further, as shown in \cite{KalaiV06}, a random point $\bx$ from the target distribution satisfies:
\[
\E[F(\bx)] \le F^* + d/\alpha.
\]
Thus, setting $\alpha = d/\eps$ suffices.}
Applying Lemma 5.1 from \cite{LovaszV06} (which is based on Lemma 5.16 from \cite{LV07}) with $B=2$ to $g_\alpha$ with $\alpha = 4(d+\ln(1/\delta))/\eps$, we have (note that $\alpha$ corresponds to $a_m= \frac{1}{B}(1+1/\sqrt{d})^m$ in that statement).
\equ{
\label{eq:walk-hit-prob}
\pr[g(\bx) < e^{-\eps}\cdot g(x^\ast)] \le \delta \left(\frac{2}{e}\right)^{d-1}.
}
Therefore, the probability that a random point $\bx$ sampled proportionately to $g_\alpha(x)$ does not satisfy $F(\bx) < F^{\ast} + \eps$ is at most $\delta(2/e)^{d-1}$.

Now we turn to the extension, which arises because we can only evaluate $F(x)$ approximately through the oracle. We assume w.l.o.g.~that the value oracle is consistent in its answers (i.e., returns the same value on the same point). The value returned by the oracle
$\tilde{F}(x)$ satisfies $|F(x) - \tilde{F}(x)| \le \eps/d$. The stationary distribution is now proportional to $\tilde{g}_\alpha(x) = e^{-\alpha \tilde{F}(x)}$ and satisfies
\begin{equation}\label{density-ratio}
\frac{\tilde{g}_\alpha(x)}{g_\alpha(x)} = e^{-\alpha(\tilde{F}(x)-F(x))} \le e^{\alpha \frac{\eps}{d}} \le e^5.
\end{equation}

We now argue that with large probability, the random walk with the approximate evaluation oracle will visit a point $x$ where $F$ has value at most $F^\ast + \eps$. Assuming that a random walk gives samples from a distribution (sufficiently close to being) proportional to $\tilde g_\alpha$, from property (\ref{density-ratio}), the probability of the set $\{x \, : \, g(x) > e^{-\eps}\cdot g(x^\ast)\}$ is at most a factor of $e^{10}$ higher than for the distribution proportional to $g_\alpha$ (given in eq.~\eqref{eq:walk-hit-prob}). Therefore with a small increase in the number of steps a random point from the walk will visit the set where $F$ has value of at most $F^\ast + \eps$ with high probability. Thus the minimum function value that can be achieved is at most $F^{\ast} + \eps+2\eps/d$.

Finally, we need the random walk to mix rapidly for the extension.
Note that $\tilde{F}(x)$ is approximately convex, \ie for any $x,y \in \K$ and any $\lambda \in [0,1]$, we have
\begin{equation}\label{approx-convex}
\tilde{F}(\lambda x + (1-\lambda) y) \le \lambda \tilde{F}(x) + (1-\lambda)\tilde{F}(y) + 2\eps/d.
\end{equation}
and therefore $\tilde{g}_\alpha$ is a near-logconcave function that satisfies, for any $x,y \in \K$ and  $\lambda \in [0,1]$,
\[
\tilde{g}_\alpha(\lambda x + (1-\lambda) y) \ge e^{-2\alpha\eps/d}\cdot \tilde{g}_\alpha(x)^\lambda \tilde{g}_\alpha(x)^{1-\lambda} \ge e^{-10}\cdot \tilde{g}_\alpha(x)^\lambda \tilde{g}_\alpha(x)^{1-\lambda}.
\]
As a result, as shown by 
\citenames{Applegate and Kannan}{ApplegateK91}, it admits an isoperimetric inequality that is weaker than that for logconcave functions by a factor of $e^{10}$.  For the grid walk, as analyzed by them, this increases the convergence time by a factor of at most $e^{20}$. The grid walk's convergence also depends (logarithmically) on the Lipshitz constant of $\tilde{g}_\alpha$. This dependence is avoided by the ball walk, whose convergence is again based on the isoperimetric inequality, as well as on local properties, namely on the $1$-step distribution of the walk. It can be verified that the analysis of the ball walk (e.g., as in \cite{LV07}) can be adapted to near-logconcave functions with an additional factor of $O(1)$ in the mixing time.
\end{proof}

Going back to the stochastic setting, let $F(x) = \E_D[f(x,\bw)]$. If $\forall w$, $f(\cdot,w) \in \F(\K,B)$ then a single query $f(x,w)$ to $\STAT(\tau/B)$ is equivalent to a query to a $\tau$-approximate value oracle for $F(x)$.
\begin{corollary}
\label{cor:random-walk}
There is an algorithm that for any distribution $D$ over $\W$ and convex program $\min_{x \in \K}\{F(x) \doteq \E_{\bw \sim D}[f(x,\bw)] \}$ in $\R^d$ where $\forall w$, $f(\cdot,w) \in \F(\K,B)$ and $\K$ is given by a membership oracle with the guarantee that $ \B_2^d(R_0) \subseteq \K \subseteq \B_2^d(R_1)$,  with probability at least $2/3$, outputs an $\eps$-optimal solution in time $\poly(d, \frac{B}{\eps}, \log{(R_1/R_0)})$ using $\poly(d, \frac{B}{\eps})$ queries to $\STAT(\eps/(dB))$.
\end{corollary}

\paragraph{SQ vs approximate value oracle:}
We point out that $\tau$-approximate value oracle is strictly weaker than $\STAT(\tau)$. This follows from a simple result of
Nemirovsky and Yudin \cite[p.360]{nemirovsky1983problem} who show that linear optimization over $\B_2^d$ with $\tau$-approximate value oracle requires $\tau = \Omega(\sqrt{\log q} \cdot \eps/\sqrt{d})$ for any algorithm using $q$ queries. Together with our upper bounds in Section \ref{sec:linear} this implies that approximate value oracle is weaker than $\STAT$.

\subsection{Center-of-Gravity}
An alternative and simpler technique to establish the $O(d^2B^2/\eps^2)$ upper bound on the estimation complexity for $B$-bounded-range functions is to use cutting-plane methods, more specifically, the classic center-of-gravity method, originally proposed by \citenames{Levin}{Levin:1965}.

We introduce some notation. Given a convex body $\K$, let $\bf x$ be a uniformly and randomly chosen point from $\K$. Let $z(\K) \doteq \E[\bx]$ and $A(\K) \doteq \E[(\bx-z(\K))(\bx-z(\K))^T]$ be the center of gravity and covariance matrix of $\K$ respectively. We define the (origin-centered) inertial ellipsoid of ${\cal K}$ as
${\cal E}_{\cal K} \doteq \{y\, :\, y^TA(\K)^{-1}y \le 1\}$.

The classic center-of-gravity method starts with $G^0\doteq \K$ and iteratively computes a progressively smaller body containing the optimum of the convex program. We call such a body a {\em localizer}. Given a localizer $G^{t-1}$, for $t \geq 1$, the algorithm
computes $x^t= z(G^{t-1})$ and defines the new localizer to be $$G^t \doteq G^{t-1} \cap\{y\in\R^d \cond \la \nabla F(x^t), y-x^t \ra \leq 0\}.$$
It is known that that any halfspace containing the center of gravity of a convex body
contains at least $1/e$ of its volume \cite{Grunbaum:1960}, that is $\mbox{vol}(G^t) \leq \gamma \cdot\mbox{vol}(G^{t-1})$, where $\gamma = 1-1/e$. We call this property the {\em volumetric guarantee} with parameter $\gamma$.

The first and well-known issue we will deal with is that the exact center of gravity of $G^{t-1}$ is hard to compute. Instead, following the approach in \cite{Bertsimas:2004}, we will let $x^t$ be an approximate center-of-gravity. For such an approximate center we will have a volumetric guarantee with somewhat larger parameter $\gamma$.

The more significant issue is that we do not have access to the exact value of $\nabla F(x^t)$. Instead will show how to compute an approximate gradient $\tilde g(x^t)$ satisfying for all $y\in G^t$,
\begin{equation} \label{CoG_approx_grad}
|\langle \tilde g(x^t)-\nabla F(x^t), y-x^t\rangle|\leq \eta.
\end{equation}
Notice that this is a weaker condition than the one required by \eqref{ApproxSubgrad}: first, we only impose the
approximation on the localizer; second, the gradient approximation
is centered at $x^t$. These two features are crucial for our results.

Condition \eqref{CoG_approx_grad} implies that for all
$y\in G^{t-1}\setminus G^t$,
$$ F(y) \geq F(x^t) +\la \nabla F(x^t),y-x^t\ra
	\geq F(x^t) + \la \tilde g(x^t),y-x^t\ra -\eta > F(x^t)-\eta.$$
Therefore we will lose at most $\eta$ by discarding points in $G^{t-1}\setminus G^t$.

Plugging this observation into the standard analysis of the center-of-gravity method (see, \eg \cite[Chapter 2]{Nemirovski:1994}) yields the following result.
\begin{proposition}\label{prop:CoG}
For $B>0$, let $\K\subseteq \R^d$ be a convex body, and
$F \in \F(\K,B)$. Let $x^1,x^2,\ldots$ and $\tilde g(x^1), \tilde g(x^2), \ldots$ be a sequence of points and gradient estimates such that for
$G_0 \doteq \K$ and $G^t \doteq G^{t-1} \cap\{y\in\R^d \cond \la \tilde g(x^t), y-x^t \ra \leq 0\}$ for all $t \geq 1$, we
have a volumetric guarantee with parameter $\gamma<1$ and condition \eqref{CoG_approx_grad} is satisfied for some fixed $\eta>0$.
Let $\hat x^T\doteq \argmin_{t\in[T]} F(x^t)$, then
$$F( \hat x^T) -\min_{x\in\K} F(x) \leq \gamma^{T/d} \cdot 2B +\eta\ .$$ In particular, choosing $\eta=\varepsilon/2$, and
$T=\lceil d\log(\frac{1}{\gamma})\log(\frac{4B}{\varepsilon})\rceil$ gives $F(\hat x^T)- \min_{x\in\K} F(x) \leq \varepsilon$.
\end{proposition}

We now describe how to compute an approximate gradient satisfying condition \eqref{CoG_approx_grad}. We show that it suffices to find an ellipsoid $\cal E$ centered at $x^t$ such that $x^t + \cal E$ is included in $G^t$ and $G^t$ is included in $x^t + R \cdot \cal E$. The first condition, together with the bound on the range of functions in the support of the distribution, implies a bound on the ellipsoidal norm of the gradients. This allows us to use Theorem \ref{thm-l2-kashin} to estimate $\nabla F(x^t)$ in the ellipsoidal norm. The second condition can be used to translate the error in the ellipsoidal norm to the error $\eta$ over $G^t$ as required by condition \eqref{CoG_approx_grad}. Formally we prove the following lemma:
\begin{lemma} \label{lem:approx_grad_isotrop}
Let $G\subseteq\R^d$ be a convex body, $x\in G$, and
${\cal E}\subseteq \R^d$ be an origin-centered ellipsoid that satisfies
\[
R_0\cdot {\cal E} \subseteq (G-x) \subseteq R_1\cdot {\cal E}.
\]
Given
$F(x)=\E_{\bw}[f(x,\bw)]$ a convex function on $G$
such that for all $w\in\W$, $f(\cdot,w)\in \F(\K,B)$,
we can compute a vector $\tilde g(x)$ satisfying
\eqref{CoG_approx_grad} in polynomial time using $2d$
queries to $\STAT\left(\Omega\left(\frac{\eta}{[R_1/R_0]B}\right)\right)$.
\begin{proof}
Let us first bound the norm of the gradients, using the norm dual to
the one induced by the ellipsoid ${\cal E}$.
\begin{eqnarray*}
\|\nabla f(x,w)\|_{\cal E,\ast} &=& \sup_{y\in {\cal E}}\la \nabla f(x,w),y \ra
\,\,\leq\,\, \frac{1}{R_0}\sup_{y\in G}\la \nabla f(x,w),y-x \ra \\
&\leq&  \frac{1}{R_0}\sup_{y\in G} [f(y,w)-f(x,w)]
\,\,\leq\,\, \frac{2B}{R_0}.
\end{eqnarray*}

Next we observe that for any vector $\tilde g$,
\begin{eqnarray*}
\sup_{y\in G}\la \nabla F(x)-\tilde g, y -x \ra
	&=& R_1 \sup_{y\in G}\left\la \nabla F(x)-\tilde g, \frac{y-x}{R_1} \right\ra
	\,\,\leq\,\, R_1 \sup_{y\in {\cal E}} \la \nabla F(x)-\tilde g, y  \ra \\
	&=& R_1\, \|\nabla F(x)-\tilde g\|_{\cal E,\ast}.
\end{eqnarray*}

From this we reduce obtaining $\tilde g(x)$ satisfying
\eqref{CoG_approx_grad} to a mean estimation problem in
an ellipsoidal norm with error $R_0\eta/[2R_1B]$, which by Theorem \ref{thm-l2-kashin} (with Lemma \ref{lem:norm-embed}) can be done using $2d$ queries
to $\STAT\left(\Omega\left(\frac{\eta}{[R_1/R_0]B}\right)\right)$.
\end{proof}
\end{lemma}

It is known that if $x^t= z(G^t)$ then the inertial ellipsoid of $G^t$ has the desired property with the ratio of the radii being $d$.
\begin{theorem}\cite{Kannan:1995} \label{thm:KLSLemma}
For any convex body $G\subseteq\R^d$, ${\cal E}_G$ (the inertial ellipsoid of $G$) satisfies
$$ \sqrt{ \dfrac{d+2}{d} }\cdot {\cal E}_G \subseteq (G-z(G)) \subseteq \sqrt{d(d+2)}\cdot {\cal E}_G. $$
\end{theorem}
This means that estimates of the gradients sufficient for executing the exact center-of-gravity method can be obtained using SQs with estimation complexity of $O(d^2B^2/\varepsilon^2)$.

Finally, before we can apply Theorem \ref{prop:CoG}, we note that instead of $\hat x^T\doteq \argmin_{t\in[T]} F(x^t)$ we can compute $\tilde x^T = \argmin_{t\in[T]} \tilde F(x^t)$ such that $F(\tilde x^T) \leq F( \hat x^T)+\eps/2$. This can be done by using $T$ queries to $\STAT(\eps/[4B])$ to obtain $\tilde F(x^t)$ such that $|\tilde F(x^t) - F(x^t)|\leq \eps/4$ for all $t\in[T]$. Plugging this into Theorem \ref{prop:CoG} we get the following (inefficient) SQ version of the center-of-gravity method.
\begin{proposition}
\label{thm:cog-sq}
Let $\K\subseteq \R^d$ be a convex body, and assume that for all $w\in\W$, $f(\cdot,w) \in \F(\K,B)$.
Then there is an algorithm that for every distribution $D$ over $\W$ finds
an $\eps$-optimal solution for the stochastic convex optimization problem $\min_{x\in\K}\{\E_{\bw\sim D}[f(x,\bw)] \}$ using $O(d^2\log(B/\varepsilon))$  queries to $\STAT(\Omega(\varepsilon/[Bd]))$.
\end{proposition}

\iffull \iffull
\subsubsection{Computational Efficiency}
\else \section{Computational Efficiency of the Center-of-Gravity Algorithm}
\label{sec:sqc-cog-efficient}
\fi

The algorithm described in Theorem \ref{thm:cog-sq} relies on the computation of the exact center of gravity and inertial ellipsoid for each localizer. Such computation is $\#$P-hard in general. We now describe a computationally efficient version of the center-of-gravity method that is based on computation of approximate center of gravity and inertial ellipsoid via random walks, an approach that was first proposed by \citenames{Bertsimas and Vempala}{Bertsimas:2004}.

We first observe that the volumetric guarantee is satisfied by any cut through an approximate center of gravity.
\begin{lemma}[\cite{Bertsimas:2004}]\label{lem:VolGuarantee}
For a convex body $G \subseteq \R^d$, let $z$ be any point s.t. $\|z-z(G)\|_{{\cal E}_G} = t$. Then, for any halfspace $H$ containing $z$,
\[
\vol(G \cap H) \ge \left(\frac{1}{e} - t \right)\vol(G).
\]
\end{lemma}

From this result, we know that it suffices to approximate the center of gravity in
the inertial ellipsoid norm in order to obtain the volumetric guarantee.

\citenames{Lovasz and Vempala}{LovaszV06b} show that for any convex body $G$ given by a membership oracle, a point $x \in G$ and $R_0, R_1$ s.t. $R_0\cdot\B_2^d \subseteq (G-x) \subseteq  R_1\cdot\B_2^d$, there is a sampling algorithm based on a random walk that outputs points that are within statistical distance $\alpha$ of the uniform distribution in time polynomial in $d, \log(1/\alpha), \log(R_1/R_0)$. The current best dependence on $d$ is $d^4$ for the first random point and $d^3$ for all subsequent points \cite{LV:2006}. Samples from such a random walk can be directly used to estimate the center of gravity and the inertial ellipsoid of $G$.

\begin{theorem}[\cite{LovaszV06b}] \label{thm:LovaszVempala}
There is a randomized algorithm that for any $\eps > 0, 1 > \delta > 0$, for a convex body $G$ given by a membership oracle and a point $x$ s.t. $R_0\cdot \B_2^d \subseteq (G-x) \subseteq R_1\cdot \B_2^d$, finds a point $z$ and an origin-centered ellipsoid ${\cal E}$ s.t. with probability at least $1-\delta$, $\|z-z(G)\|_{{\cal E}_{G}} \le \eps$ and ${\cal E} \subset {\cal E}_G \subset (1+\eps){\cal E}$. The algorithm uses $\tilde{O}(d^4\log(R_1/R_0)\log(1/\delta)/\eps^2)$ calls to the membership oracle.
\end{theorem}

We now show that an algorithm having the guarantees given in Theorem \ref{thm:cog-sq} can be implemented in time $\poly(d, B/\eps, \log(R_1/R_0))$. More formally,
\begin{theorem}
\label{thm:cog-sq-efficient}
Let $\K\subseteq \R^d$ be a convex body given by a membership oracle and a point $x$ s.t. $R_0\cdot \B_2^d \subseteq (\K-x) \subseteq R_1\cdot \B_2^d$, and assume that for all $w\in\W$, $f(\cdot,w) \in \F(\K,B)$.
Then there is an algorithm that for every distribution $D$ over $\W$ finds an $\eps$-optimal solution for the stochastic convex optimization problem $\min_{x\in\K}\{\E_{\bw\sim D}[f(x,\bw)] \}$ using $O(d^2\log(B/\varepsilon))$  queries to $\STAT(\Omega(\varepsilon/[Bd]))$. The algorithm succeeds with probability $\geq 2/3$ and runs in $\poly(d, B/\eps, \log(R_1/R_0))$ time.
\end{theorem}
\begin{proof}
Let the initial localizer be $G=\K$.
We will prove the following by induction: For every step of the method, if $G$ is
the current localizer then a membership oracle for $G$ can be implemented efficiently given a membership oracle for $\K$ and we can efficiently compute $x\in G$ such that, with probability at least $1-\delta$,
 \begin{equation} \label{sandwich}
R_0^{\prime}\cdot  \B_2^d \subseteq G-x \subseteq R_1^{\prime}\cdot \B_2^d,
\end{equation}
where $R_1^{\prime}/R_0^{\prime} \leq \max\{R_1/R_0,4d\}$.
We first note that the basis of the induction holds by the assumptions of the theorem. We next show that the assumption of the induction allows us to compute the desired approximations to the center of gravity and the inertial ellipsoid which in turn will allow us to prove the inductive step.

Since $G$ satisfies the assumptions of Theorem \ref{thm:LovaszVempala}, we can obtain in polynomial time
(with probability $1-\delta$) an approximate center $z$ and ellipsoid ${\cal E}$ satisfying
$\|z-z(G)\|_{{\cal E}_G}\leq \chi$ and ${\cal E}\subseteq {\cal E}_{G}\subseteq (1+\chi) {\cal E}$, where
$\chi\doteq1/e-1/3$. By Lemma \ref{lem:VolGuarantee} and $\|z-z(G)\|_{{\cal E}_G}\leq \chi$, we get that volumetric guarantee holds for the next localizer $G^{\prime}$ with parameter $\gamma=2/3$.

Let us now observe that
$$(\sqrt{(d+2)/d}-\chi)\cdot {\cal E} + z \subseteq
\sqrt{(d+2)/d} \cdot {\cal E}_{G}+z(G)\subseteq G.$$
We only prove the first inclusion, as the second one holds by Theorem
\ref{thm:KLSLemma}. Let $y\in\alpha {\cal E}+z$ (where $\alpha=\sqrt{(d+2)/d}-\chi)$). Now we have
$\|y-z(G)\|_{{\cal E}_G}\leq \|y-z\|_{{\cal E}_G}+\|z-z(G)\|_{{\cal E}_G}\leq \|y-z\|_{\cal E}+\chi
\leq \alpha+\chi=\sqrt{(d+2)/d}$. Similarly, we can prove that
$$G-z\subseteq \sqrt{d(d+2)}\cdot {\cal E}_{G} +(z(G)-z)
\subseteq  (\sqrt{d(d+2)}+\chi) \cdot {\cal E}_{G}\subseteq
 (1+\chi)(\sqrt{d(d+2)}+\chi) \cdot {\cal E}.$$
Denoting $r_0 \doteq \sqrt{(d+2)/d}-\chi$ and $r_1 \doteq (1+\chi)(\sqrt{d(d+2)}+\chi)$ we obtain that
$r_0\cdot {\cal E} \subseteq G-z \subseteq r_1\cdot {\cal E}$, where
$\frac{r_1}{r_0} =\frac{(1+\chi)(\sqrt{d(d+2)}+\chi)}{\sqrt{(d+2)/d}-\chi} \leq \frac{3}{2} d$. By Lemma \ref{lem:approx_grad_isotrop} this implies that using $2d$ queries to $\STAT(\Omega(\varepsilon/[Bd]))$ we can obtain an estimate $\tilde g$ of $\nabla F(z)$ that suffices for executing the approximate center-of-gravity method.

We finish the proof by establishing the inductive step.  Let the new localizer $G^{\prime}$
 be defined as $G$ after removing the cut through $z$ given by $\tilde g$ and transformed
by the affine transformation induced by $z$ and ${\cal E}$ (that is mapping $z$ to the origin and $\cal E$ to $\B_2^d$).
Notice that after the transformation
$r_0 \cdot\B_2^d \subseteq \tilde G \subseteq r_1 \cdot \B_2^d$, where $\tilde G$ denotes $G$ after the affine transformation. $G^{\prime}$ is obtained from $\tilde G$ by a cut though the origin. This implies that $G^{\prime}$ contains a ball of radius $r_0/2$ which is inscribed in the half of $r_0 \cdot\B_2^d$ that is contained in $G^{\prime}$. Let $x^{\prime}$ denote the center of this contained ball (which can be easily computed from $\tilde g$, $z$ and $\cal E$). It is also easy to see that a ball of radius $r_0/2+r_1$ centered at $x^{\prime}$ contains $G^{\prime}$. Hence $G^{\prime} -x^{\prime}$ is sandwiched by two Euclidean balls with the ratio of radii being
$(r_1+ r_0/2)/(r_0/2) \leq 4d$. Also notice that since a membership oracle for $\K$ is given and the number of iterations of this method is $O(d \log(4B/\eps))$ then a membership oracle for $G^{\prime}$ can be efficiently computed.

Finally, choosing the confidence parameter $\delta$ inversely proportional to the
number of iterations of the method guarantees a constant success probability.
\end{proof}

\fi

\section{Applications}
\label{sec:apps}
In this section we describe several applications of our results. We start by showing that our algorithms together with lower bounds for SQ algorithms give lower bounds against convex programs. We then give several easy examples of using upper bounds in other contexts. (1) New SQ implementation of algorithms for learning halfspaces that eliminate the linear dependence on the dimension in previous work. (2) Algorithms for high-dimensional mean estimation with local differential privacy that re-derive and generalize existing bounds. We also give the first algorithm for solving general stochastic convex programs with local differential privacy. (3) Strengthening and generalization of algorithms for answering sequences of convex minimization queries differentially privately given in \cite{Ullman15}.

Additional applications in settings where SQ algorithms are used can be derived easily. For example, our results immediately imply that an algorithm for answering a sequence of adaptively chosen SQs (such as those given in \cite{DworkFHPRR14:arxiv,DworkFHPRR15:arxiv,BassilyNSSSU15} can be used to solve a sequence of adaptively chosen stochastic convex minimization problems. This question that has been recently studied by \citenames{Bassily \etal}{BassilyNSSSU15} and our bounds can be easily seen to strengthen and generalize some of their results (see Sec.~\ref{sec:app-dp-queries} for an analogous comparison).

\subsection{Lower Bounds}
\label{sec:app-csp}
We describe a generic approach to combining SQ algorithms for stochastic convex optimization with lower bounds against SQ algorithms to obtain lower bounds against certain type of convex programs. These lower bounds are for problems in which we are given a set of cost functions $(v_i)_{i=1}^m$ from some collection of functions $V$ over a set of ``solutions" $Z$ and the goal is to (approximately) minimize or maximize $\frac{1}{m} \sum_{i\in [m]} v_i(z)$ for $z\in Z$. Here either $Z$ is non-convex  or functions in $V$ are non-convex (or both). Naturally, this captures loss (or error) of a model in machine learning and also the number of (un)satisfied constraints in constraint satisfaction problems (CSPs).  For example, in the MAX-CUT problem $z \in \zo^n$ represents a subset of vertices and $V$ consists of $n \choose 2$, ``$z_i \neq z_j$" predicates.

A standard approach to such non-convex problems is to map $Z$ to a convex body $\K \subseteq \R^d$ and map $V$ to convex functions over $\K$ in such a way that the resulting convex optimization problem can be solved efficiently and the solution allows one to recover a ``good" solution to the original problem. For example, by ensuring that the mappings, $M:Z \rightarrow\K$ and  $T:V \rightarrow \F$ satisfy: for all $z$ and $v$, $v(z)= (T(v))(M(z))$ and for all instances of the problem $(v_i)_{i=1}^m$,
\equ{\min_{z \in Z}\frac{1}{m} \sum_{i\in [m]} v_i(z)  - \min_{x \in \K}\frac{1}{m} \sum_{i\in [m]} (T(v_i))(x) < \eps .\label{eq:value-preserve}}
 (Approximation is also often stated in terms of the ratio between the original and relaxed values and referred to as the integrality gap. This distinction will not be essential for our discussion.) The goal of lower bounds against such approaches is to show that specific mappings (or classes of mappings) will not allow solving the original problem via this approach, \eg have a large integrality gap.

The class of convex relaxations for which our approach gives lower bounds are those that are ``easy" for SQ algorithms. Accordingly, we  define the following measure of complexity of convex optimization problems.
\begin{definition}
For  an SQ oracle $\cal O$, $t>0$ and a problem $P$ over distributions we say that $P\in \Stat(\cO, t)$ if $P$ can be solved using at most $t$ queries to $\cO$ for the input distribution. For a convex set $\K$, a set $\F$ of convex functions over $\K$ and $\eps>0$ we denote by $\Opt(\K,\F,\eps)$ the problem of finding, for every distribution $D$ over $\F$, $x^*$ such that $F(x^*) \leq \min_{x\in \K} F(x) + \eps$, where $F(x) \doteq \E_{f \sim D}[f(x)]$. 
\end{definition}

For simplicity, let's focus on the decision problem\footnote{Indeed, hardness results for optimization are commonly obtained via hardness results for appropriately chosen decision problems.} in which the input distribution $D$ belongs to $\D = \D_+ \cup \D_-$. Let $P(\D_+,\D_-)$ denote the problem of deciding whether the input distribution is in $\D_+$ or $\D_-$. This is a distributional version of a {\em promise} problem in which an instance can be of two types (for example completely satisfiable and one in which at most half of the constraints can be simultaneously satisfied). Statistical query complexity upper bounds are preserved under pointwise mappings of the domain elements and therefore an upper bound on the SQ complexity of a stochastic optimization problem implies an upper bound on any problem that can be reduced pointwise to the stochastic optimization problem.
\begin{theorem}
\label{thm:lower-reduction}
Let $\D_+$ and $\D_-$ be two sets of distributions over a collection of functions $V$ on the domain $Z$. Assume that for some $\K$ and $\F$ there exists a mapping $T:V \rightarrow \F$ such that for all $D\in \D^+$, $\min_{x\in \K}\E_{v \sim D}[(T(v))(x)] > \alpha$ and for all $D\in \D^-$, $\min_{x\in \K}\E_{v \sim D}[(T(v))(x)] \leq 0$. Then if for an SQ oracle $\cO$ and $t$ we have a lower bound $P(\D_+,\D_-) \not\in \Stat(\cO,t)$ then we obtain that $\Opt(\K,\F,\alpha/2) \not\in \Stat(\cO,t)$.
\end{theorem}
The conclusion of this theorem, namely $\Opt(\K,\F,\alpha/2) \not\in \Stat(\cO,t)$, together with upper bounds from previous sections can be translated into a variety of concrete lower bounds on the dimension, radius, smoothness and other properties of convex relaxations to which one can map (pointwise) instances of $P(\D_+,\D_-)$. We also emphasize that the resulting lower bounds are structural and do not assume that the convex program is solved using an SQ oracle or efficiently.

Note that the assumptions on the mapping in Thm.~\ref{thm:lower-reduction} are stated for the expected value $\min_{x\in \K}\E_{v \sim D}[(T(v))(x)]$ rather than for averages over given relaxed cost functions as in eq.~\eqref{eq:value-preserve}. However, for a sufficiently large number of samples $m$, for every $x$ the average over random samples $\frac{1}{m} \sum_{i\in [m]} (T(v_i))(x)$  is close to the expectation $\E_{v \sim D}[(T(v))(x)]$. Therefore, the condition can be equivalently reformulated in terms of the average over a sufficiently large number of samples drawn i.i.d.~from $D$.

\paragraph{Lower bounds for planted CSPs:}
We now describe an instantiation of this approach using lower bounds for constraint satisfaction problems established in \cite{FeldmanPV:13}. \citenames{Feldman \etal}{FeldmanPV:13} describe implications of their lower bounds for convex relaxations using results for more general (non-Lipschitz) stochastic convex optimization
and discuss their relationship to those for lift-and-project hierarchies (Sherali-Adams, Lov\'asz-Schrijver, Lasserre) of canonical LP/SDP formulations. Here we give examples of implications of our results for the Lipschitz case.

Let $Z = \on^n$ be the set of assignments to $n$ Boolean variables. A distributional $k$-CSP problem is defined by a set $\D$ of distributions over Boolean $k$-ary predicates.
One way to obtain a distribution over constraints is to first pick some assignment $z$ and then generate random constraints that are consistent with $z$ (or depend on $z$ in some other predetermined way). In this way we can obtain a family of distributions $\D$ parameterized by a ``planted" assignment $z$. Two standard examples of such instances are planted $k$-SAT (\eg \cite{coja2010efficient}) and the pseudorandom generator based on Goldreich's proposal for one-way functions \cite{goldreich2000candidate}.

Associated with every family created in this way is a complexity parameter $r$ which, as shown in \cite{FeldmanPV:13}, characterizes the SQ complexity of finding the planted assignment $z$, or even distinguishing between a distribution in $\D$ and a uniform distribution over the same type of $k$-ary constraints. This is not crucial for discussion here but, roughly, the parameter $r$ is the largest value $r$ for which the generated distribution over variables in the constraint is $(r-1)$-wise independent. In particular, random and uniform $k$-XOR constraints (consistent with an assignment) have complexity $k$. The lower bound in \cite{FeldmanPV:13} can be (somewhat informally) restated as follows.

\begin{theorem}[\cite{FeldmanPV:13}]\label{thm:lower-bound-csp}
Let $\D=\{D_z\}_{z\in \on^n}$ be a set of ``planted" distributions over $k$-ary constraints of complexity $r$ and let $U_k$ be the uniform distribution on (the same) $k$-ary constraints. Then any SQ algorithm that, given access to a distribution $D \in \D \cup \{U_k\}$ decides correctly whether $D = D_z$ or $D=U_k$ needs $\Omega(t)$ calls to $\VSTAT(\frac{n^{r}}{(\log t)^{r}})$ for any $t \geq 1$.
\end{theorem}

Combining this with Theorem \ref{thm:lower-reduction} we get the following general statement:
\begin{theorem}\label{thm:lower-bound-convex}
Let $\D=\{D_z\}_{z\in \on^n}$ be a set of ``planted" distributions over $k$-ary constraints of complexity $r$ and let $U_k$ be the uniform distribution on (the same) $k$-ary constraints. Assume that there exists a mapping $T$ that maps each constraint $C$ to a convex function $f_C \in \F$ over some convex $d$-dimensional set $\K$ such that for all $z \in \on^n$, $\min_{x\in \K}\E_{C \sim D_z}[f_C(x)] \leq 0$ and $\min_{x\in \K}\E_{C \sim U_k}[f_C(x)] > \alpha$. Then for every $t\geq 1$, $\Opt(\K,\F,\alpha/2) \not\in \Stat(\VSTAT(\frac{n^{r}}{(\log t)^{r}}),\Omega(t))$.
\end{theorem}

Note that in the context of convex minimization that we consider here, it is more natural to think of the relaxation as minimizing the number of unsatisfied constraints (although if the objective function is linear then the claim also applies to maximization over $\K$).  We now instantiate this statement for solving the $k$-SAT problem via a convex program in the class $\F_{\|\cdot\|_p}^0(\B_p^d,1)$ (see Sec.~\ref{sec:gradient}). Let $\C_k$ denote the set of all $k$-clauses (OR of $k$ distinct variables or their negations). Let $U_k$ be the uniform distribution over $\C_k$.
\begin{corollary}
\label{cor:k-sat}
There exists a family of distributions $\D=\{D_z\}_{z\in \on^n}$ over $\C_k$ such that the support of $D_z$ is satisfied by $z$ with the following property: For every $p\in[1,2]$, if there exists a mapping $T:\C_k \rightarrow \F_{\|\cdot\|_p}^0(\B_p^d,1)$ such that for all $z$, $\min_{x\in \B_p^d}\E_{C \sim D_z}[(T(C))(x)] \leq 0$ and $\min_{x\in \B_p^d}\E_{C \sim U_k}[(T(C))(x)] > \eps$ then $\eps = \tilde{O}\left((n/\log(d))^{-k/2}\right)$ or, equivalently, $d = 2^{\tilde{\Omega}(n \cdot \eps^{2/k})}$.
\end{corollary}
This lower bound excludes embeddings in exponentially high (\eg $2^{n^{1/4}}$) dimension for which the lowest value of the program for unsatisfiable instances differs from that for satisfiable instances by more than $n^{-k/4}$ (note that the range of functions in $\F_{\|\cdot\|_p}^0(\B_p^d,1)$ can be as large as $[-1,1]$ so this is a normalized additive gap). For comparison, in the original problem the values of these two types of instances are $1$ and $\approx 1- 2^{-k}$. In particular, this implies that the integrality gap is $1/(1-2^{-k}) - o(1)$ (which is optimal).

We note that the problem described in Cor.~\ref{cor:k-sat} is easier than the distributional $k$-SAT refutation problem, where $\D$ contains all distributions with satisfiable support. Therefore the assumptions of Cor.~\ref{cor:lower-convex-program-norm} that we stated in the introduction imply the assumptions of Cor.~\ref{cor:k-sat}.

Similarly, we can use the results of Sec.~\ref{sec:range} to obtain the following lower bound on the dimension of any convex relaxation:
\begin{corollary}
\label{cor:k-sat-range}
There exists a family of distributions $\D=\{D_z\}_{z\in \on^n}$ over $\C_k$ such that the support of $D_z$ is satisfied by $z$ with the following property: For every convex body $\K \subseteq \R^d$, if there exists a mapping $T:\C_k \rightarrow \F(\K,1)$ such that for all $z$, $\min_{x\in \K}\E_{C \sim D_z}[(T(C))(x)] \leq 0$ and $\min_{x\in \K}\E_{C \sim U_k}[(T(C))(x)] > \eps$ then $d = \tilde{\Omega}\left(n^{k/2} \cdot \eps\right)$.
\end{corollary}

\subsection{Learning Halfspaces}
\label{sec:halfspaces}
We now use our high-dimensional mean estimation algorithms to address the efficiency of SQ versions of online algorithms for learning halfspaces (also known as linear threshold functions). A linear threshold function is a Boolean function over $\R^d$ described  by a weight vector $w \in \R^d$ together with a threshold $\theta \in \R$ and defined as $f_{w,\theta}(x) \doteq \sgn(\la w,x\ra -\theta)$.

\paragraph{Margin Perceptron:} We start with the classic Perceptron algorithm \cite{Rosenblatt:58,Novikoff:62}. For simplicity, and without loss of generality we only consider the case of $\theta =0$. We describe a slightly more general version of the Perceptron algorithm that approximately maximizes the margin and is referred to as Margin Perceptron \cite{BalcanB06}. The Margin Perceptron with parameter $\eta$ works as follows. Initialize the weights $w^0= 0^d$. At round $t\geq 1$, given a vector $x^t$ and correct prediction $y^t\in \on$, if $y^t \cdot \la w^{t-1}, x^t \ra \geq \eta$, then we let $w^{t} = w^{t-1}$. Otherwise, we update $w^{t} = w^{t-1} + y^t x^t$.
The Perceptron algorithm corresponds to using this algorithm with $\eta =0$. This update rule has the following guarantee:
\begin{theorem}[\cite{BalcanB06}]
\label{thm:margin-per-updates}
Let $(x^1,y^1),\ldots,(x^t,y^t)$ be any sequence of examples in $\B_2^d(R) \times \on$ and assume that there exists a vector $w^* \in \B_2^d(W)$ such that for all $t$, $y^t \la w^*,x^t\ra \geq \gamma > 0$. Let $M$ be the number of rounds in which the Margin Perceptron with parameter $\eta$ updates the weights on this sequence of examples. Then $M \leq R^2 W^2/(\gamma-\eta)^2$.
\end{theorem}
The advantage of this version over the standard Perceptron is that it can be used to ensure that the final vector $w^t$ separates the positive examples from the negative ones with margin $\eta$ (as opposed to the plain Percetron which does not guarantee any margin). For example, by choosing $\eta = \gamma/2$ one can approximately maximize the margin while only paying a factor $4$ in the upper bound on the number of updates. This means that the halfspace produced by Margin-Perceptron has essentially the same properties as that produced by the SVM algorithm.


In PAC learning of halfspaces with margin assumption we are given random examples from a distribution $D$ over $\B_2^d(R) \times \on$. The distribution is assumed to be supported only on examples $(x,y)$ that for some vector $w^*$ satisfy $y \la w^*,x\ra \geq \gamma$. It has long been observed that a natural way to convert the Perceptron algorithm to the SQ setting is to use the mean vector of all counterexamples with Perceptron updates \cite{Bylander:94,BlumFKV:97}. Namely, update using the example $(\bar{x}^t,1)$, where $\bar{x}^t = \E_{(\bx,\by) \sim D}[ \by \cdot \bx \cond \by \la w^{t-1},\bx \ra < \eta]$. Naturally, by linearity of the expectation, we have that  $\la w^{t-1}, \bar{x}^t \ra < \eta$  and $\la w^*,\bar{x}^t \ra \geq \gamma$, and also, by convexity, that $\bar{x}^t \in \B_2^d(R)$. This implies that exactly the same analysis can be used for updates based on the mean counterexample vector. Naturally, we can only estimate $\bar{x}^t$ and hence our goal is to find an estimate that still allows the analysis to go through. In other words, we need to use statistical queries to find a vector $\tilde{x}$ which satisfies the conditions above (at least approximately).  The main difficulty here is preserving the condition $\la w^*,\tilde{x} \ra \geq \gamma$, since we do not know $w^*$. However, by finding a vector $\tilde{x}$ such that $\|\tilde{x} - \bar{x}^t\|_2 \leq \gamma/(3 W)$ we can ensure that $$\la w^*,\tilde{x} \ra=  \la w^*,\bar{x}^t \ra -  \la w^*,\bar{x}^t -\tilde{x} \ra \geq  \gamma - \|\tilde{x} - \bar{x}^t\|_2 \cdot \|w^*\|_2 \geq 2\gamma/3 .$$
We next note that conditions $\la w^{t-1}, \tilde{x} \ra < \eta$ and $\tilde{x} \in \B_2^d(R)$ are easy to preserve. These are known and convex constraints so we can always project $\tilde{x}$ to the (convex) intersection of these two closed convex sets. This can only decrease the distance to $\bar{x}^t$. This implies that, given an estimate $\tilde{x}$, such that $\|\tilde{x} - \bar{x}^t\|_2 \leq \gamma/(3 W)$ we can use Thm.~\ref{thm:margin-per-updates} with $\gamma' = 2\gamma/3$ to obtain an upper bound of $M \leq R^2 W^2/(2\gamma/3-\eta)^2$ on the number of updates.

Now, by definition, $$\E_{(\bx,\by) \sim D}[ \by \cdot \bx \cond \by \la w^{t-1},\bx \ra < \eta] = \frac{\E_{(\bx,\by) \sim D}[ \by \cdot \bx \cdot \ind{\by \la w^{t-1},\bx \ra < \eta}] }{\pr_{(\bx,\by) \sim D}[ \by \la w^{t-1},\bx \ra < \eta]} .$$

In PAC learning with error $\eps$ we can assume that $\alpha \doteq \pr_{(\bx,\by) \sim D}[ \by \la w^{t-1},\bx \ra < \eta] \geq \eps$ since otherwise the halfspace $f_{w^{t-1}}$ is a sufficiently accurate hypothesis (that is classifies at least a $1-\eps$ fraction of examples with margin at least $\eta$). This implies that it is sufficient to find a vector $\tilde{z}$ such that $\|\tilde{z} - \bar{z}\|_2 \leq \alpha \gamma/(3 W)$, where $\bar{z} = \E_{(\bx,\by) \sim D}[ \by \cdot \bx \cdot \ind{\by \la w^{t-1},\bx \ra < \eta}]$.

Now the distribution on $\by\cdot \bx \cdot \ind{\by \la w^{t-1},\bx \ra < \eta}$ is supported on $\B_2^d(R)$ and therefore using Theorem \ref{thm-l2-kashin} we can get the desired estimate using $2d$ queries to $\STAT(\Omega(\eps \gamma/(RW)))$. In other words, the estimation complexity of this implementation of Margin Perceptron is $O(RW/(\eps\gamma)^2)$. We make a further observation that the dependence of estimation complexity on $\eps$ can be reduced from $1/\eps^2$ to $1/\eps$ by using $\VSTAT$ in place of $\STAT$. This follows from Lemma \ref{lem:vstat-condition} which implies that we need to pay only linearly for conditioning on $\ind{\by \la w^{t-1},\bx \ra < \eta}$.  Altogether we get the following result which we for simplicity state for $\eta =\gamma/2$:
\begin{theorem}
\label{thm:opt-perceptron}
There exists an efficient algorithm {\sf Margin-Perceptron-SQ} that for every $\eps > 0$ and distribution $D$ over $\B_2^d(R) \times \on$ that is supported on examples $(x,y)$ such that for some vector $w^* \in \B_2^d(W)$ satisfy $y \la w^*,x\ra \geq \gamma$, outputs a halfspace $w$ such that $\pr_{(\bx,\by)\sim D}[\by\la w,\bx\ra < \gamma/2] \leq \eps$. {\sf Margin-Perceptron-SQ} uses $O(d(WR/\gamma)^2)$ queries to $\VSTAT(O((WR/\gamma)^2/\eps))$.
\end{theorem}
The estimation complexity of our algorithm is the same as the sample complexity of the PAC learning algorithm for learning large-margin halfspaces obtained via a standard online-to-batch conversion (\eg \cite{Cesa-BianchiCG04}).
SQ implementation of Perceptron were used to establish learnability of large-margin halfspaces with random classification noise \cite{Bylander:94} and to give a private version of Perceptron \cite{BlumDMN:05}. Perceptron is also the basis of SQ algorithms for learning halfspaces that do not require a margin assumption \cite{BlumFKV:97,DunaganV08}. All previous analyses that we are aware of used coordinate-wise estimation of $\bar{x}$ and resulted in estimation complexity bound of $O(d(WR/(\gamma\eps)^2)$. Perceptron and SVM algorithms are most commonly applied over a very large number of variables (such as when using a kernel) and the dependence of estimation complexity on $d$ would be prohibitive in such settings.

\paragraph{Online $p$-norm algorithms:}
The Perceptron algorithm can be seen as a member in the family of online $p$-norm algorithms \cite{GroveLS97} with $p=2$. The other famous member of this family is the Winnow algorithm \cite{Littlestone:87} which corresponds to $p=\infty$.
For $p \in [2,\infty]$, a $p$-norm algorithm is based on $p$-margin assumption: there exists $w^* \in \B_q^d(R)$ such that for each example $(x,y) \in  \B_p^d(R) \times \on$ we have $y \la w^*,x\ra \geq \gamma$. Under this assumption the upper bound on the number of updates is $O((WR/\gamma)^2)$ for $p \in [2,\infty)$ and $O(\log d \cdot (WR/\gamma)^2)$ for $p=\infty$.
Our $\ell_p$ mean estimation algorithms can be used in exactly the same way to (approximately) preserve the margin in this case giving us the following extension of Theorem \ref{thm:opt-perceptron}.
\begin{theorem}
\label{thm:norm-perceptron}
For every $p\in[2,\infty]$, there exists an efficient algorithm {\sf $p$-norm-SQ} that for every $\eps > 0$ and distribution $D$ over $\B_p^d(R) \times \on$ that is supported on examples $(x,y)$ that for some vector $w^* \in \B_q^d(W)$ satisfy $y \la w^*,x\ra \geq \gamma$,  outputs a halfspace $w$ such that $\pr_{(\bx,\by)\sim D}[\by\la w,\bx\ra < 0] \leq \eps$. For $p\in[2,\infty)$ {\sf $p$-norm-SQ} uses $O(d \log d (WR/\gamma)^2)$ queries to $\VSTAT(O(\log d (WR/\gamma)^2/\eps))$ and for $p =\infty$ {\sf $p$-norm-SQ} uses $O(d \log d (WR/\gamma)^2)$ queries to $\VSTAT(O((WR/\gamma)^2/\eps))$.
\end{theorem}
It is not hard to prove that margin can also be approximately maximized for these more general algorithms but we are not aware of an explicit statement of this in the literature.  We remark that to implement the Winnow algorithm, the update vector can be estimated via straightforward coordinate-wise statistical queries.

Many variants of the Perceptron and Winnow algorithms have been studied in the literature and applied in a variety of settings (\eg \cite{FreundSchapire:98,Servedio:99colt,DasguptaKM:09}). The analysis inevitably relies on a margin assumption (and its relaxations) and hence, we believe, can be implemented using SQs in a similar manner.

\newcommand{\LR}{{\mathrm{LR}}}

\subsection{Local Differential Privacy}
\label{sec:app-dp-local}
We now exploit the simulation of SQ algorithms by locally differentially private (LDP) algorithms \cite{KasiviswanathanLNRS11} to obtain new LDP mean estimation and optimization algorithms.

We first recall the definition of local differential privacy. In this model it is assumed that each data sample obtained by an analyst is randomized in a differentially private way.
\begin{definition}
An $\alpha$-local randomizer $R:\W \rightarrow \Z$ is a randomized algorithm that satisfies $\forall w\in \W$ and $z_1,z_2\in \Z$,
$\pr[R(w) = z_1] \leq e^\alpha \pr[R(w) = z_2]$. An $\LR_D$  oracle for distribution $D$ over $\W$ takes as an input a local randomizer $R$ and outputs a random value $z$ obtained by first choosing a random sample $w$ from $D$ and then outputting $R(w)$.  An algorithm is $\alpha$-local if it uses access only to $\LR_D$ oracle. Further, if the algorithm uses $n$ samples such that sample $i$ is obtained from   $\alpha_i$-randomizer $R_i$ then $\sum_{i\in [n]} \alpha_i \leq \alpha$.
\end{definition}
The composition properties of differential privacy imply that an $\alpha$-local algorithm is $\alpha$-differentially private \cite{DworkMNS:06}.

\citenames{Kasiviswanathan \etal}{KasiviswanathanLNRS11} show that one can simulate $\STAT_D(\tau)$ oracle with success probability $1-\delta$ by an $\alpha$-local algorithm using $n=O(\log(1/\delta)/(\alpha\tau)^2)$ samples from $\LR_D$ oracle. This has the following implication for simulating SQ algorithms.
\begin{theorem}[\cite{KasiviswanathanLNRS11}]
\label{thm:sq-2-LDP}
Let $\A_{SQ}$ be an algorithm that makes at most $t$ queries to $\STAT_D(\tau)$. Then for every $\alpha >0$ and $\delta >0$ there is an $\alpha$-local algorithm $\A$ that uses $n = O(t\log(t/\delta)/(\alpha\tau^2))$ samples from $\LR_D$ oracle and produces the same output as $\A_{SQ}$ (for some answers of $\STAT_D(\tau)$) with probability at least $1-\delta$.
\end{theorem}
\citenames{Kasiviswanathan \etal}{KasiviswanathanLNRS11} also prove a converse of this theorem that uses $n$ queries to $\STAT(\Theta(e^{2\alpha}\delta/n))$ to simulate $n$ samples of an $\alpha$-local algorithm with probability $1-\delta$. The high accuracy requirement of this simulation implies that it is unlikely to give a useful SQ algorithm from an LDP algorithm.

\paragraph{Mean estimation:}
\citenames{Duchi \etal}{DuchiJW:13focs} give $\alpha$-local algorithms for $\ell_2$ mean estimation using $O(d/(\varepsilon\alpha)^2)$ samples  $\ell_{\infty}$ mean estimation using $O(d\log d/(\varepsilon\alpha)^2)$ samples (their bounds are for the expected error $\eps$ but we can equivalently treat them as ensuring error $\eps$ with probability at least $2/3$). They also prove that these bounds are tight.
We observe that a direct combination of Thm.~\ref{thm:sq-2-LDP} with our mean estimation algorithms implies algorithms with nearly the same sample complexity (up to constants for $q=\infty$ and up to a $O(\log d)$ factor for $q=2$). In addition, we can as easily obtain mean estimation results for other norms. For example we can fill the $q \in (2,\infty)$ regime easily.
\begin{corollary}
For every $\alpha$ and $q \in [2,\infty]$ there is an $\alpha$-local algorithm for $\ell_q$ mean estimation with error $\eps$ and success probability of at least $2/3$ that uses $n$ samples from $\LR_D$ where:
\begin{itemize}
\item For $q=2$ and $q = \infty$, $n = O(d\log d/(\alpha\eps)^2)$.
\item For $q\in (2,\infty)$, $n = O(d\log^2 d/(\alpha\eps)^2)$.
\end{itemize}
\end{corollary}

\paragraph{Convex optimization:}
\citenames{Duchi \etal}{DuchiJW14} give locally private versions of the mirror-descent algorithm for $\ell_1$ setup and gradient descent for $\ell_2$ setup. Their algorithms achieve the guarantees of the (non-private) stochastic versions of these algorithms at the expense of using $O(d/\alpha^2)$ times more samples. For example for the mirror-descent over the $\B_1^d$ the bound is $O(d\log d(RW/\varepsilon\alpha)^2)$ samples. $\alpha$-local simulation of our algorithms from Sec.~\ref{sec:gradient} can be used to obtain $\alpha$-local algorithms for these problems. However such simulation leads to an additional factor corresponding to the number of iterations of the algorithm. For example for mirror-descent in $\ell_1$ setup we will obtain and $O(d\log d /\alpha^2 \cdot (RW/\varepsilon)^4)$ bound. At the same time our results in Sec.~\ref{sec:gradient} and Sec.~\ref{sec:range} are substantially more general. In particular, our center-of-gravity-based algorithm (Thm.~\ref{thm:cog-sq-efficient}) gives the first $\alpha$-local algorithm for non-Lipschitz setting.
\begin{corollary}
\label{cor:opt-ldp}
Let $\alpha >0,\eps >0$. There is an $\alpha$-local algorithm that for any convex body $\K$ given by a membership oracle with the guarantee that $\B_2^d(R_0) \subseteq \K \subseteq \B_2^d(R_1)$ and any convex program $\min_{x \in \K} \E_{\bw \sim D}[f(x,\bw)]$ in $\R^d$, where $\forall w$, $f(\cdot,w) \in \F(\K,B)$, with probability at least $2/3$, outputs an $\eps$-optimal solution to the program in time $\poly(d, \frac{B}{\alpha \eps}, \log{(R_1/R_0)})$ and using $n = \tilde{O}(d^4 B^2/(\eps^2 \alpha^2))$ samples from $\LR_D$.
\end{corollary}
We note that a closely related application is also discussed in \cite{BelloniLNR15}. It relies on the random walk-based approximate value oracle optimization algorithm similar to the one we outlined in Sec.~\ref{sec:random-walk}. Known optimization algorithms that use only the approximate value oracle require a substantially larger number of queries than our algorithm  in Thm.~\ref{thm:cog-sq-efficient} and hence need a substantially larger number of samples to implement (specifically, for the setting in Cor.~\ref{cor:opt-ldp}, $n = \tilde{O}(d^{6.5} B^2/(\eps^2 \alpha^2))$ is implied by the algorithm given in \cite{BelloniLNR15}).

\subsection{Differentially Private Answering of Convex Minimization Queries}
\label{sec:app-dp-queries}
An additional implication in the context of differentially private data analysis is to the problem of releasing answers to convex minimization queries over a single dataset that was recently studied by \citenames{Ullman}{Ullman15}. For a dataset $S = (w^i)_{i=1}^n \in \W^n$, a convex set $\K \subseteq \R^d$ and a family of convex functions $\F = \{f(\cdot,w)\}_{w\in \W}$ over $\K$, let $q_f(S) \doteq \argmin_{x\in \K} \frac{1}{n} \sum_{i\in [n]} f(x,w^i)$. \citenames{Ullman}{Ullman15} considers the question of how to answer sequences of such queries $\eps$-approximately (that is by a point $\tilde{x}$ such that $\frac{1}{n} \sum_{i\in [n]} f(\tilde{x},w^i) \leq q_f(S) + \eps$).

We make a simple observation that our algorithms can be used to reduce answering of such queries to answering of counting queries. A
{\em counting} query for a data set $S$, query function $\phi: \W \rar [0,1]$ and accuracy $\tau$ returns a value $v$ such that $|v-\frac{1}{n}\sum_{i\in [n]} \phi(w^i)| \leq \tau$. A long line of research in differential privacy has considered the question of answering counting  queries (see \cite{DworkRoth:14} for an overview). In particular, \citenames{Hardt and Rothblum}{HardtR10} prove that given a dataset of size $n \geq n_0 = O(\sqrt{\log(|\W|)\log(1/\beta)}\cdot \log t/(\alpha\tau^2)$ it is possible to $(\alpha,\beta)$-differentially privately answer any sequence of $t$ counting queries with accuracy $\tau$ (and success probability $\geq 2/3$).

Note that a convex minimization query is equivalent to a stochastic optimization problem when $D$ is the uniform distribution over the elements of $S$ (denote it by $U_S$). Further, a $\tau$-accurate counting query is exactly a statistical query for $D=U_S$. Therefore our SQ algorithms can be seen as reductions from convex minimization queries to counting queries. Thus to answer $t$ convex minimization queries with accuracy $\eps$ we can use the algorithm for answering $t' = t m(\eps)$  counting queries with accuracy $\tau(\eps)$, where $m(\eps)$ is the number of queries to $\STAT(\tau(\eps))$ needed to solve the corresponding stochastic convex minimization problems with accuracy $\eps$. The sample complexity of the algorithm for answering counting queries in \cite{HardtR10} depends only logarithmically on $t$. As a result, the additional price for such implementation is relatively small since such algorithms are usually considered in the setting where $t$ is large and $\log|\W| = \Theta(d)$.
Hence the counting query algorithm in \cite{HardtR10} together with the results in Corollary \ref{cor:solve_cvx_ellp}
immediately imply an algorithm for answering such queries that strengthens
quantitatively and generalizes results in \cite{Ullman15}.
\begin{corollary}
\label{cor:answer-queries}
Let $p \in [1,2]$, $L_0,R>0$, $\K\subseteq\B_p^d(R)$ be a convex body and let $\F=\{f(\cdot,w)\}_{w\in \W} \subset \F_{\|\cdot\|_p}^0(\K,L_0)$ be a finite family of convex functions. Let $\cal Q_\F$ be the set of convex minimization queries corresponding to $\F$. For any $\alpha,\beta,\eps,\delta>0$, there exists an $(\alpha,\beta)$-differentially private algorithm that, with probability at least $1-\delta$ answers any sequence  of $t$ queries from $\cal Q_\F$ with accuracy $\eps$ on datasets of size $n$ for
$$n \geq n_0= \tilde O\left(\frac{(L_0R)^2 \sqrt{\log(|\W|)}\cdot \log t}{\eps^2 \alpha} \cdot \mathrm{polylog}\left(\frac{d}{\beta \delta} \right) \right).$$
\end{corollary}
For comparison, the results in \cite{Ullman15} only consider the $p=2$ case and the stated upper bound is
$$n \geq n_0 = \tilde{O}\left(\frac{(L_0R)^2 \sqrt{\log(|\W|)}\cdot \max\{\log t, \sqrt{d}\}}{\eps^2 \alpha} \cdot \mathrm{polylog}\left(\frac{1}{\beta \delta} \right) \right).$$
Our bound is a significant generalization and an improvement by a factor of at least $\tilde{O}(\sqrt{d}/\log t)$. \citenames{Ullman}{Ullman15} also shows that for generalized linear regression one can replace the $\sqrt{d}$ in the maximum by $L_0R/\eps$. The bound in Corollary \ref{cor:answer-queries} also subsumes this improved bound (in most parameter regimes of interest).

Finally, in the $\kappa$-strongly convex case (with $p=2$),
plugging our bounds from Corollary \ref{cor:solve_str_cvx} into the algorithm in  \cite{HardtR10} we obtain that it suffices to use a dataset of size
$$ n \geq n_0 =\tilde O\left( \dfrac{L_0^2\sqrt{\log(|\W|)} \cdot \log(t\cdot d \cdot \log R)}{\varepsilon\alpha \kappa}\cdot
\mathrm{polylog}\left(\frac{1}{\beta\delta} \right)\right).$$
The bound obtained by \citenames{Ullman}{Ullman15} for the same function class is
$$ n_0 =\tilde O\left( \dfrac{L_0^2R \sqrt{\log(|\W|)}}{\varepsilon\alpha}
\cdot \max\left\{ \dfrac{\sqrt d}{\sqrt{\kappa\varepsilon}},\dfrac{R\log t}{\varepsilon} \right\} \mathrm{polylog}\left(\frac{1}{\beta\delta} \right)\right).$$
Here our improvement over  \cite{Ullman15} is two-fold: We eliminate the $\sqrt{d}$ factor and we essentially eliminate the dependence on $R$ (as in the non-private setting). We remark that our bound might appear incomparable to that in \cite{Ullman15} but is, in fact, stronger since it can be assumed that $\kappa \geq \eps/R^2$ (otherwise, bounds that do not rely on strong convexity are better).

\section{Conclusions}
In this work we give the first treatment of two basic problems in the SQ query model: high-dimensional mean estimation and stochastic convex optimization. In the process, we demonstrate new connections of our questions to concepts and tools from convex geometry, optimization with approximate oracles and compressed sensing.

Our results provide detailed (but by no means exhaustive) answers to some of the most basic questions about these problems. At a high level our findings can be summarized as ``estimation complexity of polynomial-time SQ algorithms behaves like sample complexity" for many natural settings of those problems. This correspondence should not, however, be taken for granted. In many cases the SQ version requires a completely different algorithm and for some problems we have not been able to provide upper bounds that match the sample complexity (see below).

Given the fundamental role that SQ model plays in a variety of settings, our primary motivation and focus is understanding of the SQ complexity of these basic tasks for its own sake. At the same time our results lead to numerous applications among which are new strong lower bounds for convex relaxations and results that subsume and improve on recent work that required substantial technical effort.

As usual when exploring uncharted territory, some of the most useful results can be proved relatively easily given the wealth of existing literature on related topics. Still for many questions, new insights and analyses were necessary (such as the characterization of the complexity of mean estimation for all $q\in [1,\infty)$) and we believe that those will prove useful in further research on the SQ model and its applications. There were also many interesting questions that we encountered but were not able to answer. We list some of those below:
\begin{enumerate}
\item How many samples are necessary and sufficient for answering the queries of our adaptive algorithms, such as those based on the inexact mirror descent. The answer to this question should shed new light of the power of adaptivity in statistical data analysis \cite{DworkFHPRR14:arxiv}.
\item Is there an SQ equivalent of upper bounds on sample complexity of mean estimation for uniformly smooth norms (see App.\ref{sec:Samples} for details).
Such result would give a purely geometric characterization of estimation
complexity of mean estimation.
\item In the absence of a general technique like the one above there are still many important norms we have not addressed. Most notably, we do not know what is the estimation complexity of mean estimation in the spectral norm of a matrix (or other Schatten norms).
\item Is there an efficient algorithm for mean estimation (or at least linear optimization) with estimation complexity of $O(d/\eps^2)$ for which a membership oracle for $\K$ suffices (our current algorithm is efficient only for a fixed $\K$ as it assumes knowledge of John's ellipsoid for $\K$).
\end{enumerate}

\section*{Acknowledgements}
We thank Arkadi Nemirovski, Sasha Rakhlin, Ohad Shamir and Karthik Sridharan for discussions and valuable suggestions about this work.

\appendix
\section{Uniform convexity, uniform smoothness and consequences}
\label{sec:unif_cvx}
A space $(E,\|\cdot\|)$ is $r$-uniformly convex if there exists constant $0<\delta\leq1$
such that for all $x,y\in E$
\begin{equation} \label{unif_conv}
 \|x\|^r+ \delta\|y\|^r \leq \dfrac{\|x+y\|^r + \|x-y\|^r }{2}.
\end{equation}
From classical inequalities (see, e.g., \cite{Ball:1994}) it is known that $\ell_p^d$ for $1<p<\infty$ is
$r$-uniformly convex for $r=\max\{2,p\}$. Furthermore,
\begin{itemize}
\item When $p=1$, the function $\Psi(x)=\frac{1}{2(p(d)-1)}\|x\|_{p(d)}^2$ (with
$p(d)=1+1/\ln d$) is $2$-uniformly convex w.r.t. $\|\cdot\|_1$;
\item When $1<p\leq2$, the function $\Psi(x)=\frac{1}{2(p-1)}\|x\|_p^2$ is
$2$-uniformly convex w.r.t. $\|\cdot\|_p$;
\item When $2<p<\infty$, the function $\Psi(x)=\frac{1}{p}\|x\|_p^p$ is $p$-uniformly
convex w.r.t. $\|\cdot\|_p$.
\end{itemize}

By duality, a Banach space $(E,\|\cdot\|)$ being $r$-uniformly convex is equivalent
to the dual space $(E^{\ast},\|\cdot\|_{\ast})$ being $s$-uniformly smooth,
where $1/r+1/s=1$. This means there exists a constant
$C\geq 1$ such that for all  $w,z \in E^{\ast}$
\begin{equation} \label{unif_smooth}
\dfrac{\|w+z\|_{\ast}^s + \|w-z\|_{\ast}^s}{2} \leq \|w\|_{\ast}^s+C\|z\|_{\ast}^s.
\end{equation}
In the case of $\ell_p^d$ space we obtain that its dual $\ell_q^d$ is
$s$-uniformly smooth for $s=\min\{2,q\}$. Furthermore,
when $1<q\leq 2$ the norm $\|\cdot\|_q$ satisfies \eqref{unif_smooth}
with $s=q$ and $C=1$; when $2\leq q<\infty$, the norm $\|\cdot\|_q$
satisfies \eqref{unif_smooth} with $s=2$ and $C=q-1$.
Finally, observe that for $\ell_{\infty}^d$ we can use the equivalent
norm $\|\cdot\|_{q(d)}$, with $q(d)=\ln d +1$:
$$\textstyle
\|x\|_{\infty}\leq \|x\|_{q(d)} \leq 
 e \,\|x\|_{\infty},$$
and this equivalent norm satisfies  \eqref{unif_smooth} with $s=2$ and $C=q(d)-1=\ln d$,
 that  grows only moderately with dimension.

\section{Sample complexity of mean estimation}
\label{sec:Samples}

The following is a standard analysis based on Rademacher complexity and
uniform convexity (see, e.g., \cite{Pisier:2011}).
Let $(E,\|\cdot\|)$ be an $r$-uniformly convex space. We are interested in
the convergence of the empirical mean to the true mean in the dual norm (to the one we optimize in).
By Observation \ref{obs:lin_opt_mean_est} this is sufficient to bound the
error of optimization using the empirical estimate of the gradient on
$\K \doteq \B_{\|\cdot\|}$.

Let $(\bw^j)_{j=1}^n$ be i.i.d.~samples of a random variable $\bw$ with
mean $\bar w$, and let $\bar \bw^n\doteq \frac{1}{n}\sum_{j=1}^n \bw^j$ be the
empirical mean estimator.
Notice that $$ \left\|\bar \bw^n -\bar w \right\|_{\ast}
= \sup_{x\in\K} \left|\left\la \bar \bw^n -\bar w, x  \right\ra \right|.$$
Let $(\sigma_j)_{j=1}^n$ be i.i.d. Rademacher random variables (independent
of $(\bw^j)_j$). By a standard symmetrization argument, we have
\begin{eqnarray*}
\E_{\bw^1,\ldots,\bw^n} \sup_{x\in \K}\left| \left\langle \frac{1}{n}\sum_{j=1}^n \bw^j, x\right\rangle - \left\langle \bar w,x\right\rangle \right|
&\leq & 2 \E_{\sigma_1,\ldots,\sigma_n}\E_{\bw^1,\ldots,\bw^n} \sup_{x\in{\cal K}} \left| \sum_{j=1}^n \sigma_j \langle \bw^j,x\rangle \right|.
\end{eqnarray*}

For simplicity, we will denote $\|\K\|\doteq \sup_{x\in \K}\|x\|$ the $\|\cdot\|$ radius of $\K$.
Now by the Fenchel inequality
\begin{eqnarray*}
\E_{\sigma_1,\ldots,\sigma_n} \sup_{x\in \K} \left| \sum_{j=1}^n \sigma_j \langle \bw^j,x\rangle \right|
&\leq &  \inf_{\lambda>0} \E_{\sigma_1,\ldots,\sigma_n} \left\{
\frac{1}{r\lambda}\sup_{x\in \K}\|x\|^r+\frac{1}{s\lambda}\left\|\frac{\lambda}{n} \sum_{j=1}^n \sigma_j \bw^j\right\|_{\ast}^s \right\} \\
&\leq &  \inf_{\lambda>0} \E_{\sigma_1,\ldots,\sigma_{n-1}} \left\{
\frac{1}{r\lambda}\|\K\|^r \right.\\
&       & \left. +\frac{\lambda^{s-1}}{sn^s} \frac12\left[
\left\|\sum_{j=1}^{n-1} \sigma_j \bw^j + \sigma_n \bw^n\right\|_{\ast}^s
+\left\|\sum_{j=1}^{n-1} \sigma_j \bw^j - \sigma_n \bw^n\right\|_{\ast}^s
\right]
\right\}\\
&\leq &  \inf_{\lambda>0} \E_{\sigma_1,\ldots,\sigma_{n-1}} \left\{
\frac{1}{r\lambda}\|\K\|^r+\frac{\lambda^{s-1}}{sn^s}
\left[ \left\|\sum_{j=1}^{n-1} \sigma_j \bw^j \right\|_{\ast}^s + C\|\bw^n\|_{\ast}^s \right] \right\},
\end{eqnarray*}
where the last inequality holds from the $s$-uniform smoothness of
$(E^{\ast},\|\cdot\|_{\ast})$. Proceeding inductively we obtain
\begin{eqnarray*}
\E_{\sigma_1,\ldots,\sigma_n} \sup_{x\in{\cal K}} \left| \sum_{j=1}^n \sigma_j \langle \bw^j,x\rangle \right|
&\leq & \inf_{\lambda>0} \left\{ \frac{1}{r\lambda}\|\K\|^r+
\frac{C\lambda^{s-1}}{sn^s}\sum_{j=1}^n \|\bw^j\|_{\ast}^s
\right\}.
\end{eqnarray*}

It is a straightforward computation to obtain the optimal
$\bar\lambda=\frac{\|\K\|^{r-1}n}{C^{1/s}\left(\sum_j\|\bw^j\|_{\ast}^s\right)^{1/s}}$, which gives an upper bound
$$\E_{\sigma_1,\ldots,\sigma_n} \sup_{x\in{\cal K}} \left| \sum_{j=1}^n \sigma_j \langle \bw^j,x\rangle \right|\leq
\dfrac{1}{n^{1/r}}C^{1/s} \sup_{x\in \K}\|x\|\left(\frac{1}{n}\sum_{j=1}^n\|\bw^j\|_{\ast}^s\right)^{1/s}.$$

By simply upper bounding the quantity above by $\varepsilon>0$,
we get a sample complexity bound for achieving
$\varepsilon$ accuracy in expectation, $n=\lceil C^{r/s}/\varepsilon^r \rceil$, where
$C\geq 1$ is any constant satisfying \eqref{unif_smooth}.
For the standard $\ell_p^d$-setup, i.e., where $(E,\|\cdot\|)=(\R^d,\|\cdot\|_p)$, by
the parameters of uniform convexity and uniform smoothness provided in Appendix
\ref{sec:unif_cvx}, we obtain the following bounds on sample complexity:
\begin{enumerate}
\item[(i)] For $p=1$, we have $r=s=2$ and $C=\ln d$, by using the
equivalent norm $\|\cdot\|_{p(d)}$. This implies that
$n=O\left(\dfrac{\ln d}{\varepsilon^2}\right)$ samples suffice.
\item[(ii)] For $1<p\leq 2$, we have $r=s=2$ and $C=q-1$. This implies
that $n=\left\lceil\dfrac{q-1}{\varepsilon^2}\right\rceil$ samples suffice.
\item[(iii)] For $2<p<\infty$, we have $r=p$, $s=q$ and $C=1$. This implies
that $n=\left\lceil\dfrac{1}{\varepsilon^p}\right\rceil$ samples suffice.
\end{enumerate}

\iffull\else \fi
\iffull\else \fi
\section{Proof of Corollary \ref{cor:solve_cvx_ellp}}
\label{proof_solve_cvx_ellp}

Note that by Proposition \ref{Prop:Inexact_MD} in order to obtain an
$\varepsilon$-optimal solution to a non-smooth convex optimization problem it suffices
to choose $\eta=\varepsilon/2$, and
$T=\left\lceil r2^r L_0^r D_{\Psi}(\K)/\varepsilon^r\right\rceil$.
Since $\K\subseteq\B_p(R)$, to satisfy \eqref{ApproxSubgrad} it is sufficient
to have for all $y\in \B_p(R)$,
$$ \la \nabla F(x)-\tilde g(x), y \ra \leq \eta/2. $$
Maximizing the left hand side on $y$, we get a sufficient condition:
$\|\nabla F(x)-\tilde g(x)\|_q R \leq \eta/2$. We can satisfy this condition by solving
the mean estimation problem in $\ell_q$-norm with error
$\eta/[2L_0R]=\varepsilon/[4L_0 R]$
(recall that $f(\cdot,w)$ is $L_0$ Lipschitz w.r.t. $\|\cdot\|_p$).
Next, using the uniformly convex functions for $\ell_p$ from Appendix
\ref{sec:unif_cvx}, together with the bound on the number of queries and
error for the mean estimation problems in $\ell_q$-norm
from Section \ref{Subsec:L_q}, we obtain that the total number of queries and the type of queries we need for stochastic optimization in the non-smooth $\ell_p$-setup are:

\begin{itemize}
\item $p=1$: We have $r=2$ and $D_{\Psi}(\K)=\dfrac{e^2\ln d}{2}R^2$.
As a consequnce, solving the convex program amounts to using
 $O\left(d\cdot \left(\dfrac{L_0R}{\varepsilon}\right)^2 \ln d\right)$ queries to
$\STAT\left(\dfrac{\varepsilon}{4L_0R}\right)$.

\item $1< p< 2$: We have $r=2$ and $D_{\Psi}(\K)=\dfrac{1}{2(p-1)}R^2$.
As a consequence, solving the convex program amounts to using
$O\left(d\log d\cdot \dfrac{1}{(p-1)}\left(\dfrac{L_0R}{\varepsilon}\right)^2\right)$
queries to $\STAT\left( \Omega\left(\dfrac{\varepsilon}{[\log d]L_0R}\right)\right)$.

\item $p=2$: We have $r=2$ and $D_{\Psi}(\K)=R^2$.
As a consequence, solving the convex program amounts to using
$O\left(d\cdot \left(\dfrac{L_0R}{\varepsilon}\right)^2\right)$
queries to $\STAT\left( \Omega\left(\dfrac{\varepsilon}{L_0R}\right)\right)$.

\item {\bf $2<p<\infty$:} We may choose  $r=p$, $D_{\Psi}(\K)=\dfrac{2^{p-2}}{p}R^p$.
As a consequence, solving the convex program amounts to using
$O\left(d\log d\cdot 2^{2p-2}\left(\dfrac{L_0R}{\varepsilon}\right)^p\right)$
queries to $\VSTAT\left(\left(\dfrac{64 L_0R \log d}{\varepsilon}\right)^p\right)$.
\end{itemize}
\hfill $\qed$ 
\section{Proof of Corollary \ref{cor:solve_smooth_cvx_ellp}}
\label{proof_solve_smooth_cvx_ellp}
Similarly as in Appendix \ref{proof_solve_cvx_ellp}, given $x\in\K$,
we can obtain $\tilde g(x)$ by mean estimation problem in $\ell_q$-norm
with error $\varepsilon/[12L_0 R]$ (notice we have chosen $\eta=\varepsilon/6$).

Now, by Proposition \ref{prop:dAspremont}, in order to obtain an
$\varepsilon$-optimal solution it suffices to run the accelerated method for
$T=\left\lceil\sqrt{2L_1D_{\Psi}(\K)/\varepsilon}\right\rceil$ iterations,
each of them requiring $\tilde g$ as defined above.
By using the $2$-uniformly convex functions for $\ell_p$, with $1\leq p\leq 2$,
from Appendix \ref{sec:unif_cvx}, together with the bound on the number of
queries and error for the mean estimation problems in $\ell_q$-norm
from Section \ref{Subsec:L_q}, we obtain that the total number of queries and the type of queries we need for stochastic optimization in the smooth $\ell_p$-setup is:

\begin{itemize}
\item $p=1$: We have $r=2$ and $D_{\Psi}(\K)=\dfrac{e^2\ln d}{2}R^2$.
As a consequnce, solving the convex program amounts to using
 $O\left(d\cdot \sqrt{ \ln d\cdot\dfrac{L_1R^2}{\varepsilon} } \right)$ queries to
$\STAT\left(\dfrac{\varepsilon}{12L_0R}\right)$.

\item $1< p< 2$: We have $r=2$ and $D_{\Psi}(\K)=\dfrac{1}{2(p-1)}R^2$.
As a consequence, solving the convex program amounts to using
$O\left(d\log d\cdot \sqrt{\dfrac{1}{(p-1)}\cdot\dfrac{L_1R^2}{\varepsilon} }\right)$
queries to $\STAT\left( \Omega\left(\dfrac{\varepsilon}{[\log d]L_0R}\right)\right)$;

\item $p=2$: We have $r=2$ and $D_{\Psi}(\K)=R^2$.
As a consequence, solving the convex program amounts to using
$O\left(d\cdot \sqrt{ \dfrac{L_1R^2}{\varepsilon} }\right)$
queries to $\STAT\left( \Omega\left(\dfrac{\varepsilon}{L_0R}\right)\right)$.
\end{itemize}
\hfill $\qed$ 
\iffull\else\fi
\iffull\else \fi

\bibliographystyle{alpha}
\bibliography{../sq-convex,../vf-allrefs}

\end{document}